\documentclass[10pt,journal,compsoc]{IEEEtran}
\ifCLASSOPTIONcompsoc
  \usepackage[nocompress]{cite}
\else
  \usepackage{cite}
\fi
\ifCLASSINFOpdf
\else
\fi

\usepackage{amsmath}
\usepackage{amsfonts}
\usepackage{cleveref}
\usepackage{amsthm}
\usepackage{enumitem}
\usepackage{todonotes}
\usepackage{algorithm}
\usepackage{algpseudocode}
\usepackage{amsfonts}
\usepackage{array}
\usepackage{caption}
\usepackage{subcaption}
\usepackage{booktabs} 
\usepackage{xcolor}
\usepackage{threeparttable}
\usepackage{subcaption}

\usepackage{algorithmicx}
\usepackage[utf8]{inputenc}
\usepackage[T1]{fontenc}
\usepackage{url}

\begin{document}
%
\title{Towards Rotation-only Imaging Geometry:\\ Rotation Estimation}
%
%
%
%

\author{Xinrui~Li, Qi~Cai, and~Yuanxin~Wu, ~\IEEEmembership{Senior~Member,~IEEE}
\IEEEcompsocitemizethanks{\IEEEcompsocthanksitem  Xinrui~Li, Qi~Cai, and Yuanxin~Wu are with the Shanghai Key Laboratory of Navigation and Location-based Services, School of Electronic Information and Electrical Engineering, Shanghai Jiao Tong University, Shanghai 200240, China.
E-mail: physalis@sjtu.edu.cn, qicaicn@gmail.com, yuanx\_wu@hotmail.com.
Qi~Cai and Yuanxin~Wu contributted equally to this work.}
\thanks{}}

\IEEEtitleabstractindextext{%
\begin{abstract}
  Structure from Motion (SfM) is a critical task in computer vision, aiming to recover the 3D scene structure and camera motion from a sequence of 2D images. The recent pose-only imaging geometry decouples 3D coordinates from camera poses and demonstrates significantly better SfM performance through pose adjustment. 
  Continuing the pose-only perspective, this paper explores the critical relationship between the scene structures, rotation and translation. Notably, the translation can be expressed in terms of rotation, allowing us to condense the imaging geometry representation onto the rotation manifold. A rotation-only optimization framework based on reprojection error is proposed for both two-view and multi-view scenarios. 
  The experiment results demonstrate superior accuracy and robustness performance  over the current state-of-the-art rotation estimation methods, even comparable to multiple bundle adjustment iteration results. 
  Hopefully, this work contributes to even more accurate, efficient and reliable 3D visual computing.
\end{abstract}

}

\maketitle

\IEEEdisplaynontitleabstractindextext

%
\IEEEpeerreviewmaketitle

\IEEEraisesectionheading{\section{Introduction}\label{sec:introduction}}
\IEEEPARstart{S}{tructure} from motion plays a significant role in modern computer vision, enabling the recovery of three-dimensional (3D) scene structure and camera motion trajectories from a sequence of 2D images. SfM techniques have been extensively applied to domains including but not limited to autonomous driving, robotic navigation, augmented reality, and scene reconstruction.

Nonlinear optimization methods are the key determinants in the overall accuracy of vision systems. The Bundle Adjustment (BA) algorithm \cite{BAinLarge,BundleAdjustmentAModernSynthesis} is widely employed as a mainstream optimization approach in various SfM platforms \cite{KneipOpenGV,OpenMVGPierre,theia-manual, glomap}.
The BA algorithm needs initial camera poses and 3D point coordinates as inputs, and is highly sensitive to these initial values \cite{OutOfCoreBA, DistributedBA,1dsfm}. 
In two-view 3D estimation, there are many methods for camera pose initialization \cite{STEWENIUS2006284,NisterRelative,Pizarro2003RelativePE,SevenPtHartley,LONGUETHIGGINS198761,cai2024linearrelativeposeestimation}, among which the five-point method \cite{STEWENIUS2006284,NisterRelative} is popular because of its adaptability to different scene structures and good accuracy. In the multi-view 3D estimation, global rotation initialization \cite{Govindu01,Martinec07,Govindu004,HartleyL1,ChatterjeePAMI2018}, global translation initialization \cite{QiCai_TPAMI, LUD, 1dsfm}, and global 3D reconstruction \cite{QiCai_TPAMI,KANG20142974,MultipleViewGeometry} need to be carried out in sequence, serving as initial inputs to global optimization.
In order to ensure the overall accuracy, these initial inputs require delicate processing before being fed into the optimizer. 
For example, the OpenMVG platforms perform multiple-rounds of BA optimization on different parameters to achieve satisfactory performance \cite{OpenMVGPierre}.
However, some special scene structures might lead to abnormal initial solutions \cite{Maybank90,MultipleViewGeometry,1dsfm, jiangGlobalLinearMethod2013a}, which seriously affect the accuracy and robustness of the final optimization. Additionally, it is worth mentioning that too many feature points lead to the curse of dimensionality in the BA \cite{BAinLarge, erikssonConsensusBasedFrameworkDistributed2016}.

Agarwal \cite{agarwal2022chiral} pointed out that the classical multi-view geometry founded on the epipolar relationship is incomplete, which isolates 3D depth information and ignores the chirality constraint.
Recent work \cite{QiCai_IJCV} proposed a complete pose-only imaging geometry, in which 3D scene points are represented by camera poses, and thus the imaging geometry is formulated on the pose manifold \cite{QiCai_TPAMI}.
Drawing inspiration from the process of deriving camera translations from rotations and observations \cite{QiCai_TPAMI}, we are led to question whether this implies that the origin of imaging geometry should be attributed to the problem of rotation estimation.

Rotation estimation has always been a fundamental challenge in modern SfM. 
Yet, experiments have shown that prevailing methods, whether two-view or multi-view, exhibit a gap in the accuracy of direct rotation estimation compared to the class of BA optimization methods, particularly in multi-view 3D estimation \cite{KneipOpenGV,OpenMVGPierre}. 
This might stem from the lack of a pure-rotation perspective in addressing imaging geometry, particularly the underlying reprojection error model.

Thus, our work---focused on a rotation-only imaging geometry and its corresponding rotation estimation methods---unfolds from here. The technical contributions of this work include:

\begin{enumerate}
    \item \textbf{Rotation-only imaging geometry:} Proposing to condense our recent pose-only representation into a lower dimension, namely, a rotation-only parametrization. It is independent of camera translation estimation.

    \item \textbf{New framework for rotation estimation:} Representing the reprojection error totally on the rotation manifold, leading to superior accuracy, robustness and efficiency of rotation estimation in both two-view and multi-view scenes.

    \item \textbf{Rationale and detector of translation degradation:} Revealing the impact of scene structure on translation estimation with a scene structure detector.
\end{enumerate}

\section{Related Works}

\subsection{Two-view Rotation Optimization Methods}
Mainstream two-view rotation optimization methods typically employ the BA algorithm \cite{BundleAdjustmentAModernSynthesis} to simultaneously estimate the camera's intrinsic parameters and poses as well as 3D feature points. The objective function is designed to minimize the reprojection errors of all observations. Many studies have made modifications to the reprojection error to enhance BA's adaptability to specific scenarios. 
For instance, Kneip \cite{KneipOpenGV} formulated a spherical reprojection error based on unit sphere chordal distances. A rotation-only optimization method was also proposed by Kneip in \cite{Kneip13ICCV}, leveraging the smallest eigenvalue associated with visual-geometric properties. 
Additionally, Zhao \cite{ZhaoJi} circumvented direct optimization of camera rotation parameters in two-view geometry framework, by solving the essential matrix through a quadratically constrained quadratic program (QCQP). Nevertheless, 
its performance degrades significantly in planar scenarios \cite{cai2024linearrelativeposeestimation}. Cai \cite{cai2024linearrelativeposeestimation} introduced a pose-only optimization framework for two-view estimation, which demonstrates remarkable enhancements in terms of scene adaptability, robustness and precision.

\subsection{Multi-view Rotation Estimation Methods}
Multi-view rotation estimation is widely carried out by global averaging or graph optimization methods, which feed on relative rotations derived from two-view pose estimation. 
Due to the computational efficiency, global averaging and graph optimization methods have been extensively adopted in various SfM software platforms such as OpenMVG \cite{OpenMVGPierre}, GLOMAP \cite{glomap} and Theia Vision Library \cite{theia-manual}. Govindu \cite{Govindu01} employed quaternion metric distances between relative and global rotations to solve the global rotation estimation problem, which transforms the original optimization into a linear least squares problem and significantly improves computational efficiency. Later in \cite{Govindu004}, Govindu further proposed an iterative averaging method for rotation estimation on the \( \mathrm{SO}(3) \) manifold.
Martinec \cite{Martinec07} introduced an approximate rotation representation and used Singular Value Decomposition (SVD) to compute chordal distance, while enforcing constraints within the \( \mathrm{SO}(3) \) manifold to maintain the orthogonality of rotation matrices. Hartley \cite{HartleyL1} leveraged the transitivity of relative rotations by the Weiszfeld algorithm to iteratively solve the $L_1$ mean of each camera rotation on the \( \mathrm{SO}(3) \) manifold. Chatterjee and Govindu \cite{Chatterjee13,ChatterjeePAMI2018} proposed global rotation estimation based on the quasi-Newton method and iterative reweighted least squares (IRLS), which significantly enhances robustness to outliers.

However, a key limitation of the methods discussed above lies in a strong dependence on the accuracy of relative pose estimates. Incorporating image measurements generally produces superior accuracy and robustness \cite{ROBA}.
Lee \cite{ROBA} extended observation matrix in the work \cite{Kneip13ICCV} to a multi-view optimization by representing relative rotations as global rotations, achieving the rotation-only estimation, although strictly speaking, it does not fall under the category of minimizing reprojection error. 

Commonly serving as the final stage of 3D estimation, global optimization algorithms, including the class of BA optimization methods \cite{BundleAdjustmentAModernSynthesis,BAinLarge, zhaoParallaxBABundleAdjustment2015} and the recent pose adjustment (PA) \cite{QiCai_TPAMI} optimization method, minimizing the reprojection error, can generally be used for optimal rotation estimation. PA requires the initialization of translation, while BA additionally requires initial 3D coordinates.

\section{Mathematical Notations and Equalities}
We use bold font to denote vectors and matrices. The symbol $\left[ {\boldsymbol{v}} \right]_\times$ represents a skew-symmetric matrix corresponding to the vector $\boldsymbol{v}$ in the three-dimensional real space $\mathbb{R}^3$. The notation $\|\boldsymbol{v}\|$ represents the norm of the vector $\boldsymbol{v}$, and $\vec{\boldsymbol{v}}$ denotes the unit direction of the vector $\boldsymbol{v}$, i.e., $\vec{\boldsymbol{v}} = \frac{\boldsymbol{v}}{\|\boldsymbol{v}\|}$. The notation $\boldsymbol{v}^{\pm}$ is used to represent the vector $\vec{\boldsymbol{v}}$  in both positive and negative directions. The element in the $i$-th row and $j$-th column of the matrix $\boldsymbol{M}$ is denoted by $\boldsymbol{M}_{(i,j)}$, and the $i$-th element of the vector $\boldsymbol{v}$ is denoted by $\boldsymbol{v}_{(i)}$.

The natural basis vectors of three-dimensional space are defined as $\boldsymbol{e}_1 = \left[ {1, 0, 0} \right]^T$, $\boldsymbol{e}_2 = \left[ {0, 1, 0} \right]^T$, and $\boldsymbol{e}_3 = \left[ {0, 0, 1} \right]^T$, whereas the identity matrix in this space is represented by $\boldsymbol{I}$. In addition, let $\boldsymbol{a}$ and $\boldsymbol{b}$ denote two vectors in three-dimensional space, with the angle between them expressed as $\angle \left( \boldsymbol{a}, \boldsymbol{b} \right)$ and the span plane formed by these two vectors denoted as \( \langle \boldsymbol{a}, \boldsymbol{b} \rangle \).

The following mathematical equalities will be utilized throughout the mathematical derivations in this paper:
\begin{equation}
    {\left[ {\boldsymbol{a}} \right]_ \times }{\left[ {\boldsymbol{b}} \right]_ \times } = {\boldsymbol{b}}{{\boldsymbol{a}}^T} - {{\boldsymbol{a}}^T}{\boldsymbol{bI}}.
    \label{eq:fundamental1}
\end{equation}
\begin{equation}
   {\left[ {{{\left[ {\boldsymbol{a}} \right]}_ \times }{\boldsymbol{b}}} \right]_ \times } = {\left[ {\boldsymbol{a}} \right]_ \times }{\left[ {\boldsymbol{b}} \right]_ \times } - {\left[ {\boldsymbol{b}} \right]_ \times }{\left[ {\boldsymbol{a}} \right]_ \times } = {\boldsymbol{b}}{{\boldsymbol{a}}^T} - {\boldsymbol{a}}{{\boldsymbol{b}}^T}.
    \label{eq:fundamental2}
\end{equation}

\section{Two-View Representation on Rotation Manifold}
\label{section:TWOVIEW}

Suppose there are $m$ 3D points and $n$ observing views. Let the set of 3D point coordinates be denoted as: ${P_f} = \left\{ {{\boldsymbol{X}}_k^w = {{\left( {x_k^w,y_k^w,z_k^w} \right)}^T}\left| {k = 1,...,m} \right.} \right\}$, where $\boldsymbol{X}_k^w $ represents the coordinates of the $k$-th point in the world coordinate system. Let the set of camera poses be denoted as ${P_c} = \left\{ {\left[ {{{\boldsymbol{R}}_i}\left| {{{\boldsymbol{t}}_i}} \right.} \right]\left| {i = 1,...,n} \right.} \right\}$, where $\boldsymbol{R}_i$, $\boldsymbol{t}_i$ represent the global rotation and translation of the $i$-th view, respectively. In the noise-free case, the set of observations is denoted as ${P_o} = \left\{ {{{\boldsymbol{X}}_{ik}} = {{\left( {{x_{ik}},{y_{ik}},1} \right)}^T}\left| {i = 1,...,n,k = 1,...,m} \right.} \right\}$, where $\boldsymbol{X}_{ik}$ represents the observation of the $k$-th 3D point by the $i$-th view.

In multi-view geometry, the motion structure of views can be represented by a graph $\mathcal{G} = (\mathcal{V}, \mathcal{E})$. The set of nodes is denoted as $\mathcal{V}=\left\{ {{C_1},{C_2},...,{C_n}} \right\}$, where each node $C_i$ represents the pose and observation information of the $i$-th view. The set of edges is denoted as $\mathcal{E}=\left\{ e_{ij} \mid C_i \text{ and } C_j \text{ have matched observations} \right\}$, as illustrated in Fig.~\ref{fig:connectivity_graph}. In the connectivity graph, each edge inherently represents a two-view constraint.  The relative rotation and relative translation are denoted as $\boldsymbol{R}_{ij} = {{\boldsymbol{R}}_j}{\boldsymbol{R}}_i^T$ and ${{\boldsymbol{t}}_{ij}} = {{\boldsymbol{R}}_j}\left( {{{\boldsymbol{t}}_i} - {{\boldsymbol{t}}_j}} \right)$, respectively. The section will introduce the key property of pose-only imaging geometry \cite{QiCai_IJCV}---the pair of pose-only (PPO) constraint---and then formulate the visual geometric constraints on the rotation manifold, serving as the theoretical foundation for subsequent optimization.

\begin{figure}[!t] 
    \centering 
    \includegraphics[width=0.48\textwidth]{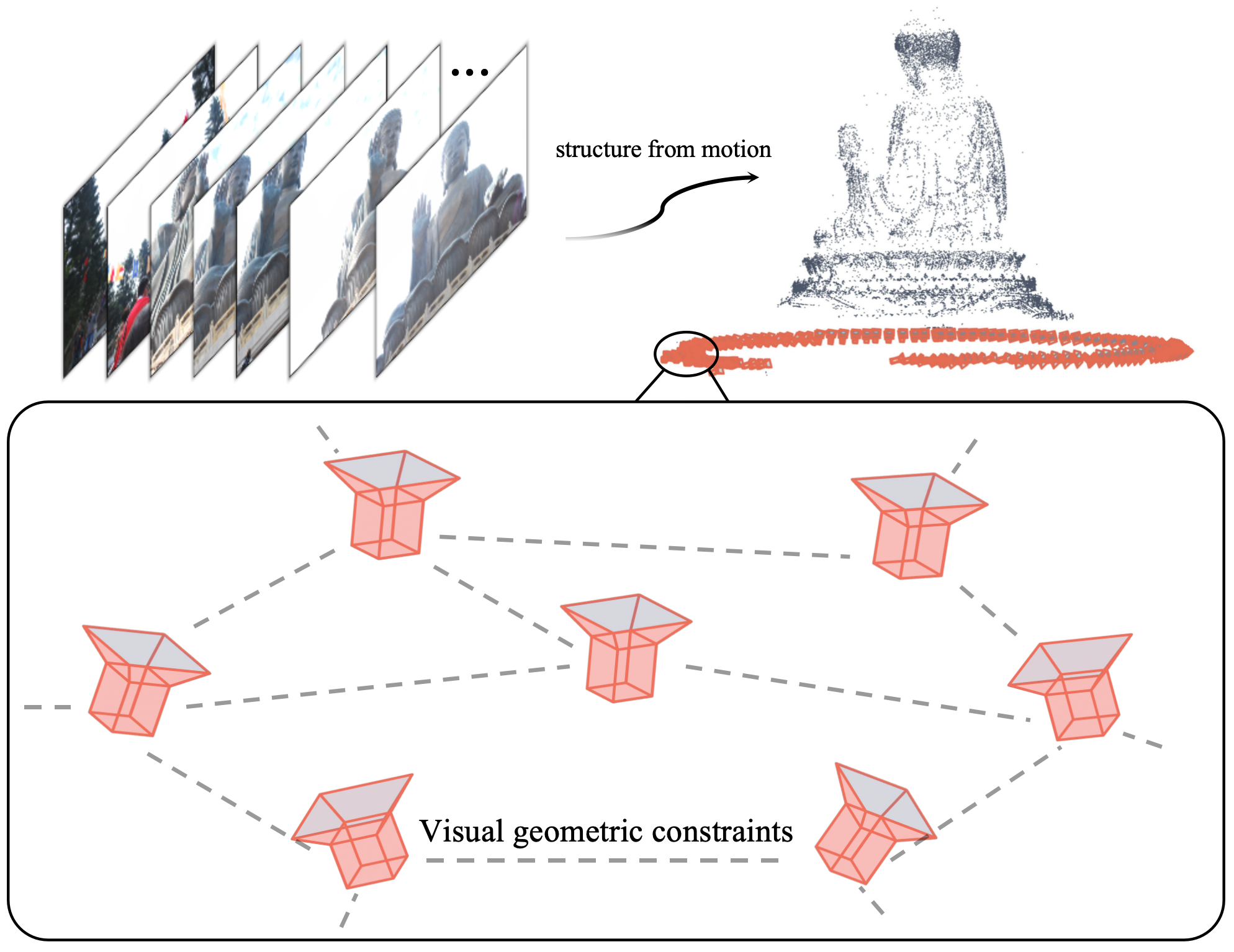} 
    \caption{An illustration of reconstruction using Lund dataset \cite{lunddataset}. Views form a connectivity graph $\mathcal{G}$, where nodes of graph include observation and pose information of views, while edges connecting nodes indicate presence of matched observations and visual geometric constraints.
} 
    \label{fig:connectivity_graph} 
\end{figure}

\subsection{Exploration of PPO Constraints}
Without loss of generality, assume that the left view is the \(i\)-th view and the right view is the \(j\)-th view. Suppose that there are \(m_{ij}\) 3D points observed by both views \(i\) and \(j\). The coordinates of the 3D point \(k\) in views \(i\) and \(j\) are denoted as ${\boldsymbol{X}}_{ik}^C = {\left( {x_{ik}^C,y_{ik}^C,z_{ik}^C} \right)^T}$ and ${\boldsymbol{X}}_{jk}^C = {\left( {x_{jk}^C,y_{jk}^C,z_{jk}^C} \right)^T}$, respectively. The following coordinate system mapping relationship holds:
\begin{equation}
    {{\boldsymbol{X}}_{jk}^C = {{\boldsymbol{R}}_{ij}}{\boldsymbol{X}}_{ik}^C + {{\boldsymbol{t}}_{ij}}}.
    \label{eq:RigidTransformation}
\end{equation}

By the ideal perspective camera model, we obtain the two-view imaging equation as follows:
\begin{equation}
    z_{jk}^C{{\boldsymbol{X}}_{jk}} = z_{ik}^C{{\boldsymbol{R}}_{ij}}{{\boldsymbol{X}}_{ik}} + {{\boldsymbol{t}}_{ij}},
    \label{eq:two-view_imaging_equation}
\end{equation}
where $z_{ik}^C$ and $z_{jk}^C$ represent depths of the 3D point \( k \) in views \( i \) and \( j \), respectively. The chirality constraint must be satisfied \cite{MultipleViewGeometry}, i.e., \( z_{ik}^C,z_{jk}^C \in {\mathbb{R}^ + } \). That is to say, 3D points in the real world must be located in front of cameras to be observable. Left-multiplying both sides of (\ref{eq:two-view_imaging_equation}) by \( \left[ {\boldsymbol{X}_{jk}} \right]_{\times} \) yields:
\begin{equation}
    z_{ik}^C{{\boldsymbol{\theta }}_{ijk}} = {{\boldsymbol{\beta }}_{ijk}}\|{{\boldsymbol{t}}_{ij}}\|.
    \label{eq:crosspruduct1}
\end{equation}
Here, 
\begin{equation}
    {{\boldsymbol{\theta }}_{ijk}} = {\left[ {{{\boldsymbol{R}}_{ij}}{{\boldsymbol{X}}_{ik}}} \right]_ \times }{{\boldsymbol{X}}_{jk}},
    \label{eq:theta}
\end{equation}
and ${{\boldsymbol{\beta }}_{ijk}} = {\left[ {{{\boldsymbol{X}}_{jk}}} \right]_ \times }{\overrightarrow {\boldsymbol{t}} _{ij}}$. This indicates that the direction of the vector ${{\boldsymbol{\theta }}_{ijk}}$ must be the same as that of ${{\boldsymbol{\beta }}_{ijk}}$, which is a necessary and sufficient condition for $z_{ik}^C$ to satisfy the chirality constraint \cite{QiCai_TPAMI}. Similarly, left-multiplying both sides of equation (\ref{eq:two-view_imaging_equation}) by $ {\left[ {{{\boldsymbol{R}}_{ij}}{{\boldsymbol{X}}_{ik}}} \right]_ \times }$yields:
\begin{equation}
    z_{jk}^C{{\boldsymbol{\theta }}_{ijk}} = {{\boldsymbol{\alpha }}_{ijk}}\|{{\boldsymbol{t}}_{ij}}\|.
    \label{eq:crosspruduct2}
\end{equation}
where ${{\boldsymbol{\alpha }}_{ijk}} = {\left[ {{{\boldsymbol{R}}_{ij}}{{\boldsymbol{X}}_{ik}}} \right]_ \times }{\overrightarrow {\boldsymbol{t}} _{ij}}$.
It follows that the vectors \( {\boldsymbol{\theta }}_{ijk} \) and \( {{\boldsymbol{\alpha }}_{ijk}} \) must be in the same direction, which is a necessary and sufficient condition for $z_{jk}^C$ to satisfy the chirality constraint. Since \( z_{ik}^C,z_{jk}^C \in {\mathbb{R}^ + } \), we can take the magnitude of both sides of (\ref{eq:crosspruduct1}) and (\ref{eq:crosspruduct2}) to obtain:
\begin{equation}
    z_{ik}^C\left\|{{\boldsymbol{\theta }}_{ijk}}\right\| = 
    \|{{\boldsymbol{\beta }}_{ijk}}\|
    \|{{\boldsymbol{t}}_{ij}}\|.
    \label{eq:modulus1}
\end{equation}
\begin{equation}
    z_{jk}^C\left\| {{{\boldsymbol{\theta }}_{ijk}}} \right\| = 
    \|{{\boldsymbol{\alpha }}_{ijk}}\|
    \|{{\boldsymbol{t}}_{ij}}\|.
    \label{eq:modulus2}
\end{equation}

Multiplying both sides of (\ref{eq:two-view_imaging_equation}) by $\left\| {{{\boldsymbol{\theta }}_{ijk}}} \right\|$ and using (\ref{eq:modulus1}) and (\ref{eq:modulus2}), we obtain:
\begin{equation}
    \|{{\boldsymbol{\alpha }}_{ijk}}\|
    \|{{\boldsymbol{t}}_{ij}}\|{{\boldsymbol{X}}_{jk}} = \|{{\boldsymbol{\beta }}_{ijk}}\|
    \|{{\boldsymbol{t}}_{ij}}\|{{\boldsymbol{R}}_{ij}}{{\boldsymbol{X}}_{ik}} + \left\| {{{\boldsymbol{\theta }}_{ijk}}} \right\|{{\boldsymbol{t}}_{ij}}.
    \label{eq:PPO1}
\end{equation}
or
\begin{equation}
   \left\| {{{\boldsymbol{t}}_{ij}}} \right\|\left( {\left\| {{{\boldsymbol{\alpha }}_{ijk}}} \right\|{{\boldsymbol{X}}_{jk}} - \left\| {{{\boldsymbol{\beta }}_{ijk}}} \right\|{{\boldsymbol{R}}_{ij}}{{\boldsymbol{X}}_{ik}} - \left\| {{{\boldsymbol{\theta }}_{ijk}}} \right\|{{\overrightarrow {\boldsymbol{t}} }_{ij}}} \right) = {\boldsymbol{0}},
    \label{eq:PPO2}
\end{equation}
where ${{\boldsymbol{\alpha }}_{ijk}} = {\left[ {{{\boldsymbol{R}}_{ij}}{{\boldsymbol{X}}_{ik}}} \right]_ \times }{\overrightarrow {\boldsymbol{t}} _{ij}}$ and ${{\boldsymbol{\beta }}_{ijk}} = {\left[ {{{\boldsymbol{X}}_{jk}}} \right]_ \times }{\overrightarrow {\boldsymbol{t}} _{ij}}$. When the corresponding camera undergoes non-pure rotation motion, namely $\left\| {\boldsymbol{t}} \right\| \ne 0$, the above constraint can be derived from (\ref{eq:PPO2})
\begin{equation}
   \left\| {{{\boldsymbol{\alpha }}_{ijk}}} \right\|{{\boldsymbol{X}}_{jk}} = \left\| {{{\boldsymbol{\beta }}_{ijk}}} \right\|{{\boldsymbol{R}}_{ij}}{{\boldsymbol{X}}_{ik}} + \left\| {{{\boldsymbol{\theta }}_{ijk}}} \right\|{\overrightarrow {\boldsymbol{t}} _{ij}}.
    \label{eq:PPO3}
\end{equation}

When the corresponding camera undergoes pure rotation motion, $\left\| {{{\boldsymbol{t}}_{ij}}} \right\| = 0$. From (\ref{eq:modulus1}), we know that $\left\| {{{\boldsymbol{\theta }}_{ijk}}} \right\| = 0$. From (\ref{eq:two-view_imaging_equation}), it follows that the vector ${{\boldsymbol{X}}_{jk}}$ is in the same direction as ${{\boldsymbol{R}}_{ij}}{{\boldsymbol{X}}_{ik}}$, so is the vector $\left\| {{{\boldsymbol{\alpha }}_{ijk}}} \right\|{{\boldsymbol{X}}_{jk}}$ with $\left\| {{{\boldsymbol{\beta }}_{ijk}}} \right\|{{\boldsymbol{R}}_{ij}}{{\boldsymbol{X}}_{ik}} $. Therefore, we have $\sin \angle \left( {{{\boldsymbol{R}}_{ij}}{{\boldsymbol{X}}_{ik}},\overrightarrow {\boldsymbol{t_{ij}}} } \right) = \sin \angle \left( {{{\boldsymbol{X}}_{jk}},{{\overrightarrow {\boldsymbol{t}} }_{ij}}} \right)$. We can further analyze the magnitude of the elements in (\ref{eq:PPO3}):
\begin{equation}
  \left\| {{{\boldsymbol{\alpha }}_{ijk}}} \right\|\left\| {{{\boldsymbol{X}}_{jk}}} \right\| = \left\| {{{\boldsymbol{X}}_{jk}}} \right\|\left\| {{{\boldsymbol{X}}_{ik}}} \right\|\sin \angle \left( {{{\boldsymbol{R}}_{ij}}{{\boldsymbol{X}}_{ik}},{{\overrightarrow {\boldsymbol{t}} }_{ij}}} \right).
    \label{eq:PPOconfirm1}
\end{equation}
\begin{equation}
  \left\| {{{\boldsymbol{\beta }}_{ijk}}} \right\|  \left\| {{\boldsymbol{R}}_{ij}}{{\boldsymbol{X}}_{ik}} \right\|    = \left\| {{{\boldsymbol{X}}_{jk}}} \right\|\left\| {{{\boldsymbol{X}}_{ik}}} \right\|\sin \angle \left( {{{\boldsymbol{X}}_{jk}},{{\overrightarrow {\boldsymbol{t}} }_{ij}}} \right).
    \label{eq:PPOconfirm2}
\end{equation}

Therefore, $\left\| {{{\boldsymbol{\alpha }}_{ijk}}} \right\|\left\| {{{\boldsymbol{X}}_{jk}}} \right\| = \left\| {{{\boldsymbol{\beta }}_{ijk}}} \right\|\left\| {{{\boldsymbol{R}}_{ij}}{{\boldsymbol{X}}_{ik}}} \right\|$. From the above, it can be concluded that when the corresponding cameras' relative motion is pure rotation, we have
\begin{equation}
  \left\| {{{\boldsymbol{\alpha }}_{ijk}}} \right\|{{\boldsymbol{X}}_{jk}} = \left\| {{{\boldsymbol{\beta }}_{ijk}}} \right\|{{\boldsymbol{R}}_{ij}}{{\boldsymbol{X}}_{ik}}.
    \label{eq:PPO_PureRotation}
\end{equation}

Thus, the constraint (\ref{eq:PPO3}) holds regardless of whether the corresponding cameras' motion is pure rotation. It was first raised in \cite{QiCai_IJCV} and referred to as the PPO constraint.

We will next explore how camera translation can be represented on the rotation manifold.

Similar to the derivation process in \cite{cai2024linearrelativeposeestimation}, taking the \( j \)-th view as an example, left-multiplying both sides of the PPO constraint (\ref{eq:PPO3})  by $\left\| {{{\boldsymbol{\theta }}_{ijk}}} \right\|{\left[ {{{\boldsymbol{X}}_{jk}}} \right]_ \times }$, the linear relative translation (LiRT) constraint is constructed as follows:
\begin{equation}
 \left\| {{{\boldsymbol{\theta }}_{ijk}}} \right\|\left\| {{{\boldsymbol{\beta }}_{ijk}}} \right\|{{\boldsymbol{\theta }}_{ijk}} - {\left\| {{{\boldsymbol{\theta }}_{ijk}}} \right\|^2}{{\boldsymbol{\beta }}_{ijk}} = 0.
    \label{eq:TwoViewLiGT}
\end{equation}

It indicates that the LiRT constraint only ensures that the direction of $\boldsymbol{\theta }_{ijk}$ is the same as that of $\boldsymbol{\beta }_{ijk}$. According to the necessary and sufficient condition for $z_{ik}^C$ to satisfy the chirality constraint, the LiRT constraint only imposes a chirality constraint on view \( i \) but does not impose the chirality constraint for the other view. Multiplying both sides of (\ref{eq:PPO3}) by $\left\| {{{\boldsymbol{\theta }}_{ijk}}} \right\|$, a linear translation constraint that satisfies the chirality constraint in both views \cite{QiCai_TPAMI} can be constructed:
\begin{equation}
 \left\| {{{\boldsymbol{\theta }}_{ijk}}} \right\|\left\| {{{\boldsymbol{\alpha }}_{ijk}}} \right\|{{\boldsymbol{X}}_{jk}} - \left\| {{{\boldsymbol{\theta }}_{ijk}}} \right\|\left\| {{{\boldsymbol{\beta }}_{ijk}}} \right\|{{\boldsymbol{R}}_{ij}}{{\boldsymbol{X}}_{ik}} - {\left\| {{{\boldsymbol{\theta }}_{ijk}}} \right\|^2}{\overrightarrow {\boldsymbol{t}} _{ij}} = 0.
    \label{eq:TwoViewPPO}
\end{equation}
Since the directions of  \( {\boldsymbol{\theta }}_{ijk} \), ${{{\boldsymbol{\alpha }}_{ijk}}}$, and \( {{\boldsymbol{\beta }}_{ijk}} \) are the same, the following property holds:
\begin{equation}
    \left\| {{{\boldsymbol{\theta }}_{ijk}}} \right\|\left\| {{{\boldsymbol{\alpha }}_{ijk}}} \right\| = {\boldsymbol{\theta }}_{ijk}^T{\left[ {{{\boldsymbol{R}}_{ij}}{{\boldsymbol{X}}_{ik}}} \right]_ \times }{\overrightarrow {\boldsymbol{t}} _{ij}}.
    \label{eq:modulus3}
\end{equation}
\begin{equation}
    \left\| {{{\boldsymbol{\theta }}_{ijk}}} \right\|\left\| {{{\boldsymbol{\beta }}_{ijk}}} \right\| = {\boldsymbol{\theta }}_{ijk}^T{\left[ {{{\boldsymbol{X}}_{jk}}} \right]_ \times }{\overrightarrow {\boldsymbol{t}} _{ij}}.
    \label{eq:modulus4}
\end{equation}

Based on properties (\ref{eq:modulus3}) and (\ref{eq:modulus4}), the above relationship can be expressed as follows:
\begin{equation}
{{\boldsymbol{P}}_{ijk}}{\overrightarrow {\boldsymbol{t}} _{ij}} = 0,
    \label{eq:TwoViewPPOLinear}
\end{equation}
where
\begin{equation}
\scalebox{0.92}{${{\boldsymbol{P}}_{ijk}} =  - {{\boldsymbol{X}}_{jk}}{\boldsymbol{\theta }}_{ijk}^T{\left[ {{{\boldsymbol{R}}_{ij}}{{\boldsymbol{X}}_{ik}}} \right]_ \times } + {{\boldsymbol{R}}_{ij}}{{\boldsymbol{X}}_{ik}}{\boldsymbol{\theta }}_{ijk}^T{\left[ {\boldsymbol{X}}_{jk} \right]_ \times } + {\left\| {{{\boldsymbol{\theta }}_{ijk}}} \right\|^2}{\boldsymbol{I}}.$
}
\label{eq:TwoViewPPOLinear2}
\end{equation}

\newtheorem{theorem}{Theorem}
\begin{theorem}
${{\boldsymbol{P}}_{ijk}} = {{\boldsymbol{\theta }}_{ijk}}{\boldsymbol{\theta }}_{ijk}^T$.
\label{theorem:theorem1}
\end{theorem}

\begin{proof}
We will express (\ref{eq:TwoViewPPOLinear2}) in the following form:
\begin{equation}
{{\boldsymbol{P}}_{ijk}} = {\boldsymbol{N}} + {\left\| {{{\boldsymbol{\theta }}_{ijk}}} \right\|^2}{\boldsymbol{I}},
    \label{eq:ActionP}
\end{equation}
where ${\boldsymbol{N}} =  - {{\boldsymbol{X}}_{jk}}{\boldsymbol{\theta }}_{ijk}^T{\left[ {{{\boldsymbol{R}}_{ij}}{{\boldsymbol{X}}_{ik}}} \right]_ \times } + {{\boldsymbol{R}}_{ij}}{{\boldsymbol{X}}_{ik}}{\boldsymbol{\theta }}_{ijk}^T{\left[ {\boldsymbol{X}}_{jk} \right]_ \times }$. Substituting ${{\boldsymbol{\theta }}_{ijk}} = {\left[ {{{\boldsymbol{R}}_{ij}}{{\boldsymbol{X}}_{ik}}} \right]_ \times }{{\boldsymbol{X}}_{jk}}$ into $\boldsymbol{N}$, we get
\begin{equation}
\scalebox{0.74}{$\begin{aligned}
{\boldsymbol{N}} = & - {{\boldsymbol{X}}_{jk}}{\left( \left[ {\boldsymbol{R}}_{ij}{\boldsymbol{X}}_{ik} \right]_\times {\boldsymbol{X}}_{jk} \right)^T}{\left[ {\boldsymbol{R}}_{ij}{\boldsymbol{X}}_{ik} \right]_\times}  + {\boldsymbol{R}}_{ij}{\boldsymbol{X}}_{ik}{\left( \left[ {\boldsymbol{R}}_{ij}{\boldsymbol{X}}_{ik} \right]_\times {\boldsymbol{X}}_{jk} \right)^T}{\left[ {\boldsymbol{X}}_{jk} \right]_\times} \\
= & {\boldsymbol{X}}_{jk}{\boldsymbol{X}}_{jk}^T{\left[ {\boldsymbol{R}}_{ij}{\boldsymbol{X}}_{ik} \right]_\times}{\left[ {\boldsymbol{R}}_{ij}{\boldsymbol{X}}_{ik} \right]_\times}  - {\boldsymbol{R}}_{ij}{\boldsymbol{X}}_{ik}{\boldsymbol{X}}_{jk}^T{\left[ {\boldsymbol{R}}_{ij}{\boldsymbol{X}}_{ik} \right]_\times}{\left[ {\boldsymbol{X}}_{jk} \right]_\times}.
\end{aligned}$
}
\label{eq:divideE1}
\end{equation}

According to the equality (\ref{eq:fundamental1}), we further express (\ref{eq:divideE1}) as:
\begin{equation}
{\boldsymbol{N}} = \left( {{\boldsymbol{F}} - {\boldsymbol{G}}} \right) - \left( {{\boldsymbol{U}} - {\boldsymbol{V}}} \right),
\label{eq:devidee2}
\end{equation}

\noindent
where, 
\begin{equation}
\begin{array}{c}
{{\boldsymbol{X}}_{jk}}{\boldsymbol{X}}_{jk}^T{\left[ {{{\boldsymbol{R}}_{ij}}{{\boldsymbol{X}}_{ik}}} \right]_ \times }{\left[ {{{\boldsymbol{R}}_{ij}}{{\boldsymbol{X}}_{ik}}} \right]_ \times } = {\boldsymbol{F}} - {\boldsymbol{G}},\\
{{\boldsymbol{R}}_{ij}}{{\boldsymbol{X}}_{ik}}{\boldsymbol{X}}_{jk}^T{\left[ {{{\boldsymbol{R}}_{ij}}{{\boldsymbol{X}}_{ik}}} \right]_ \times }{\left[ {{{\boldsymbol{X}}_{jk}}} \right]_ \times } = {\boldsymbol{U}} - {\boldsymbol{V}}.
\end{array}
\label{eq:dividee3}
\end{equation}
\begin{equation}
\scalebox{0.84}{$\begin{array}{c}
{\boldsymbol{F}} = {{\boldsymbol{X}}_{jk}}{\boldsymbol{X}}_{jk}^T{{\boldsymbol{R}}_{ij}}{{\boldsymbol{X}}_{ik}}{\left( {{{\boldsymbol{R}}_{ij}}{{\boldsymbol{X}}_{ik}}} \right)^T},{\boldsymbol{G}} = {{\boldsymbol{X}}_{jk}}{\boldsymbol{X}}_{jk}^T{\left\| {{{\boldsymbol{R}}_{ij}}{{\boldsymbol{X}}_{ik}}} \right\|^2}{\boldsymbol{I}},\\
{\boldsymbol{U}} = {{\boldsymbol{R}}_{ij}}{{\boldsymbol{X}}_{ik}}{\boldsymbol{X}}_{jk}^T{{\boldsymbol{X}}_{jk}}{\left( {{{\boldsymbol{R}}_{ij}}{{\boldsymbol{X}}_{ik}}} \right)^T},{\boldsymbol{V}} = {{\boldsymbol{R}}_{ij}}{{\boldsymbol{X}}_{ik}}{\boldsymbol{X}}_{jk}^T{\left( {{{\boldsymbol{R}}_{ij}}{{\boldsymbol{X}}_{ik}}} \right)^T}{{\boldsymbol{X}}_{jk}}{\boldsymbol{I}}.
\end{array}$
}
\label{eq:dividee4}
\end{equation}

\noindent
Substituting (\ref{eq:dividee4}) into (\ref{eq:devidee2}) and changing the order of operations, (\ref{eq:devidee2}) can be rewritten as: ${\boldsymbol{N}} = \left( {{\boldsymbol{F}} - {\boldsymbol{U}}} \right) - \left( {{\boldsymbol{G}} - {\boldsymbol{V}}} \right)$. According to the equality (\ref{eq:fundamental2}),  ${{\boldsymbol{F}} - {\boldsymbol{U}}}$ and ${{\boldsymbol{G}} - {\boldsymbol{V}}}$ can be simplified as follows:
\begin{equation}
\scalebox{0.87}{$\begin{aligned}
{\boldsymbol{F}} - {\boldsymbol{U}} = &{{\boldsymbol{X}}_{jk}}{\boldsymbol{X}}_{jk}^T{{\boldsymbol{R}}_{ij}}{{\boldsymbol{X}}_{ik}}{\left( {{{\boldsymbol{R}}_{ij}}{{\boldsymbol{X}}_{ik}}} \right)^T} - {{\boldsymbol{R}}_{ij}}{{\boldsymbol{X}}_{ik}}{\boldsymbol{X}}_{jk}^T{{\boldsymbol{X}}_{jk}}{\left( {{{\boldsymbol{R}}_{ij}}{{\boldsymbol{X}}_{ik}}} \right)^T}\\
 = & {\boldsymbol{X}}_{jk}^T{{\boldsymbol{R}}_{ij}}{{\boldsymbol{X}}_{ik}}{{\boldsymbol{X}}_{jk}}{\left( {{{\boldsymbol{R}}_{ij}}{{\boldsymbol{X}}_{ik}}} \right)^T} - {{\boldsymbol{R}}_{ij}}{{\boldsymbol{X}}_{ik}}{\boldsymbol{X}}_{jk}^T{{\boldsymbol{X}}_{jk}}{\left( {{{\boldsymbol{R}}_{ij}}{{\boldsymbol{X}}_{ik}}} \right)^T}\\
 = & \left( {{\boldsymbol{X}}_{jk}^T{{\boldsymbol{R}}_{ij}}{{\boldsymbol{X}}_{ik}}{\boldsymbol{I}} - {{\boldsymbol{R}}_{ij}}{{\boldsymbol{X}}_{ik}}{\boldsymbol{X}}_{jk}^T} \right){{\boldsymbol{X}}_{jk}}{\left( {{{\boldsymbol{R}}_{ij}}{{\boldsymbol{X}}_{ik}}} \right)^T}\\
 = &  - {\left[ {{{\boldsymbol{X}}_{jk}}} \right]_ \times }{\left[ {{{\boldsymbol{R}}_{ij}}{{\boldsymbol{X}}_{ik}}} \right]_ \times }{{\boldsymbol{X}}_{jk}}{\left( {{{\boldsymbol{R}}_{ij}}{{\boldsymbol{X}}_{ik}}} \right)^T}\\
 = & {\left[ {{{\boldsymbol{\theta }}_{ijk}}} \right]_ \times }{{\boldsymbol{X}}_{jk}}{\left( {{{\boldsymbol{R}}_{ij}}{{\boldsymbol{X}}_{ik}}} \right)^T} - {\left[ {{{\boldsymbol{R}}_{ij}}{{\boldsymbol{X}}_{ik}}} \right]_ \times }{\left[ {{{\boldsymbol{X}}_{jk}}} \right]_ \times }{{\boldsymbol{X}}_{jk}}{\left( {{{\boldsymbol{R}}_{ij}}{{\boldsymbol{X}}_{ik}}} \right)^T}\\
 = & {\left[ {{{\boldsymbol{\theta }}_{ijk}}} \right]_ \times }{{\boldsymbol{X}}_{jk}}{\left( {{{\boldsymbol{R}}_{ij}}{{\boldsymbol{X}}_{ik}}} \right)^T},
\end{aligned}$
}
\label{eq:dividee5}
\end{equation}
\begin{equation}
\scalebox{0.87}{$\begin{aligned}
{\boldsymbol{G}} - {\boldsymbol{V}} =& {{\boldsymbol{X}}_{jk}}{\boldsymbol{X}}_{jk}^T{\left\| {{{\boldsymbol{R}}_{ij}}{{\boldsymbol{X}}_{ik}}} \right\|^2}{\boldsymbol{I}} - {{\boldsymbol{R}}_{ij}}{{\boldsymbol{X}}_{ik}}{\boldsymbol{X}}_{jk}^T{\left( {{{\boldsymbol{R}}_{ij}}{{\boldsymbol{X}}_{ik}}} \right)^T}{{\boldsymbol{X}}_{jk}}{\boldsymbol{I}}\\
 =& {\left( {{{\boldsymbol{R}}_{ij}}{{\boldsymbol{X}}_{ik}}} \right)^T}{{\boldsymbol{R}}_{ij}}{{\boldsymbol{X}}_{ik}}{{\boldsymbol{X}}_{jk}}{\boldsymbol{X}}_{jk}^T - {\left( {{{\boldsymbol{R}}_{ij}}{{\boldsymbol{X}}_{ik}}} \right)^T}{{\boldsymbol{X}}_{jk}}{{\boldsymbol{R}}_{ij}}{{\boldsymbol{X}}_{ik}}{\boldsymbol{X}}_{jk}^T\\
 =& \left( {{{\left( {{{\boldsymbol{R}}_{ij}}{{\boldsymbol{X}}_{ik}}} \right)}^T}{{\boldsymbol{R}}_{ij}}{{\boldsymbol{X}}_{ik}}{{\boldsymbol{X}}_{jk}} - {{\left( {{{\boldsymbol{R}}_{ij}}{{\boldsymbol{X}}_{ik}}} \right)}^T}{{\boldsymbol{X}}_{jk}}{{\boldsymbol{R}}_{ij}}{{\boldsymbol{X}}_{ik}}} \right){\boldsymbol{X}}_{jk}^T\\
 =& {\left[ {{{\boldsymbol{R}}_{ij}}{{\boldsymbol{X}}_{ik}}} \right]_ \times }{\left[ {{{\boldsymbol{X}}_{jk}}} \right]_ \times }{{\boldsymbol{R}}_{ij}}{{\boldsymbol{X}}_{ik}}{\boldsymbol{X}}_{jk}^T\\
 =& {\left[ {{{\boldsymbol{\theta }}_{ijk}}} \right]_ \times }{{\boldsymbol{R}}_{ij}}{{\boldsymbol{X}}_{ik}}{\boldsymbol{X}}_{jk}^T + {\left[ {{{\boldsymbol{X}}_{jk}}} \right]_ \times }{\left[ {{{\boldsymbol{R}}_{ij}}{{\boldsymbol{X}}_{ik}}} \right]_ \times }{{\boldsymbol{R}}_{ij}}{{\boldsymbol{X}}_{ik}}{\boldsymbol{X}}_{jk}^T\\
 =& {\left[ {{{\boldsymbol{\theta }}_{ijk}}} \right]_ \times }{{\boldsymbol{R}}_{ij}}{{\boldsymbol{X}}_{ik}}{\boldsymbol{X}}_{jk}^T.
\end{aligned}$
}
\label{eq:dividee6}
\end{equation}
Based on (\ref{eq:dividee5}) and (\ref{eq:dividee6}), \(\boldsymbol{N}\) can be rewritten as:
\begin{equation}
\begin{aligned}
{\boldsymbol{N}} =& {\left[ {{{\boldsymbol{\theta }}_{ijk}}} \right]_ \times }{{\boldsymbol{X}}_{jk}}{\left( {{{\boldsymbol{R}}_{ij}}{{\boldsymbol{X}}_{ik}}} \right)^T} - {\left[ {{{\boldsymbol{\theta }}_{ijk}}} \right]_ \times }{{\boldsymbol{R}}_{ij}}{{\boldsymbol{X}}_{ik}}{\boldsymbol{X}}_{jk}^T\\
 =& {\left[ {{{\boldsymbol{\theta }}_{ijk}}} \right]_ \times }\left( {{{\boldsymbol{X}}_{jk}}{{\left( {{{\boldsymbol{R}}_{ij}}{{\boldsymbol{X}}_{ik}}} \right)}^T} - {{\boldsymbol{R}}_{ij}}{{\boldsymbol{X}}_{ik}}{\boldsymbol{X}}_{jk}^T} \right)\\
 =& {\left[ {{{\boldsymbol{\theta }}_{ijk}}} \right]_ \times }{\left[ {{{\boldsymbol{\theta }}_{ijk}}} \right]_ \times }
\end{aligned}
\label{eq:devidee7}
\end{equation}
Substituting (\ref{eq:devidee7}) into (\ref{eq:ActionP}), we get
\begin{equation}
\begin{aligned}
{{\boldsymbol{P}}_{ijk}} =& {\left[ {{{\boldsymbol{\theta }}_{ijk}}} \right]_ \times }{\left[ {{{\boldsymbol{\theta }}_{ijk}}} \right]_ \times } + {\left\| {{{\boldsymbol{\theta }}_{ijk}}} \right\|^2}{\boldsymbol{I}}\\
 =& {{\boldsymbol{\theta }}_{ijk}}{\boldsymbol{\theta }}_{ijk}^T.
\end{aligned}
    \label{eq:ActionP2}
\end{equation}
\end{proof}
We define the joint observation matrix as 
${{\boldsymbol{P}}^J_{ij}} = {\left[ {{\boldsymbol{P}}_{ij1},{\boldsymbol{P}}_{ij2},...,{\boldsymbol{P}}_{ijk}} \right]^T}$. 
As ${{\boldsymbol{\theta }}_{ijk}}{\boldsymbol{\theta }}_{ijk}^T{\boldsymbol{x}} = 0$ and ${\boldsymbol{\theta }}_{ijk}^T{\boldsymbol{x}} = 0$ have the same solution, ${{\boldsymbol{P}}^J_{ij}}{\boldsymbol{x}} = {\boldsymbol{0}}$ is equivalent to ${\left[ {{{\boldsymbol{\theta }}_{ij1}},{{\boldsymbol{\theta }}_{ij2}},...,{{\boldsymbol{\theta }}_{ijk}}} \right]^T}{\boldsymbol{x}} = 0$, i.e.,
\begin{equation}
\begin{aligned}
\left[ {{{\boldsymbol{\theta }}_{ij1}},{{\boldsymbol{\theta }}_{ij2}},...,{{\boldsymbol{\theta }}_{ijk}}} \right]\left[ \begin{array}{l}
{\boldsymbol{\theta }}_{ij1}^T\\
{\boldsymbol{\theta }}_{ij2}^T\\
...\\
{\boldsymbol{\theta }}_{ijk}^T
\end{array} \right]{\boldsymbol{x}} = 0.
\end{aligned}
    \label{eq:observeDerive1}
\end{equation} 

Let the simplified joint observation matrix be denoted as
\begin{equation}
\begin{aligned}
{{\boldsymbol{P}}^S_{ij}} = \left[ {{{\boldsymbol{\theta }}_{ij1}},{{\boldsymbol{\theta }}_{ij2}},...,{{\boldsymbol{\theta }}_{ijk}}} \right]\left[ \begin{array}{l}
{\boldsymbol{\theta }}_{ij1}^T\\
{\boldsymbol{\theta }}_{ij2}^T\\
...\\
{\boldsymbol{\theta }}_{ijk}^T
\end{array} \right] = \sum\limits_{k = 1}^{{m_{ij}}} {{{\boldsymbol{P}}_{ijk}}} .
\end{aligned}
    \label{eq:observeDerive2}
\end{equation} 
It is evident that $\boldsymbol{P}^S_{ij}$ and  $\boldsymbol{P}_{ijk}$ only depend on the relative rotation, and the relative translation ${{\boldsymbol{t}}_{ij}}$  is a solution to the system of equations ${{\boldsymbol{P}}^S_{ij}}{\boldsymbol{x}} = {\boldsymbol{0}}$, which therefore implies that the relative translation can be represented on the rotation manifold. We next further analyze how $\boldsymbol{P}^S_{ij}$ restricts relative translation under different scene structures.

\subsection{Properties of Translation Solution Space}

\newtheorem{corollary}{Corollary}

\begin{corollary}\label{corollary:one}
The solution space of ${{{\boldsymbol{P}}^S_{ij}}}$ is determined by its rank. Specifically, three different ranks correspond to three distinct scene structure cases.

    {Case 1:} $ rank\left( {{{\boldsymbol{P}}^S_{ij}}} \right) = 0 \Leftrightarrow$ The corresponding cameras' motion is pure rotation, or all observed 3D points lie on the camera translation baseline or at infinity. We refer to this scene structure as \textit{PR/B/I} for short.

    {Case 2:} $rank\left( {{{\boldsymbol{P}}^S_{ij}}} \right) = 1 \Leftrightarrow$ The corresponding cameras' motion is non-pure rotation, and at least one scene point does not lie on the translation baseline or at infinity. Furthermore, all non-infinite 3D points lie on the same plane as the corresponding cameras. In this scene structure, which we refer to as \textit{Holoplane}, the corresponding solution space of ${{\boldsymbol{P}}^S_{ij}}{\boldsymbol{x}} = {\boldsymbol{0}}$ is the plane. 

    {Case 3: }$ rank\left( {{{\boldsymbol{P}}^S_{ij}}} \right) = 2 \Leftrightarrow$ The corresponding cameras' motion is non-pure rotation, and all the non-infinite 3D points and cameras do not lie on a single plane. We refer to this as \textit{RankRegular} scene structure.

\end{corollary}

\begin{proof}
    The proofs are as follows:
    
    {Case 1:} $ \Rightarrow $: 
    When $rank\left( {{{\boldsymbol{P}}^S_{ij}}} \right) = 0$, ${{\boldsymbol{P}}^S_{ij}} = {\boldsymbol{0}}$. It follows from (\ref{eq:observeDerive2}) that $tr\left( {{{\boldsymbol{P}}^S_{ij}}} \right) = \sum\limits_{k = 1}^{{m_{ij}}} {{{\left\| {{{\boldsymbol{\theta }}_{ijk}}} \right\|}^2}}  = 0$, which indicates that for any 3D points \( k \), we have ${{\boldsymbol{\theta }}_{ijk}} = {\boldsymbol{0}}$. According to \cite{QiCai_IJCV}, the scene structure is \textit{PR/B/I}.
        
        $ \Leftarrow $: 
        It follows from \cite{QiCai_IJCV} that all observed 3D points \( k \) satisfy ${{\boldsymbol{\theta }}_{ijk}} = {\boldsymbol{0}}$, therefore from (\ref{eq:ActionP2}), it can be inferred that $\boldsymbol{P}_{ijk}=0$. And it follows from (\ref{eq:observeDerive2}) that ${{\boldsymbol{P}}^S_{ij}} = \sum\limits_{k = 1}^{{m_{ij}}} {{{\boldsymbol{P}}_{ijk}}}  = 0$. Therefore, $rank\left( {{{\boldsymbol{P}}^S_{ij}}} \right) = 0$. 
        \hfill
        
        {Case 2:}
        $ \Rightarrow $:
        If $rank\left( {{{\boldsymbol{P}}^S_{ij}}} \right) = 1$, then there exists at least one 3D point $k$ does not lie on the translation baseline or at infinity that satisfies ${{\boldsymbol{\theta }}_{ijk}} \neq {\boldsymbol{0}}$.
        For $r \in \{1, 2, \ldots, m_{ij}\}, \, r \neq k$, $\boldsymbol{\theta}_{ijr}$ has two possible sub-cases:
        \begin{itemize}
            \item \( \boldsymbol{\theta}_{ijr} \neq \boldsymbol{0} \) and \( \boldsymbol{\theta}_{ijr} \) is parallel to \( \boldsymbol{\theta}_{ijk} \). As $\boldsymbol{\theta}_{ijr} \perp \boldsymbol{X}_{jr}$, so $\boldsymbol{\theta}_{ijk} \perp z_{jr}^C \boldsymbol{X}_{jr}$. It can be inferred that 3D point $r$ is neither at infinity or on the relative translation baseline \cite{QiCai_IJCV}; it lies on a plane with normal vector $\boldsymbol{\theta}_{ijk}$.
            \item \( \boldsymbol{\theta}_{ijr} = \boldsymbol{0} \). 3D point $r$ lies on the camera translation baseline or at infinity. 
        \end{itemize}

        It can be concluded that at least one 3D point does not lie on the relative translation baseline or at infinity, and all non-infinite 3D points lie on the plane with normal vector $\boldsymbol{\theta}_{ijk}$. Since the corresponding cameras also lie on the plane with normal vector $\boldsymbol{\theta}_{ijk}$, it follows that all non-infinite 3D points lie on the same plane as the corresponding cameras.

        $ \Leftarrow $: 
        Under the given conditions, it follows that the rank of $\boldsymbol{P}^S_{ij}$ is at least $1$, and there exists a 3D point $k$ satisfies ${{\boldsymbol{\theta }}_{ijk}} \neq {\boldsymbol{0}}$. For other 3D points, we denote one of them as $r$, and there are three possible positions for the 3D points.

        \begin{itemize}
            \item At infinity, then ${\boldsymbol{\theta }}_{ijr} = \boldsymbol{0}$.
            \item Lie on the relative translation baseline, then $\boldsymbol{\theta}_{ijr} = \boldsymbol{0}$.
            \item Lie on the same plane as the 3D point $k$ and corresponding cameras. It follows that $\left[\boldsymbol{X}_{jk}\right]_\times \boldsymbol{t}_{ij} \text{ and } \left[\boldsymbol{X}_{jr}\right]_\times \boldsymbol{t}_{ij}$ are parallel. Based on (\ref{eq:crosspruduct1}), it's shown that $\boldsymbol{\theta}_{ijk}$ is parallel to $\boldsymbol{\theta}_{ijr}$.
        \end{itemize}
        
        Since $rank\left( {{{\boldsymbol{P}}^S_{ij}}} \right) = rank\left( {\left[ {{{\boldsymbol{\theta }}_{ij1}},...,{{\boldsymbol{\theta }}_{ij{m_{ij}}}}} \right]} \right)$, none of the positions mentioned above will increase the rank of $\boldsymbol{P}^S_{ij}$. Therefore, we have $rank\left( {{{\boldsymbol{P}}^S_{ij}}} \right) = rank\left( {\left[ {{{\boldsymbol{\theta }}_{ij1}},...,{{\boldsymbol{\theta }}_{ij{m_{ij}}}}} \right]} \right)=rank\left( { {{{\boldsymbol{\theta }}_{ijk}}} } \right)=1$.

        In \textit{Holoplane}, the rank of the solution space of \( \boldsymbol{P}^S_{ij} \boldsymbol{x} = \boldsymbol{0} \) is 2. For any 3D points \( k \) not located on the relative translation baseline or at infinity satisfies $\boldsymbol{P}^S_{ij}\left( \boldsymbol{R}_{ij} \boldsymbol{X}_{ik} \right) = \boldsymbol{0}$ and $\boldsymbol{P}^S_{ij} \boldsymbol{X}_{jk} = \boldsymbol{0}$. Therefore, the solution space of \( \boldsymbol{P}^S_{ij} \boldsymbol{x} = \boldsymbol{0} \) is \( \left\langle \boldsymbol{R}_{ij} \boldsymbol{X}_{ij}, \boldsymbol{X}_{jk} \right\rangle \). Since \( \boldsymbol{t}_{ij} \) also lies on \( \left\langle \boldsymbol{R}_{ij} \boldsymbol{X}_{ij}, \boldsymbol{X}_{jk} \right\rangle \), the solution space is the plane containing corresponding cameras, and all 3D points that are not at infinity. \hfill
        
        {Case 3:} 
        $\Rightarrow$ : 
        When ${rank}\left({{{\boldsymbol{P}}^S_{ij}}}\right)=2$, since ${rank}\left({{{\boldsymbol{P}}^S_{ij}}}\right) = {rank}\left(\left[\boldsymbol{\theta}_{i j 1}, \ldots, \boldsymbol{\theta}_{i j m_{i j}}\right]\right)$, it follows that there are at least two 3D points $k$ and $r$ such that $\boldsymbol{\theta}_{i j k}, \boldsymbol{\theta}_{i j r} \neq \boldsymbol{0}$ and $\boldsymbol{\theta}_{i j k}$ and $\boldsymbol{\theta}_{i j r}$ are not parallel. Since $\left[\boldsymbol{X}_{j k}\right]_\times \boldsymbol{t}_{i j}$ and $\left[\boldsymbol{X}_{j r}\right]_\times \boldsymbol{t}_{i j}$ are parallel to $\boldsymbol{\theta}_{i j k}$ and $\boldsymbol{\theta}_{i j r}$, respectively, it follows that $\left[\boldsymbol{X}_{j k}\right]_\times \boldsymbol{t}_{i j}$ and $\left[\boldsymbol{X}_{j r}\right]_\times \boldsymbol{t}_{i j}$ are not parallel. Since $\left[\boldsymbol{X}_{j k}\right]_{\times} \boldsymbol{t}_{i j}$ and $\left[\boldsymbol{X}_{j r}\right]_{\times} \boldsymbol{t}_{i j}$ are the normal vectors of the planes $\left\langle\boldsymbol{X}_{j k}, \boldsymbol{t}_{i j}\right\rangle$ and $\left\langle\boldsymbol{X}_{j r}, \boldsymbol{t}_{i j}\right\rangle$, respectively, it follows that $\left\langle\boldsymbol{X}_{j k}, \boldsymbol{t}_{i j}\right\rangle \neq \left\langle\boldsymbol{X}_{j r}, \boldsymbol{t}_{i j}\right\rangle$. This means that the vectors $\boldsymbol{X}_{j k}, \boldsymbol{X}_{j r}$, and $\boldsymbol{t}_{i j}$ are not coplanar. Consequently, all the non-infinite 3D points and cameras do not lie on a single plane.

        $\Leftarrow$: 
        Since $\boldsymbol{P}^S_{ij}$ is a $3 \times 3$ matrix and satisfies $\boldsymbol{P}^S_{ij} \boldsymbol{t}_{ij} = \boldsymbol{0}$, it follows that ${rank}\left(\boldsymbol{P}^S_{ij}\right) \leq 2$. Under the given conditions, the scene structure differs from both \textit{PR/B/I} and \textit{Holoplane}, so it can be concluded that ${rank}\left(\boldsymbol{P}^S_{ij}\right) \geq 2$. From the above, it follows that ${rank}\left(\boldsymbol{P}^S_{ij}\right) = 2$.
\end{proof}
The work \cite{QiCai_IJCV} proposed a method for identifying the scene structure of \textit{PR/B/I}. As indicated in \cite{MultipleViewGeometry,Maybank90}, a specific type of \textit{RankRegular} scene structure that all points lie on a line can cause singularities in 3D pose initialization, and we refer to it as \textit{RankRegular-line}. In the sequel, we refer to the scene structures of \textit{Holoplane} and \textit{RankRegular-line} collectively  as \textit{RotationSingular} structure.

Here, we further provide two detection matrices for identifying {\textit{RotationSingular} structures}:
\begin{equation} 
    \begin{aligned}
    \boldsymbol{G}_{i j}^i = \sum_k^{m_{i j}} \left(\boldsymbol{X}_{i k} \boldsymbol{X}_{i k}^T \right), \quad \boldsymbol{G}_{i j}^j = \sum_k^{m_{i j}} \left(\boldsymbol{X}_{j k} \boldsymbol{X}_{j k}^T \right).
    \end{aligned} \label{eq:case2IdentifyingMatrixBasic}
\end{equation}

\begin{corollary}\label{corollary:Two}
    The scene structure is \textit{RotationSingular}
    \begin{equation}
        \begin{aligned}
            \Leftrightarrow max \left({rank}\left(\boldsymbol{G}_{i j}^i\right),{rank}\left(\boldsymbol{G}_{i j}^j\right)\right)  < 3.
        \end{aligned} \label{eq:case2IdentifyingMatrix}
    \end{equation}
\end{corollary}

\begin{proof}
    $\Rightarrow$: From {Corollary \ref{corollary:one}}, when the scene structure is in the form of {\textit{Holoplane}}, there exists \( k \in \{1, 2, \ldots, m_{ij}\} \) such that for all \( r \in \{1, 2, \ldots, m_{ij}\} \), we have: $\boldsymbol{\theta}_{i j k} \perp \boldsymbol{X}_{j r}, \boldsymbol{\theta}_{i j k} \perp \boldsymbol{R}_{i j} \boldsymbol{X}_{i r}.$ Therefore, we have
    \begin{equation}
        \left[\boldsymbol{X}_{j 1}, \boldsymbol{X}_{j 2}, \cdots, \boldsymbol{X}_{j m_{ij}}\right]^T \boldsymbol{\theta}_{i j k} = \boldsymbol{0}.
        \label{eq:case2IdentifyingVector1}
    \end{equation}
    \begin{equation}
        \left[         \boldsymbol{R}_{i j} \boldsymbol{X}_{i 1}, \boldsymbol{R}_{i j} \boldsymbol{X}_{i 2}, \cdots, \boldsymbol{R}_{i j} \boldsymbol{X}_{i m_{ij}}\right]^T \boldsymbol{\theta}_{i j k} = \boldsymbol{0}.
        \label{eq:case2IdentifyingVector2}
    \end{equation}
    According to the fundamental equalities of matrices, (\ref{eq:case2IdentifyingVector1}) and (\ref{eq:case2IdentifyingVector2}) can be equivalently expressed as:
    \begin{equation}
    \begin{aligned}
        \boldsymbol{G}_{i j}^j \boldsymbol{\theta}_{i j k} = \boldsymbol{0}.
    \end{aligned} \label{eq:case2IdentifyingProof1}
    \end{equation}
    \begin{equation}
        \begin{aligned}
            \boldsymbol{R}_{i j}\boldsymbol{G}_{i j}^i 
            \boldsymbol{R}_{i j}^T \boldsymbol{\theta}_{i j k} = \boldsymbol{0}.
        \end{aligned} \label{eq:case2IdentifyingProof2}
    \end{equation}
    Since \(\boldsymbol{G}_{ij}^j\) and \(\boldsymbol{G}_{ij}^i\) are \(3 \times 3\) matrices, it follows from (\ref{eq:case2IdentifyingProof1}) and (\ref{eq:case2IdentifyingProof2}) that \({rank}\left(\boldsymbol{G}_{ij}^i\right) < 3\) and \({rank}\left(\boldsymbol{G}_{ij}^j\right) < 3\). 

    When the scene structure is in the form of {\textit{RankRegular-line}}, let the coordinate of the 3D point \(k\) in the camera coordinate system \(j\) be denoted as \(\boldsymbol{X}_{jk}^C=\boldsymbol{d}+s_k \boldsymbol{v}\), where \(\boldsymbol{d}, \boldsymbol{v} \in \mathbb{R}^3\). 
     When \(\boldsymbol{d} \parallel \boldsymbol{v}\), for \(\forall k, r \in \{1, 2, \ldots, m_{ij}\}\), it holds that \(\boldsymbol{X}_{jk} = \boldsymbol{X}_{jk}^C / \boldsymbol{X}_{jk(3)}^C = \boldsymbol{X}_{jr}^C / \boldsymbol{X}_{jr(3)}^C = \boldsymbol{X}_{jr}\). It is then evident that \({rank}\left(\boldsymbol{G}_{ij}^j\right) = 1\). 
    When \(\boldsymbol{d}\) and \(\boldsymbol{v}\) are not parallel, for \(\forall k \in [1, m_{ij}]\), it holds that \(\boldsymbol{X}_{jk}^C \perp [\boldsymbol{d}]_{\times} \boldsymbol{v}\). Furthermore, the observations satisfy the following relationship: $\left[
        \boldsymbol{X}_{j1}, \boldsymbol{X}_{j2}, \ldots, \boldsymbol{X}_{jm_{ij}}
        \right]^T [\boldsymbol{d}]_{\times} \boldsymbol{v} = \mathbf{0}$, thus, 
 \({rank}\left(\boldsymbol{G}_{ij}^j\right) = 2\). 
Similarly, it can be demonstrated that \({rank}\left(\boldsymbol{G}_{ij}^j\right)\) and \({rank}\left(\boldsymbol{G}_{ij}^i\right)\) are consistent. 

Therefore, when the scene structure is {\textit{RotationSingular} structures}, it can be concluded that \({rank}\left(\boldsymbol{G}_{ij}^i\right) < 3\) and \({rank}\left(\boldsymbol{G}_{ij}^j\right) < 3\). 

    $\Leftarrow$: Given that the observation satisfies (\ref{eq:case2IdentifyingMatrix}), it is evident that there exists a non-zero vector \(\boldsymbol{v}_j\) such that \(\boldsymbol{G}_{ij}^j \boldsymbol{v}_j = \mathbf{0}\). Since \({rank}\left(\boldsymbol{R}_{ij} \boldsymbol{G}_{ij}^i \boldsymbol{R}_{ij}^T\right) = {rank}\left(\boldsymbol{G}_{ij}^i\right) < 3\), there exists a non-zero vector \(\boldsymbol{v}_i\) such that \(\boldsymbol{R}_{ij} \boldsymbol{G}_{ij}^i \boldsymbol{R}_{ij}^T \boldsymbol{v}_i = \mathbf{0}\). By the fundamental equalities of matrices, these two constraints are equivalent to $\left[ \boldsymbol{R}_{ij} \boldsymbol{X}_{i1}, \boldsymbol{R}_{ij} \boldsymbol{X}_{i2}, \ldots, \boldsymbol{R}_{ij} \boldsymbol{X}_{im_{ij}} \right]^T \boldsymbol{v}_i = \mathbf{0}$ and $\left[ \boldsymbol{X}_{j1}, \boldsymbol{X}_{j2}, \ldots, \boldsymbol{X}_{jm_{ij}} \right]^T \boldsymbol{v}_j = \mathbf{0}$. Thus, for \(\forall k \in [1, m_{ij}]\), \(\boldsymbol{v}_i \perp \boldsymbol{R}_{ij} \boldsymbol{X}_{ik}\) and \(\boldsymbol{v}_j \perp \boldsymbol{X}_{jk}\). This indicates that in camera coordinate system \(j\), the projection rays from the 3D points to views \(i\) and \(j\) respectively lie on planes with normal vectors \(\boldsymbol{v}_i\) and \(\boldsymbol{v}_j\). When \(\boldsymbol{v}_i \parallel \boldsymbol{v}_j\), for \(\forall k \in [1, m_{ij}]\), \(\boldsymbol{v}_i \perp \boldsymbol{R}_{ij} \boldsymbol{X}_{ik}\) and \(\boldsymbol{v}_i \perp \boldsymbol{X}_{jk}\). Therefore, \(\boldsymbol{v}_i \perp \left(z_j^C \boldsymbol{X}_{jk} - z_i^C \boldsymbol{R}_{ij} \boldsymbol{X}_{ik}\right)\), which implies \(\boldsymbol{v}_i \perp \boldsymbol{t}_{ij}\), which implies that the scene structure corresponds to \textit{Holoplane}.
    When \(\boldsymbol{v}_i\) and \(\boldsymbol{v}_j\) are not parallel, the two planes intersect along a line, and their intersection corresponds to the 3D points \(\{\boldsymbol{X}_1^w, \boldsymbol{X}_2^w, \ldots, \boldsymbol{X}_{m_{ij}}^w\}\), which implies that the scene structure corresponds to {\textit{RankRegular-line}}.

    Therefore, when \({rank}\left(\boldsymbol{G}_{ij}^i\right) < 3\) and \({rank}\left(\boldsymbol{G}_{ij}^j\right) < 3\), the scene structure is {\textit{RotationSingular} structures}.
\end{proof}

We next further investigate the properties of reprojection error across different translation solution spaces and formulate the objective function on the rotation manifold.

\subsection{Reprojection-error Optimization on Rotation Manifold}

Firstly, we review the pose-only optimization method proposed in \cite{QiCai_TPAMI}. Let
\begin{equation}
    \boldsymbol{Y}_{ijk}^j = \left\| \left[\boldsymbol{X}_{jk}\right]_\times \boldsymbol{t}_{ij}\right\|  \boldsymbol{R}_{ij} \boldsymbol{X}_{ik} +  \left\| \boldsymbol{\theta}_{ijk} \right\| \boldsymbol{t}_{ij}. \label{eq:PA1}
\end{equation}
It should be noted that $\boldsymbol{Y}_{ijk}^j = \boldsymbol{Y}_{jik}^j$ and $\boldsymbol{Y}_{ijk(3)}^j=\|\left[\boldsymbol{R}_{ji} \boldsymbol{X}_{jk}\right]_\times \boldsymbol{t}_{ji}\|$. According to (\ref{eq:PPO1}), the two-view pose-only reprojection coordinates in view $j$ can be expressed as:
\begin{equation}
    \boldsymbol{X}_{ijk}^j = \boldsymbol{Y}_{ijk}^j / \boldsymbol{Y}_{ijk(3)}^j.
    \label{eq:PA2}
\end{equation}
And the bearing vector of pose-only reprojection coordinates can be expressed as:
\begin{equation}
    \vec{\boldsymbol{X}}_{ijk}^j = \boldsymbol{Y}_{ijk}^j / \| \boldsymbol{Y}_{ijk}^j \|.
    \label{eq:PA2bearing}
\end{equation}
Under the influence of noisy observations, the two-view pose-only reprojection residual (using view $j$ as an example) is expressed as:
\begin{equation}
    \boldsymbol{V}_{ijk}^{PA,j}\left(\boldsymbol{R}_{ij}, \boldsymbol{t}_{ij}, \tilde{\boldsymbol{X}}_{ik}, \tilde{\boldsymbol{X}}_{jk}\right) = \tilde{\boldsymbol{X}}_{jk} - \boldsymbol{X}_{ijk}^j.
    \label{eq:PA3}
\end{equation}
Observations affected by noise are highlighted in tilde markers for clarity in the reprojection residual formula. And the two-view pose-only reprojection error in the bearing vector form can be expressed as:
\begin{equation}
    \boldsymbol{V}_{ijk,bearing}^{PA,j}\left(\boldsymbol{R}_{ij}, \boldsymbol{t}_{ij}, \tilde{\boldsymbol{X}}_{ik}, \tilde{\boldsymbol{X}}_{jk}\right) = \vec{\tilde{\boldsymbol{X}}}_{jk} - \vec{\boldsymbol{X}}_{ijk}^j.
    \label{eq:PA3bearing}
\end{equation}
Fig.~\ref{fig:poseresidual} graphically illustrates the two-view pose-only reprojection residual (\ref{eq:PA3}) and its corresponding bearing vector form (\ref{eq:PA3bearing}).

\begin{figure}[t] 
    \centering 
    \includegraphics[width=0.48\textwidth]{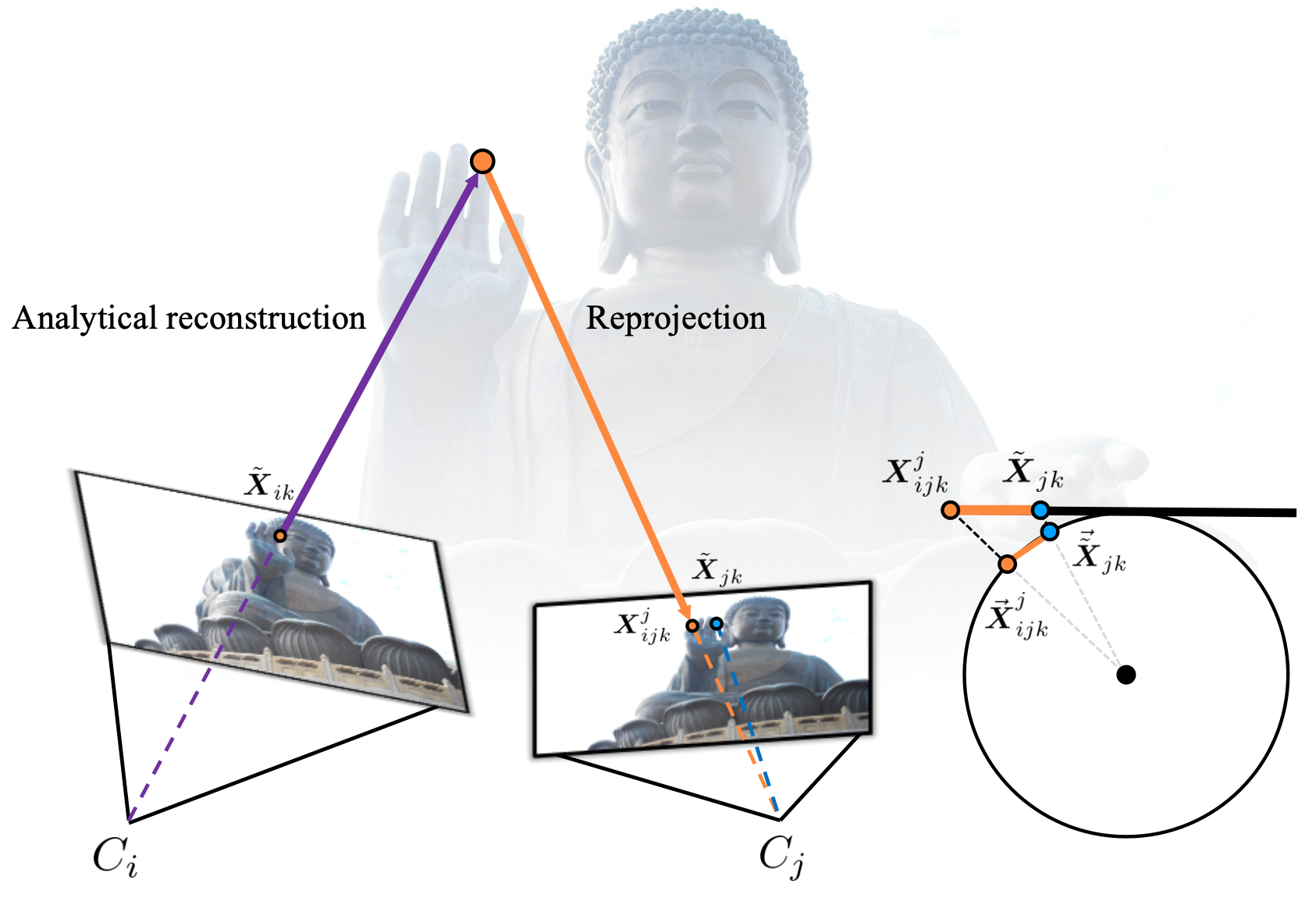} 
    \caption{Generation of projected point \( \boldsymbol{X}_{ijk}^j \), reconstructed along projection ray from \(i\)-th view and projected to \(j\)-th view. On image plane, distance between \( \boldsymbol{X}_{ijk}^j \) and \( \tilde{\boldsymbol{X}}_{jk} \) represents pose-only reprojection residual $\| \boldsymbol{V}_{ijk}^{PA,j} \|$. 
    Bearing vector of pose-only reprojection residual $\| \boldsymbol{V}_{ijk,bearing}^{PA,j} \|$ represents chord distance between \(\vec{\boldsymbol{X}}_{ijk}^j \) and \(\vec{\tilde{\boldsymbol{X}}}_{jk} \) of unit sphere centered at camera origin.
} 
    \label{fig:poseresidual} 
\end{figure}

The total two-view pose-only reprojection residual on edge $e_{ij}$ is expressed as:
\begin{equation}
    \boldsymbol{V}_{ijk}^{PA}\left(\boldsymbol{R}_{ij}, \boldsymbol{t}_{ij}, \tilde{\boldsymbol{X}}_{ik}, \tilde{\boldsymbol{X}}_{jk}\right) = 
\begin{bmatrix}
    \boldsymbol{V}_{ijk}^{PA,i} \\
\boldsymbol{V}_{ijk}^{PA,j}
\end{bmatrix}.
\label{eq:PA4}
\end{equation}

\begin{equation}
    \boldsymbol{V}_{ijk,bearing}^{PA}\left(\boldsymbol{R}_{ij}, \boldsymbol{t}_{ij}, \tilde{\boldsymbol{X}}_{ik}, \tilde{\boldsymbol{X}}_{jk}\right) = 
\begin{bmatrix}
    \boldsymbol{V}_{ijk,bearing}^{PA,i} \\
\boldsymbol{V}_{ijk,bearing}^{PA,j}
\end{bmatrix}.
\label{eq:PA4bearing}
\end{equation}

Thus, the pose-only error function can be constructed as:
\begin{equation}
    \scalebox{0.76}{$f_{PA}\left(\boldsymbol{R}_{ij}, \boldsymbol{t}_{ij}, \left\{\tilde{\boldsymbol{X}}_{ik}, \tilde{\boldsymbol{X}}_{jk}\right\}_{k=1,\ldots,m_{ij}}\right) = 
\sum_{k=1}^{m_{ij}} \left\| \boldsymbol{V}_{ijk}^{PA}\left(\boldsymbol{R}_{ij}, \boldsymbol{t}_{ij}, \tilde{\boldsymbol{X}}_{ik}, \tilde{\boldsymbol{X}}_{jk}\right) \right\|^2.$
    }
    \label{eq:PAtotal1}
\end{equation}
The corresponding minimization problem for two-view pose-only adjustment \cite{QiCai_TPAMI} is formulated as:
\begin{equation}
    \underset{\boldsymbol{R}_{ij}, \boldsymbol{t}_{ij}}{\arg\min} \, f_{PA}\left(\boldsymbol{R}_{ij}, \boldsymbol{t}_{ij}, \left\{\tilde{\boldsymbol{X}}_{ik}, \tilde{\boldsymbol{X}}_{jk}\right\}_{k=1,\ldots,m_{ij}}\right).
    \label{eq:PAtotal2}
\end{equation}

Compared to BA, the PA algorithm reduces the optimization dimensionality to the pose manifold (6 DoF) \cite{QiCai_TPAMI}. 
Advancing this thread, we aim to develop a reprojection-error-based optimization method that only evolves on the 3DoF camera rotation. To achieve this goal, the translation vector needs to be analytically represented using the rotation parameters.

Under the \textit{RankRegular} scene structure, the dimension of the solution space of \( {\boldsymbol{P}}^S_{ij} {\boldsymbol{x}} = {\boldsymbol{0}} \) is 1. Thus, we can only obtain partial information about the relative translation, denoted as \( \boldsymbol{t}_{ij}^\pm \). Next, we will prove that the partial information \( \boldsymbol{t}_{ij}^\pm \) alone is interestingly sufficient to calculate the pose-only reprojection residual (\ref{eq:PA4}) of a two-view geometry.

\begin{corollary}
    \label{corollary:corollary3}
    Under the assumption of chirality constraint and noise-free conditions,
        \begin{equation}
        \boldsymbol{V}_{ijk}^{PA}\left(\boldsymbol{R}_{ij}, \boldsymbol{t}_{ij}, {\boldsymbol{X}}_{ik}, {\boldsymbol{X}}_{jk}\right) = \boldsymbol{V}_{ijk}^{PA}\left(\boldsymbol{R}_{ij}, \boldsymbol{t}_{ij}^\pm, {\boldsymbol{X}}_{ik}, {\boldsymbol{X}}_{jk}\right)
        \label{eq:YPA2linearT}
    \end{equation}
    and
            \begin{equation}
        \scalebox{0.85}{$\boldsymbol{V}_{ijk,bearing}^{PA}\left(\boldsymbol{R}_{ij}, \boldsymbol{t}_{ij}, {\boldsymbol{X}}_{ik}, {\boldsymbol{X}}_{jk}\right) = \boldsymbol{V}_{ijk,bearing}^{PA}\left(\boldsymbol{R}_{ij}, \boldsymbol{t}_{ij}^\pm, {\boldsymbol{X}}_{ik}, {\boldsymbol{X}}_{jk}\right).$}
        \label{eq:YPA2BearinglinearT}
    \end{equation}
\end{corollary}

\begin{proof}
    Multiply both sides of (\ref{eq:PA1}) by the coefficient \( \|\boldsymbol{\theta}_{ijk}\| \), obtaining:
    \begin{equation}
        \scalebox{0.73}{$\begin{aligned}\left\|\boldsymbol{\theta}_{i j k}\right\| \boldsymbol{Y}_{i j k}^j & =\boldsymbol{R}_{i j} {\boldsymbol{X}}_{i k}\left\|\boldsymbol{\theta}_{i j k}\right\| \left\| \left[\boldsymbol{X}_{jk}\right]_\times \boldsymbol{t}_{ij}\right\| 
            +\left\| \boldsymbol{\theta}_{i j k} \right\|^2 \boldsymbol{t}_{i j} \\ & =-\left(\boldsymbol{R}_{i j} {\boldsymbol{X}}_{i k}\left(\left[{\boldsymbol{X}}_{j k}\right]_\times \boldsymbol{\theta}_{i j k}\right)^T-\left(\left[{\boldsymbol{X}}_{j k}\right]_\times \boldsymbol{\theta}_{i j k}\right)^T \boldsymbol{R}_{i j} {\boldsymbol{X}}_{i k} \boldsymbol{I}\right) \boldsymbol{t}_{i j} \\ & =\boldsymbol{S}_{ijk}^j \boldsymbol{t}_{i j} ,\end{aligned}$
        }
        \label{eq:YPA1}
    \end{equation}
    where $\boldsymbol{S}_{ijk}^j = -\left[\left[{\boldsymbol{X}}_{j k}\right]_\times \boldsymbol{\theta}_{i j k}\right]_\times\left[\boldsymbol{R}_{i j} {\boldsymbol{X}}_{i k}\right]_\times $. Thus, the two-view pose-only reprojection coordinate \( \boldsymbol{X}_{ijk}^j \) can be expressed as:
    \begin{equation}
        \boldsymbol{X}_{i j k}^j=\frac{\left\|\boldsymbol{\theta}_{i j k}\right\| \boldsymbol{Y}_{i j k}^j}{\left\|\boldsymbol{\theta}_{i j k}\right\| \boldsymbol{Y}_{i j k(3)}^j}=\frac{\boldsymbol{S}_{ijk}^j \boldsymbol{t}_{i j}}{\boldsymbol{e}_3^T \boldsymbol{S}_{ijk}^j \boldsymbol{t}_{i j}}.
        \label{eq:YPA2}
    \end{equation}
    \begin{equation}
        \vec{\boldsymbol{X}}_{ijk}^j=\frac{\left\|\boldsymbol{\theta}_{i j k}\right\| \boldsymbol{Y}_{i j k}^j}{\left\|\boldsymbol{\theta}_{i j k}\right\| \left\|\boldsymbol{Y}_{ijk}^j\right\|}=\frac{\boldsymbol{S}_{ijk}^j \boldsymbol{t}_{i j}}{\left\|\boldsymbol{S}_{ijk}^j \boldsymbol{t}_{i j}\right\|}.
        \label{eq:YPA2bearing}
    \end{equation}
    According to the equality (\ref{eq:fundamental1}), substituting (\ref{eq:YPA2})  into  (\ref{eq:PA3}), the two-view pose-only reprojection residual can be expressed as: 
    \begin{equation}
        \scalebox{1}{$\begin{aligned}
            \boldsymbol{V}_{ijk}^{PA,j}\left(\boldsymbol{R}_{ij}, \boldsymbol{t}_{ij}, {\boldsymbol{X}}_{ik}, {\boldsymbol{X}}_{jk}\right) & = \frac{ \left({\boldsymbol{X}}_{jk}\boldsymbol{e}_3^T - \boldsymbol{I} \right) \boldsymbol{S}_{ijk}^j \boldsymbol{t}_{i j}}{\boldsymbol{e}_3^T \boldsymbol{S}_{ijk}^j \boldsymbol{t}_{i j} }\\ & =\frac{ \left[\boldsymbol{e}_3\right]_\times \left[\boldsymbol{X}_{jk}\right]_\times \boldsymbol{S}_{ijk}^j \boldsymbol{t}_{i j}}{\boldsymbol{e}_3^T \boldsymbol{S}_{ijk}^j \boldsymbol{t}_{i j} }. \end{aligned}$}
    \label{eq:PA3linearT}
\end{equation}

    Substituting (\ref{eq:YPA2bearing})  into  (\ref{eq:PA3bearing}), since the direction of $\boldsymbol{X}_{ik}$ and $\boldsymbol{S}_{ijk}^j \boldsymbol{t}_{i j}$ are the same, the bearing vector of the pose-only reprojection residual can be expressed as: 
    \begin{equation}
        \scalebox{0.9}{$\begin{aligned}\boldsymbol{V}_{ijk,bearing}^{PA,j}\left(\boldsymbol{R}_{ij}, \boldsymbol{t}_{ij}, {\boldsymbol{X}}_{ik}, {\boldsymbol{X}}_{jk}\right) & =\frac{ \boldsymbol{X}_{jk} }{ \left\| \boldsymbol{X}_{jk} \right\|  } - \frac{ \boldsymbol{S}_{ijk}^j \boldsymbol{t}_{i j} }{ \left\| \boldsymbol{S}_{ijk}^j \boldsymbol{t}_{i j} \right\|  } \\ & =\frac{ \left( \boldsymbol{X}_{jk} \boldsymbol{X}_{jk}^T - \left\| \boldsymbol{X}_{jk} \right\|^2 \boldsymbol{I} \right)  \boldsymbol{S}_{ijk}^j \boldsymbol{t}_{i j}}{\left\| \boldsymbol{X}_{jk} \right\|^2 \left\| \boldsymbol{S}_{ijk}^j \boldsymbol{t}_{i j} \right\| } \\ & =\frac{ \left[ \boldsymbol{X}_{jk} \right]_\times \left[ \boldsymbol{X}_{jk} \right]_\times \boldsymbol{S}_{ijk}^j \boldsymbol{t}_{i j}}{\left\| \boldsymbol{X}_{jk} \right\|^2 \left\| \boldsymbol{S}_{ijk}^j \boldsymbol{t}_{i j} \right\| } .\end{aligned}$
        }
        \label{eq:PA3linearTbearing}
    \end{equation}

    From (\ref{eq:PA3linearT}) and (\ref{eq:PA3linearTbearing}), it is evident that the pose-only reprojection residual and its corresponding bearing vector form are invariant to the sign or magnitude of the relative translation, i.e., 
    \begin{equation}
        \boldsymbol{V}_{ijk}^{PA,j}\left(\boldsymbol{R}_{ij}, \boldsymbol{t}_{ij}, {\boldsymbol{X}}_{ik}, {\boldsymbol{X}}_{jk}\right) = \boldsymbol{V}_{ijk}^{PA,j}\left(\boldsymbol{R}_{ij}, \boldsymbol{t}_{ij}^\pm, {\boldsymbol{X}}_{ik}, {\boldsymbol{X}}_{jk}\right),
    \end{equation}
    \begin{equation}
        \scalebox{0.87}{$\boldsymbol{V}_{ijk,bearing}^{PA,j}\left(\boldsymbol{R}_{ij}, \boldsymbol{t}_{ij}, {\boldsymbol{X}}_{ik}, {\boldsymbol{X}}_{jk}\right) = \boldsymbol{V}_{ijk,bearing}^{PA,j}\left(\boldsymbol{R}_{ij}, \boldsymbol{t}_{ij}^\pm, {\boldsymbol{X}}_{ik}, {\boldsymbol{X}}_{jk}\right)$}.
    \end{equation}
    
    Similarly, the same property holds for view $i$, leading to the conclusion that (\ref{eq:YPA2linearT}) and (\ref{eq:YPA2BearinglinearT}) are valid.
\end{proof}

Following this, we will discuss the impact of the solutions of \( {\boldsymbol{P}}^S_{ij} {\boldsymbol{x}} = {\boldsymbol{0}} \) on the pose-only reprojection residual (\ref{eq:PA3}) when the scene structure corresponds to \textit{PR/B/I} and \textit{Holoplane}.

\begin{corollary}\label{corollary:three}
    Under the noise-free condition, the pose-only reprojection residual has the following properties:

    {Case 1:} When the scene structure corresponds to \textit{PR/B/I}, the pose-only reprojection error (\ref{eq:PA3}) and its corresponding bearing vector form (\ref{eq:PA3bearing}) collapse onto the rotation manifold and is independent of translation.

    {Case 2:} When the scene structure corresponds to {\textit{Holoplane}}, although the linear equation \( {\boldsymbol{P}}^S_{ij} {\boldsymbol{x}} = {\boldsymbol{0}} \) cannot uniquely constrain the direction of the relative translation, any solution to \( {\boldsymbol{P}}^S_{ij} {\boldsymbol{x}} = {\boldsymbol{0}} \) satisfies the pose-only reprojection constraint and its corresponding bearing vector form. 
\end{corollary}

\begin{proof}
\mbox{}
\hspace*{\parindent}

    {Case 1:} When the scene structure corresponds to \textit{PR/B/I}, \(\boldsymbol{\theta}_{ijk} = \boldsymbol{0}\). 
    If the 3D point $k$ does not lie on the translation baseline, then $\left\|\boldsymbol{\beta}_{ijk}\right\| \neq 0$.  The pose-only reprojection coordinate (\ref{eq:PA2}) degenerates to
        \begin{equation}
            \begin{aligned}
            \boldsymbol{X}_{ijk}^j & = \left\|\boldsymbol{\beta}_{ijk}\right\| \boldsymbol{R}_{ij} {\boldsymbol{X}}_{ik} /\left(\left\|\boldsymbol{\beta}_{ijk}\right\| \boldsymbol{R}_{ij} {\boldsymbol{X}}_{ik}\right)_{(3)} \\
            & = \boldsymbol{R}_{ij} {\boldsymbol{X}}_{ik} /\left(\boldsymbol{R}_{ij} {\boldsymbol{X}}_{ik}\right)_{(3)} .
            \end{aligned}
        \label{eq:PureRotationReproj}
        \end{equation}
        If the 3D point $k$ lies on the translation baseline, then $\boldsymbol{X}^C_{jk}=a \boldsymbol{t}_{ij}, a\in\mathbb{R}$. Based on (\ref{eq:RigidTransformation}) and (\ref{eq:two-view_imaging_equation}), ${{\boldsymbol{R}}_{ij}}{{\boldsymbol{X}}_{ik}}=(a-1)\boldsymbol{t}_{ij}/z_{ik}^C$. Therefore, the reprojection coordinate $\boldsymbol{X}_{ijk}^j  =\boldsymbol{X}^C_{jk} / \boldsymbol{X}^C_{jk(3)}=\boldsymbol{R}_{ij} {\boldsymbol{X}}_{ik} /\left(\boldsymbol{R}_{ij} {\boldsymbol{X}}_{ik}\right)_{(3)}$. The equation (\ref{eq:PureRotationReproj}) holds regardless of whether the 3D point $k$ lies on the translation baseline.
        Substituting (\ref{eq:PureRotationReproj}) into (\ref{eq:PA3}) and (\ref{eq:PA3bearing}), the pose-only reprojection residual and its corresponding bearing vector form degenerate to:
        \begin{equation}
            \scalebox{0.9}{$\boldsymbol{V}_{ijk}^{PA, j}\left(\boldsymbol{R}_{ij}, \boldsymbol{t}_{ij}, {\boldsymbol{X}}_{ik}, {\boldsymbol{X}}_{jk}\right) = {\boldsymbol{X}}_{jk} - \boldsymbol{R}_{ij} {\boldsymbol{X}}_{ik} /\left(\boldsymbol{R}_{ij} {\boldsymbol{X}}_{ik}\right)_{(3)}$.
        }\label{eq:PureRotationReprojError}
        \end{equation}
        \begin{equation}
            \scalebox{0.9}{$\boldsymbol{V}_{ijk,bearing}^{PA, j}\left(\boldsymbol{R}_{ij}, \boldsymbol{t}_{ij}, {\boldsymbol{X}}_{ik}, {\boldsymbol{X}}_{jk}\right) = \vec{{\boldsymbol{X}}}_{jk} - \boldsymbol{R}_{ij} {\boldsymbol{X}}_{ik} /\left\|\boldsymbol{R}_{ij} {\boldsymbol{X}}_{ik}\right\|$.
        }\label{eq:PureRotationReprojBearingError}
        \end{equation}
        At this point, \(\boldsymbol{V}_{ijk}^{PA, j}\) and $\boldsymbol{V}_{ijk,bearing}^{PA, j}$ become independent of \(\boldsymbol{t}_{ij}\), and can also be expressed as \(\boldsymbol{V}_{ijk}^{PA, j}\left(\boldsymbol{R}_{ij}, \boldsymbol{0}, {\boldsymbol{X}}_{ik}, {\boldsymbol{X}}_{jk}\right)\) and \(\boldsymbol{V}_{ijk,bearing}^{PA, j}\left(\boldsymbol{R}_{ij}, \boldsymbol{0}, {\boldsymbol{X}}_{ik}, {\boldsymbol{X}}_{jk}\right)\). This indicates that under scene structure \textit{PR/B/I}, the pose-only reprojection error collapses directly onto the rotation manifold.

        {Case 2:} When the scene structure corresponds to {\textit{Holoplane}}, \(\boldsymbol{\theta}_{ijk} \neq \boldsymbol{0}\). Based on (\ref{eq:TwoViewPPO}) and (\ref{eq:YPA1}), the linear equation \( {\boldsymbol{P}}^S_{ij} {\boldsymbol{x}} = {\boldsymbol{0}} \) can be expressed as:
        \begin{equation}
            \left\| {{{\boldsymbol{\theta }}_{ijk}}} \right\|\left\| {{{\boldsymbol{\alpha }}_{ijk}}} \right\|{{\boldsymbol{X}}_{jk}}=\boldsymbol{S}_{ijk}^j \boldsymbol{x}.
            \label{eq:PPO1Sform}
        \end{equation}
        Left-multiplying both sides of (\ref{eq:PPO1Sform}) by \( \left[ {\boldsymbol{X}_{jk}} \right]_{\times} \) yields:
        \begin{equation}
            \left[ {\boldsymbol{X}_{jk}} \right]_{\times} \boldsymbol{S}_{ijk}^j \boldsymbol{x}=\boldsymbol{0}.
            \label{eq:PPO1Sform2}
        \end{equation}
        Substituting (\ref{eq:PPO1Sform2}) into the pose-only reprojection residual (\ref{eq:PA3linearT}) and its corresponding bearing vector form (\ref{eq:PA3linearTbearing}), we get:
    \begin{equation}
            \scalebox{1}{$\begin{aligned}
                \boldsymbol{V}_{ijk}^{PA,j}\left(\boldsymbol{R}_{ij}, \boldsymbol{x}, {\boldsymbol{X}}_{ik}, {\boldsymbol{X}}_{jk}\right) & =\frac{ \left[\boldsymbol{e}_3\right]_\times \left[\boldsymbol{X}_{jk}\right]_\times \boldsymbol{S}_{ijk}^j \boldsymbol{x} }{\boldsymbol{e}_3^T \boldsymbol{S}_{ijk}^j \boldsymbol{x} } = \boldsymbol{0}. \end{aligned}$}
        \label{eq:PA3linearTx}
    \end{equation}
    \begin{equation}
        \scalebox{0.95}{$\begin{aligned}
            \boldsymbol{V}_{ijk,bearing}^{PA,j}\left(\boldsymbol{R}_{ij}, \boldsymbol{x}, {\boldsymbol{X}}_{ik}, {\boldsymbol{X}}_{jk}\right) & =\frac{ \left[ \boldsymbol{X}_{jk} \right]_\times \left[\boldsymbol{X}_{jk}\right]_\times \boldsymbol{S}_{ijk}^j \boldsymbol{x} }{\left\| \boldsymbol{X}_{jk} \right\|^2 \boldsymbol{S}_{ijk}^j \boldsymbol{x} } = \boldsymbol{0}. \end{aligned}$}
    \label{eq:PA3linearTBearingx}
\end{equation}

        This implies that when the scene structure corresponds to {\textit{Holoplane}}, although the linear equation \( {\boldsymbol{P}}^S_{ij} {\boldsymbol{x}} = {\boldsymbol{0}} \) cannot uniquely constrain the direction of the relative translation, any solution to \( {\boldsymbol{P}}^S_{ij} {\boldsymbol{x}} = {\boldsymbol{0}} \) satisfies the pose-only reprojection constraint and its corresponding bearing vector form.
    \end{proof}

Subsequently, we employ $\boldsymbol{P}^S_{ij}$ to derive the direction vector of the relative translation. This constitutes the final piece in formalizing the two-view reprojection residuals framework on the rotation manifold.

We can obtain the characteristic polynomial of $\boldsymbol{P}^S_{ij}$:
\begin{equation}
    p \left(\lambda\right) = \lambda ^3 - Tr\left(\boldsymbol{P}^S_{ij}\right) \lambda^2 + Tr\left(adj\left(\boldsymbol{P}^S_{ij}\right)\right) \lambda - det(\boldsymbol{P}^S_{ij}),
    \label{eq:characteristicPolynomial}
\end{equation}
where $adj\left(\boldsymbol{P}^S_{ij}\right)$ is the adjugate matrix of $\boldsymbol{P}^S_{ij}$. By applying Cardano's formula, we set $\lambda=y+{Tr\left(\boldsymbol{P}^S_{ij}\right)}/{3}$, thereby obtaining the simplified expression:
\begin{equation}
    y^3 + py + q = 0,
    \label{eq:CardanoSimple}
\end{equation}
where:
\begin{equation}
    p = Tr\left(adj\left(\boldsymbol{P}^S_{ij}\right)\right) - \frac{Tr\left(\boldsymbol{P}^S_{ij}\right)^2}{3}.
    \label{eq:p_Cardano}
 \end{equation}
\begin{equation}
    q = - det(\boldsymbol{P}^S_{ij}) + \frac{ Tr\left(\boldsymbol{P}^S_{ij}\right) Tr\left(adj\left(\boldsymbol{P}^S_{ij}\right)\right)}{3} - \frac{2Tr\left(\boldsymbol{P}^S_{ij}\right)^3}{27}.
    \label{eq:q_Cardano}
\end{equation}
The discriminant is defined as $ \Delta = \left( q/2 \right)^2 + \left( p/3 \right)^3 $. Since $\boldsymbol{P}^S_{ij}$  is a real symmetric matrix, the equation (\ref{eq:characteristicPolynomial}) has three real roots, and $\Delta \leq 0$. In this case, by introducing a trigonometric substitution $y=-2\sqrt{-{p}/{3}} \cos\theta$ into equation (\ref{eq:CardanoSimple}) and applying the triple-angle formula, we obtain:
\begin{equation}
    \cos 3\theta = \frac{q}{2} \sqrt{-\frac{27}{p^3}}.
\end{equation}
Thus, we can further derive:
\begin{equation}
    \theta =  \frac{\arccos \left(\frac{q}{2} \sqrt{-\frac{27}{p^3}}\right) +2 k \pi}{3}, \quad k = 0, 1, 2.
    \label{eq:car_theta}
\end{equation}
Since the range of  $\phi = \arccos x$  is $[0, \pi]$, it follows that $\cos ((\phi + 2 k \pi)/3) < \cos (\phi/3)$, $k = 1, 2$, which means: 
\begin{equation}
     \lambda_{\min}=-2\sqrt{-\frac{p}{3}} \cos \left(\frac{\arccos \left(\frac{q}{2} \sqrt{-\frac{27}{p^3}}\right)}{3}\right) +\frac{Tr\left(\boldsymbol{P}^S_{ij}\right)}{3}.
    \label{eq:minlambda}
\end{equation}
Since the linear constraint \(\boldsymbol{P}^S_{ij} \boldsymbol{x} = \boldsymbol{0}\) holds, it follows that in the absence of noise, $\lambda_{\min}=0$. 

We decompose matrix $\boldsymbol{P}^S_{ij}$ as \(\left[\boldsymbol{p}_1, \boldsymbol{p}_2, \boldsymbol{p}_3\right]^T\), where \(\boldsymbol{p}_1^T = [p_{11}, p_{12}, p_{13}] $, \(\boldsymbol{p}_2^T = [p_{12}, p_{22}, p_{23}]\), and \(\boldsymbol{p}_3^T = [p_{13}, p_{23}, p_{33}]\) are the row vectors of $\boldsymbol{P}^S_{ij}\).  Define the set \(S = \left\{\boldsymbol{\xi}_1, \boldsymbol{\xi}_2, \boldsymbol{\xi}_3\right\} \subset \mathbb{R}^3\), where
\begin{equation}
    \begin{aligned}
        \boldsymbol{\xi}_1 = \left(\boldsymbol{p}_1^T-\lambda_{\min} \boldsymbol{e}_1\right) \times \left(\boldsymbol{p}_2^T-\lambda_{\min} \boldsymbol{e}_2\right), \\
        \boldsymbol{\xi}_2 = \left(\boldsymbol{p}_1^T-\lambda_{\min} \boldsymbol{e}_1\right) \times \left(\boldsymbol{p}_3^T-\lambda_{\min} \boldsymbol{e}_3\right), \\
        \boldsymbol{\xi}_3 = \left(\boldsymbol{p}_2^T-\lambda_{\min} \boldsymbol{e}_2\right) \times \left(\boldsymbol{p}_3^T-\lambda_{\min} \boldsymbol{e}_3\right).
    \label{eq:anaTset}
    \end{aligned}
\end{equation}
Any non-zero vector in set $S$ is a analytical solution to \(\boldsymbol{P}^S_{ij} \boldsymbol{t}_{ij} = \boldsymbol{0}\). To illustrate the analytical properties of our relative translation solution, we express it as \(\boldsymbol{t}_{ij}\left(\boldsymbol{R}_{ij}, \left\{\tilde{\boldsymbol{X}}_{ik}, \tilde{\boldsymbol{X}}_{jk}\right\}_{k=1, \ldots, m_{ij}}\right)\). Substituting this into (\ref{eq:PA3}) or (\ref{eq:PA3bearing}), we can construct the two-view reprojection residuals on the rotation manifold (\textit{TRRM}) and its corresponding bearing vector form:

\begin{equation}
    \scalebox{0.87}{$\begin{aligned}
        & \boldsymbol{V}_{i j k}^{TRRM}\left(\boldsymbol{R}_{i j},\left\{\tilde{\boldsymbol{X}}_{i k}, \tilde{\boldsymbol{X}}_{j k}\right\}_{k=1, \ldots, m_{ij}}\right) \\
        & =\boldsymbol{V}_{i j k}^{PA}\left(\boldsymbol{R}_{i j}, \boldsymbol{t}_{i j}\left(\boldsymbol{R}_{i j},\left\{\tilde{\boldsymbol{X}}_{i k}, \tilde{\boldsymbol{X}}_{j k}\right\}_{k=1, \ldots m_{i j}}\right), \tilde{\boldsymbol{X}}_{i k}, \tilde{\boldsymbol{X}}_{j k}\right).
    \end{aligned}$
    }\footnote{Note: $\boldsymbol{V}_{i j k}^{TRRM}=\begin{bmatrix}
        \boldsymbol{V}_{ijk}^{TRRM,i} \\
    \boldsymbol{V}_{ijk}^{TRRM,j}
    \end{bmatrix}$, and similarly for $\boldsymbol{V}_{i j k,bearing}^{TRRM}$.}
    \label{eq:ROMres}
\end{equation}   
\begin{equation}
    \scalebox{0.87}{$\begin{aligned}
        & \boldsymbol{V}_{i j k,bearing}^{TRRM}\left(\boldsymbol{R}_{i j},\left\{\tilde{\boldsymbol{X}}_{i k}, \tilde{\boldsymbol{X}}_{j k}\right\}_{k=1, \ldots, m_{ij}}\right) \\
        & =\boldsymbol{V}_{i j k,bearing}^{PA}\left(\boldsymbol{R}_{i j}, \boldsymbol{t}_{i j}\left(\boldsymbol{R}_{i j},\left\{\tilde{\boldsymbol{X}}_{i k}, \tilde{\boldsymbol{X}}_{j k}\right\}_{k=1, \ldots m_{i j}}\right), \tilde{\boldsymbol{X}}_{i k}, \tilde{\boldsymbol{X}}_{j k}\right).
    \end{aligned}$
    }
    \label{eq:ROMresBearing}
\end{equation}  

When the scene structure corresponds to \textit{PR/B/I}, according to Corollary \ref{corollary:three}, the pose-only reprojection error degenerates to a form on the rotation manifold, as shown in (\ref{eq:PureRotationReprojError}) and (\ref{eq:PureRotationReprojBearingError}). This indicates that at this point, \( \boldsymbol{V}_{i j k}^{TRRM} \left( \boldsymbol{R}_{i j}, \left\{\tilde{\boldsymbol{X}}_{i k}, \tilde{\boldsymbol{X}}_{j k}\right\}_{k=1, \ldots, m_{i j}} \right) \) is consistent with pose-only reprojection error in form, so is $\boldsymbol{V}_{i j k,bearing}^{TRRM}\left( \boldsymbol{R}_{i j}, \left\{\tilde{\boldsymbol{X}}_{i k}, \tilde{\boldsymbol{X}}_{j k}\right\}_{k=1, \ldots, m_{ij}} \right)$. 
We recommend using the latter to construct the error function for the relative rotation optimization on the edge \( e_{i j} \):
\begin{equation}
    \scalebox{0.87}{$\begin{aligned}
        & f_{TRRM}\left(\boldsymbol{R}_{i j},\left\{\tilde{\boldsymbol{X}}_{i k}, \tilde{\boldsymbol{X}}_{j k}\right\}_{k=1, \ldots, m_{ij}}\right) \\
        & =\sum_{k=1}^{m_{ij}}\left\|\boldsymbol{V}_{i j k,bearing}^{ {TRRM}}\left(\boldsymbol{R}_{i j},\left\{\tilde{\boldsymbol{X}}_{i k}, \tilde{\boldsymbol{X}}_{j k}\right\}_{k=1, \ldots, m_{ij}}\right)\right\|^2.
    \end{aligned}$
    }
    \label{eq:ROMerr}
\end{equation}
The corresponding minimization problem is formulated as:
\begin{equation}
    \scalebox{1}{$\underset{\boldsymbol{R}_{i j}}{\operatorname{argmin}} \, f_{TRRM} \left( \boldsymbol{R}_{i j}, \left\{\tilde{\boldsymbol{X}}_{i k}, \tilde{\boldsymbol{X}}_{j k}\right\}_{k=1, \ldots, m_{ij}} \right).$
    }
    \label{eq:TRRMopt}
\end{equation}

In particular, the matrix \( {\boldsymbol{P}}^S_{ij} \) in {Theorem \ref{theorem:theorem1}} exhibits a formal consistency with the matrix \( \boldsymbol{M} \) in \cite{Kneip13ICCV}. The difference is that \( \boldsymbol{P}^S_{ij} \) in this paper is derived by using the chirality visual constraint, while \( \boldsymbol{M} \) is derived from the epipolar constraint that omits partial visual geometric information \cite{agarwal2022chiral,QiCai_IJCV}. 
In addition, this paper uses \( \boldsymbol{P}^S_{ij} \) only to provide the direction of the relative translation and construct the reprojection error function (\ref{eq:ROMerr}), whereas \cite{Kneip13ICCV} directly uses the minimum eigenvalue of \( \boldsymbol{M} \) as the optimization objective. Considering that \(\boldsymbol{M}\) is a degenerate form of the imaging equation (\ref{eq:two-view_imaging_equation}), we believe that the proposed method would yield better estimation results.

\section{Multi-view Rotation Estimation}

Multi-view SfM requires utilizing information from more than two matched views. This chapter constructs a global reprojection residual function on the rotation manifold to optimize global rotations.

In the world coordinate system, the projection equation of a 3D scene point is expressed as:

\begin{equation}
    \boldsymbol{X}_{ik} = \frac{\boldsymbol{R}_i\left(\boldsymbol{X}_k^w - \boldsymbol{t}_i\right)}{z_{ik}^C}.
    \label{eq:projectionEquation}
\end{equation}

In the classical SfM pipeline, after globally reconstructing the 3D point \( k \), the global reprojection coordinates \( \boldsymbol{X}_{ik}^r \) can be solved using the right-hand side of equation (\ref{eq:projectionEquation}). Here, \( z_{ik}^C = \boldsymbol{e}_3^T \boldsymbol{R}_i\left(\boldsymbol{X}_k^w - \boldsymbol{t}_i\right) \) represents the depth of the 3D point \( k \) observed by camera \( i \).

In  \cite{QiCai_TPAMI}, the global analytical depth was calculated as \( d_{ik} = \sum_{e_{ij} \in \mathcal{E}_{ik}} \omega_{ijk} d_{ijk}^i \) by performing a linear weighted summation of the analytical depths derived from matched views. Here, \( \mathcal{E}_{ik} \) denotes the set of matched views about camera \( i \) that observe the 3D point \( k \). When \( \|\boldsymbol{\theta}_{ijk}\| \neq 0 \), \( d_{ijk}^i \) and \( d_{ijk}^j \) denote the analytical depths of 3D point \( k \) on matched views \( i \) and \( j \), respectively,
\begin{equation}
    \begin{aligned}
        & d_{ijk}^i = z_{ik}^C = \frac{\|\left[\boldsymbol{X}_{jk}\right]_\times \boldsymbol{t}_{ij}\|}{\|\boldsymbol{\theta}_{ijk}\|}, \\
        & d_{ijk}^j = z_{jk}^C = \frac{\|\left[\boldsymbol{R}_{ij} \boldsymbol{X}_{ik}\right]_\times \boldsymbol{t}_{ij}\|}{\|\boldsymbol{\theta}_{ijk}\|}.
        \label{eq:analyticalDepth}
    \end{aligned}
\end{equation}
The weighting coefficient \(\omega_{ijk}\) is defined as:
\begin{equation}
    \omega_{ijk} = \frac{\|\boldsymbol{\theta}_{ijk}\|}{\omega_{ik}},
    \label{eq:weight}
\end{equation}
where \(\omega_{ik} = \sum_{e_{ij} \in \mathcal{E}_{ik}} \|\boldsymbol{\theta}_{ijk}\|\). When analyzing the relationship between global reprojection coordinates and two-view reprojection coordinates, we adopt the above approach to handling depth information from \cite{QiCai_TPAMI}.

\begin{corollary}\label{corollary:4}
    The reprojection coordinates of the 3D point \(k\) observed by camera \( i \) can be expressed as a linear weighted summation of the reprojection coordinates of \(\mathcal{E}_{ik}\) at camera \(i\):
    \begin{equation}
        \boldsymbol{X}_{ik}^r = \sum_{e_{ij} \in \mathcal{E}_{ik}} \omega_{ijk} \boldsymbol{X}_{ijk}^i.
        \label{eq:linearSumCoord}
    \end{equation}
\end{corollary}

\begin{proof}
    Using the projection equation for 3D points (\ref{eq:projectionEquation}), the reconstructed coordinates can be expressed as the back-projection of the image points:
    \begin{equation}
        \boldsymbol{X}_k^w = z_{ik}^C \boldsymbol{R}_i^T \boldsymbol{X}_{ik} + \boldsymbol{t}_i.
        \label{eq:backProjection}
    \end{equation}

    For matched views \(i, j\) where \( \|\boldsymbol{\theta}_{ijk}\| \neq 0\), replacing \(z_{ik}^C\) with the global analytical depth \(d_{ik}\), (\ref{eq:backProjection}) can be rewritten as:
    \begin{equation}
        \begin{aligned}
        \boldsymbol{X}_k^w & = d_{ik} \boldsymbol{R}_i^T \boldsymbol{X}_{ik} + \boldsymbol{t}_i \\
        & = \sum_{e_{ij} \in \mathcal{E}_{ik}} \omega_{ijk} \left(d_{ijk}^i \boldsymbol{R}_i^T \boldsymbol{X}_{ik} + \boldsymbol{t}_i\right). \\
        &= \sum_{e_{ij} \in \mathcal{E}_{ik}} \omega_{ijk} \left(d_{ijk}^j \boldsymbol{R}_j^T \boldsymbol{X}_{jk} + \boldsymbol{t}_j\right).
        \end{aligned}
        \label{eq:backProjection2}
    \end{equation}
    Substituting (\ref{eq:backProjection2}) into (\ref{eq:projectionEquation}), $\boldsymbol{X}_{ik}^r$ can be expressed as
    \begin{equation}
        \begin{aligned}
            \boldsymbol{X}_{ik}^r & = \frac{\boldsymbol{R}_i \left(\sum_{e_{ij} \in \mathcal{E}_{ik}} \omega_{ijk} \left(d_{ijk}^j \boldsymbol{R}_j^T \boldsymbol{X}_{jk} + \boldsymbol{t}_j\right) - \boldsymbol{t}_i\right)}{z_{ik}^C} \\
            & = \sum_{e_{ij} \in \mathcal{E}_{ik}} \omega_{ijk} \frac{d_{ijk}^j \boldsymbol{R}_{ji} \boldsymbol{X}_{jk} + \boldsymbol{t}_{ji}}{z_{ik}^C}.
        \end{aligned}
        \label{eq:backProjection3}
    \end{equation}
    Substituting (\ref{eq:analyticalDepth}) into (\ref{eq:backProjection3}) and multiplying both the numerator and denominator of the fraction by \( \|\boldsymbol{\theta}_{ijk}\| \), the global reprojection coordinates can be expressed in the following form:
    \begin{equation}
        \begin{aligned}
            \boldsymbol{X}_{ik}^r = \sum_{e_{ij} \in \mathcal{E}_{ik}} \omega_{ijk} \frac{\|\left[\boldsymbol{X}_{ik}\right]_\times \boldsymbol{t}_{ji}\| \boldsymbol{R}_{ji} \boldsymbol{X}_{jk} + \|\boldsymbol{\theta}_{ijk}\| \boldsymbol{t}_{ji}}{\|\left[\boldsymbol{R}_{ji} \boldsymbol{X}_{jk}\right]_\times \boldsymbol{t}_{ji}\|}.
        \end{aligned}
        \label{eq:backProjection4}
    \end{equation}
    Based on the derivations in (\ref{eq:PA1}) and (\ref{eq:PA2}), (\ref{eq:backProjection4}) can be expressed as
    \begin{equation}
        \boldsymbol{X}_{ik}^r = \sum_{e_{ij} \in \mathcal{E}_{ik}} \omega_{ijk} \boldsymbol{X}_{ijk}^i.
        \label{eq:backProjection5}
    \end{equation}

    For matched views \( i, j \) with \( \|\boldsymbol{\theta}_{ijk}\| = 0 \), the two-view reprojection coordinates \( \boldsymbol{X}_{ijk}^i \) degenerate to (\ref{eq:PureRotationReproj}). Since \( \omega_{ijk} = \|\boldsymbol{\theta}_{ijk}\| / \omega_{ik} = 0 \), those matched views with \( \|\boldsymbol{\theta}_{ijk}\| = 0 \) do not affect the value of (\ref{eq:linearSumCoord}). Therefore, (\ref{eq:linearSumCoord}) is complete for all matched views.
\end{proof}

Under the influence of noise observations, the reprojection residual of BA in classical SfMs is given as
\begin{equation}
    \begin{aligned}
        \boldsymbol{V}_{ik}^{BA}\left(\boldsymbol{R}_i, \boldsymbol{t}_i, \boldsymbol{X}_k^w, \tilde{\boldsymbol{X}}_{ik}\right) &= \boldsymbol{X}_{ik}^r - \tilde{\boldsymbol{X}}_{ik} \\
        &= \frac{\boldsymbol{R}_i\left(\boldsymbol{X}_k^w - \boldsymbol{t}_i\right)}{\boldsymbol{e}_3^T \boldsymbol{R}_i\left(\boldsymbol{X}_k^w - \boldsymbol{t}_i\right)} - \tilde{\boldsymbol{X}}_{ik}.
    \end{aligned}
    \label{eq:BAresidual}
\end{equation}
The global BA error function can be expressed as
\begin{equation}
    \begin{aligned}
    &f_{BA}\left(\left\{\boldsymbol{X}_k^w\right\}_{k=1,\ldots,m}, \left\{\left[\boldsymbol{R}_i \mid \boldsymbol{t}_i\right]\right\}_{i=1,\ldots,n}, \left\{\tilde{\boldsymbol{X}}_{ik}\right\}_{\substack{i=1,\ldots,n \\ k=1,\ldots,m}}\right) \\
    &= \sum_{k=1}^m \sum_{i=1}^n \left\|\boldsymbol{V}_{ik}^{BA}\left(\boldsymbol{R}_i, \boldsymbol{t}_i, \boldsymbol{X}_k^w, \tilde{\boldsymbol{X}}_{ik}\right)\right\|^2.
    \end{aligned}
    \label{eq:BAErrorFun}
\end{equation}

Based on {Corollary \ref{corollary:4}}, substituting the reprojection coordinates in (\ref{eq:BAresidual}) with the weighted two-view reprojection coordinates (\ref{eq:backProjection5}), the global reprojection residual on the rotation manifold (\textit{GRRM}) is given as:
\begin{equation}
    \begin{aligned}
    \boldsymbol{V}_{ik}^{{GRRM}} &= \left(\sum_{e_{ij} \in \mathcal{E}_{ik}} \omega_{ijk} \boldsymbol{X}_{ijk}^i\right) - \tilde{\boldsymbol{X}}_{ik} \\
    &= \sum_{e_{ij} \in \mathcal{E}_{ik}} \omega_{ijk} \left(\boldsymbol{X}_{ijk}^i - \tilde{\boldsymbol{X}}_{ik}\right).
\end{aligned}
\label{eq:GRRMRes1}
\end{equation}
The global reprojection residual (\ref{eq:GRRMRes1}) is demonstrated in Fig.~\ref{fig:globalresidual}.

\begin{figure}[t] 
    \centering 
    \includegraphics[width=0.48\textwidth]{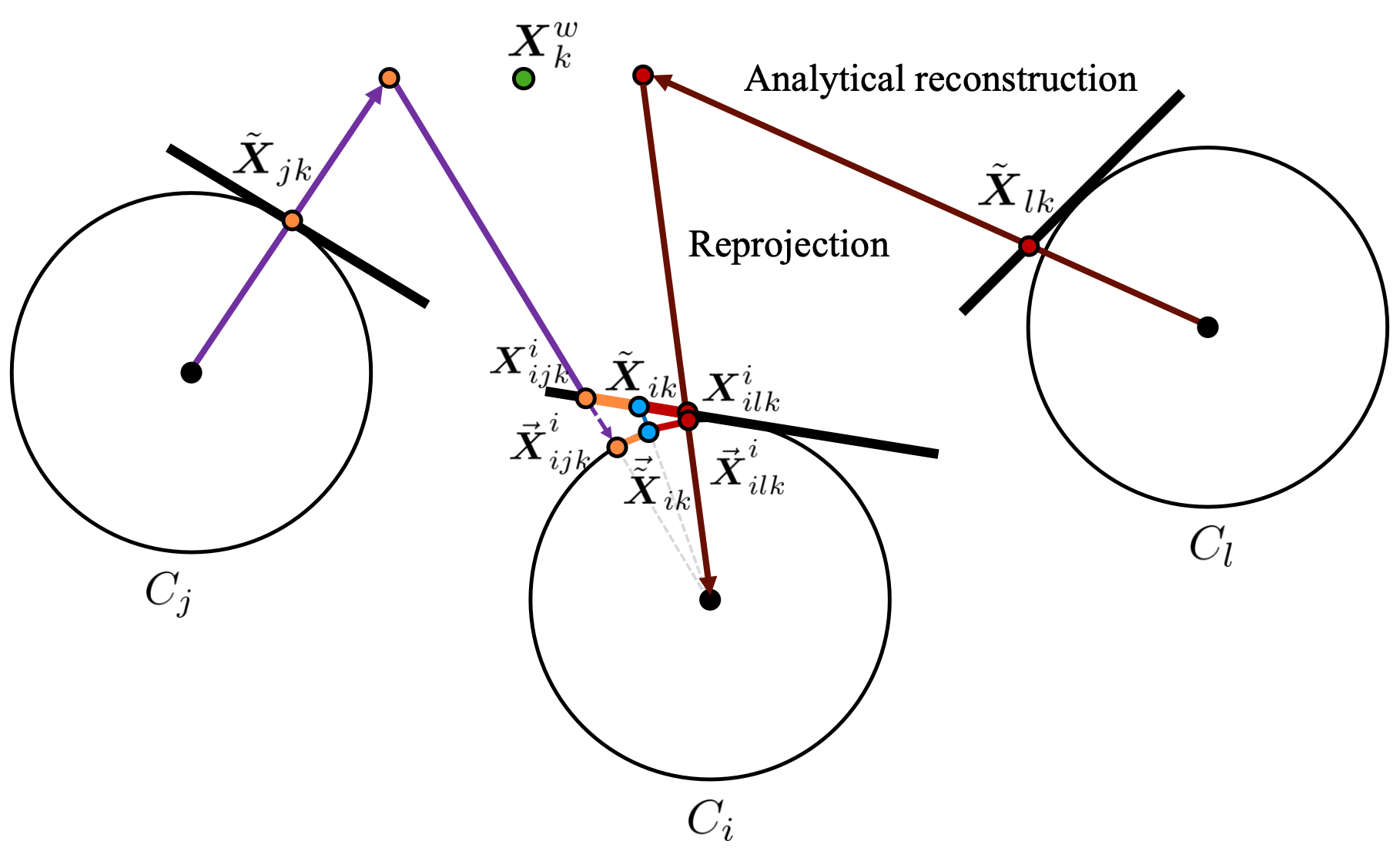} 
    \caption{Generation mechanism of $\boldsymbol{V}_{ik}^{{GRRM}}$ (using three views as an example). $\boldsymbol{X}_{ijk}^j$ and $\boldsymbol{X}_{ilk}^i$ represent pose-only reprojection coordinates on camera $i$ of observations $\tilde{\boldsymbol{X}}_{jk}$ and $\tilde{\boldsymbol{X}}_{lk}$ from cameras $j$ and $l$, respectively. Pose-only reprojection residual on imaging plane of camera $i$ and its corresponding bearing vector on unit sphere are illustrated by orange and red lines, respectively. $\boldsymbol{V}_{ik}^{{GRRM}}$ is formed through a linear weighted summation of a series of such pose-only reprojection residuals.} 
    \label{fig:globalresidual} 
\end{figure}

Using (\ref{eq:PA3}) and (\ref{eq:ROMres}), we can express global reprojection residual (\ref{eq:GRRMRes1}) in the form of the rotation manifold:
\begin{equation}
    \scalebox{0.86}{$\boldsymbol{V}_{ik}^{{GRRM}} = \sum_{e_{ij} \in \mathcal{E}_{ik}}\omega_{ijk} \boldsymbol{V}_{ijk}^{{TRRM}, i}\left(\boldsymbol{R}_i \boldsymbol{R}_j^T, \left\{\tilde{\boldsymbol{X}}_{ik}, \tilde{\boldsymbol{X}}_{jk}\right\}_{k=1,\ldots,m_{ij}}\right).$}
    \label{eq:GRRMRes2}
\end{equation}  
The corresponding error function can be expressed as
\begin{equation}
    \scalebox{0.92}{$\begin{aligned}
        f_{{GRRM}}\left(\left\{\boldsymbol{R}_i\right\}_{i=1,\ldots,n}, \left\{\tilde{\boldsymbol{X}}_{ik}\right\}_{\substack{i=1,\ldots,n \\ k=1,\ldots,m}}\right) = \sum_{k=1}^m \sum_{i=1}^n \left\|\boldsymbol{V}_{ik}^{{GRRM}}\right\|^2.
    \end{aligned}$}
\label{eq:GRRMerror}
\end{equation}
Finally, the minimization problem of global rotations can be obtained as
\begin{equation}
    \underset{\{\boldsymbol{R}_i\}_{i=1,\ldots,n}}{\arg\min} \, f_{{GRRM}}\left(\{\boldsymbol{R}_i\}_{i=1,\ldots,n}, \left\{\tilde{\boldsymbol{X}}_{ik}\right\}_{\substack{i=1,\ldots,n \\ k=1,\ldots,m}}\right).
\end{equation}

\section{Algorithm Implementation}

\subsection{Scene Identification}

Based on {Corollary 1} and {Corollary 2}, Scene structures affect the solution space of the observation matrix, leading to potential degradation of the optimization function. It is therefore necessary to identify specific scene structures in advance. Specifically, we need to identify \textit{PR/B/I} first (e.g. by the method in \cite{QiCai_IJCV}). With detection matrices $\boldsymbol{G}_{ij}^i$ and $\boldsymbol{G}_{ij}^j$, we can calculate the corresponding minimum eigenvalues $\lambda_{ij}^i$ and $\lambda_{ij}^j$. Let the detection value for \textit{RotationSingular} structures be

$$
v_{\text{rs}} = \max \left(\lambda_{ij}^i / m_{ij},\lambda_{ij}^j / m_{ij}\right). 
$$

According to {Corollary 2}, if $v_{\text{rs}}$ is less than a prescribed $threshold$, the two-view scene structure can be classified as \textit{RotationSingular} and should be excluded, and further details will be provided in the experiments below.

\subsection{Camera Rotation Estimation}

Our rotation estimation framework consists of two parts: two-view optimization and multi-view optimization, both implemented within the framework of LM optimization. Specifically, Algorithm 1 details the computation process of \textit{TRRM}, while Algorithm 2 explains the implementation of two-view rotation optimization algorithm (\textit{TRRM} optimization algorithm).

\begin{algorithm}[H]
    \caption{Two-View Reprojection Residual Computation}
    \label{alg:ROMComputation}
    \begin{algorithmic}[1]
    \Require Relative rotation $\boldsymbol{R}_{ij}$, observation set $P_{ij} = \left\{\boldsymbol{\tilde{X}}_{ik}, \boldsymbol{\tilde{X}}_{jk}\mid k = 1, \ldots, m_{ij}\right\}$.
    \Ensure  $\boldsymbol{V}_{i j k,bearing}^{TRRM}$.

    \If{pure rotation motion} 
        \State Compute $\boldsymbol{V}_{i j k,bearing}^{TRRM}$ using (\ref{eq:PureRotationReprojBearingError}).
    \Else
        \State Compute $\boldsymbol{P}_{ij}^S$ using (\ref{eq:observeDerive2}), or the fast method in \cite{Kneip13ICCV}.
        \State Compute $\boldsymbol{t}_{ij}^{\pm}$ using (\ref{eq:minlambda}) and (\ref{eq:anaTset}).
        \State Compute $\boldsymbol{V}_{i j k,bearing}^{TRRM}$ using (\ref{eq:ROMresBearing}).
    \EndIf
    
    \State \textbf{Return} $\boldsymbol{V}_{i j k,bearing}^{TRRM}$.
    \end{algorithmic}
    \end{algorithm}
\begin{algorithm}[H]
    \caption{Two-view Rotation Optimization}
    \label{alg:relativeOpt}
    \begin{algorithmic}[1]
    \Require Initial relative rotation $\boldsymbol{R}_{ij}$, observation set $P_{ij} = \left\{\boldsymbol{\tilde{X}}_{ik}, \boldsymbol{\tilde{X}}_{jk}\mid k = 1, \ldots, m_{ij}\right\}$.
    \Ensure Optimized relative rotation $\boldsymbol{R}_{ij}$.
    
    \State Initialize parameters for LM algorithm: damping factor $\lambda$, convergence threshold $\epsilon$, max iterations $k_{\max}$, and $k \gets 0$.
    \If{scene structure classified as \textit{Holoplane}}
        \State \textbf{Return}.
    \Else
        \Repeat
            \State Compute (\ref{eq:ROMerr}), with each $\boldsymbol{V}_{i j k,bearing}^{TRRM}$ from Alg.\ref{alg:ROMComputation}.
            \State Compute the corresponding Jacobian.
            \State Solve LM linear system to update parameters.
            \State $k \gets k + 1$.
        \Until{$Convergence \text{ or } k \geq k_{\max}$}

    \EndIf
    \State \textbf{Return} Optimized $\boldsymbol{R}_{ij}$.
    
    \end{algorithmic}
    \end{algorithm}

    After completing two-view rotation optimization, it is necessary to initialize the global rotation of cameras. For which, various existing algorithms can be selected. For instance, the algorithm \cite{ROBA} employs rotation averaging method by Chatterjee \cite{ChatterjeePAMI2018} as the initialization method, while other options, such as Hartley and Govindu \cite{HartleyL1,Govindu004}, are also available.

    In Algorithm 3, we explain the computation process of \textit{GRRM}, and Algorithm 4 further describes the implementation of multi-view rotation optimization (\textit{GRRM} optimization algorithm).

    \begin{algorithm}[H]
        \caption{Global Reprojection Residual Computation}
        \label{alg:GRRMComputation}
        \begin{algorithmic}[1]
        \Require Camera rotation $\left\{\boldsymbol{R}_{i}, \boldsymbol{R}_{j} \mid e_{ij} \in \mathcal{E}_{ik}\right\}$. Observation set $P_{ik} = \left\{\boldsymbol{\tilde{X}}_{ik}, \boldsymbol{\tilde{X}}_{jk}\mid e_{ij} \in \mathcal{E}_{ik}\right\}$.
        \Ensure $\boldsymbol{V}_{ik}^{{GRRM}}$.
        
        \State Compute $\left\{ \boldsymbol{R}_{ij} \gets \boldsymbol{R}_j \boldsymbol{R}_i^T \mid e_{ij} \in \mathcal{E}_{ik} \right\}$.
        \State Compute $\left\{ \boldsymbol{\theta }_{ijk} \mid e_{ij} \in \mathcal{E}_{ik} \right\}$ using (\ref{eq:theta}).
        \State Compute $\left\{ \omega_{ijk} \mid e_{ij} \in \mathcal{E}_{ik} \right\}$ using (\ref{eq:weight}).
        \State Compute $\left\{ \boldsymbol{V}_{i j k}^{TRRM,i} \mid e_{ij} \in \mathcal{E}_{ik} \right\}$ using Alg.\ref{alg:ROMComputation}.
        \State Compute $\boldsymbol{V}_{ik}^{{GRRM}}$ using (\ref{eq:GRRMRes2}).
        \State \textbf{Return} $\boldsymbol{V}_{ik}^{{GRRM}}$.
        \end{algorithmic}
\end{algorithm}

\begin{algorithm}[H]
    \caption{Multi-view Rotation Optimization}
    \label{alg:globalOpt}
    \begin{algorithmic}[1]
    \Require Initial global rotation $\left\{ {{\boldsymbol{R}}_i\left| {i = 1,...,n} \right.} \right\}$, observation set ${P_o} = \left\{ {{{\boldsymbol{\tilde{X}}}_{ik}} \left| {i = 1,...,n,k = 1,...,m} \right.} \right\}$.
    \Ensure Optimized global rotation $\left\{ {{\boldsymbol{R}}_i\left| {i = 1,...,n} \right.} \right\}$.
    
    \State Initialize parameters for LM algorithm: damping factor $\lambda$, convergence threshold $\epsilon$, max iterations $k_{\max}$, and $k \gets 0$.
    \Repeat
    \State Compute (\ref{eq:GRRMerror}), with each $\boldsymbol{V}_{ik}^{{GRRM}}$ from Alg.\ref{alg:GRRMComputation}.
    \State Compute the corresponding Jacobian.
    \State Solve LM linear system to update parameters.
    \State $k \gets k + 1$.
    \Until{$Convergence \text{ or } k \geq k_{\max}$}
    \State \textbf{Return} Optimized $\left\{ {{\boldsymbol{R}}_i\left| {i = 1,...,n} \right.} \right\}$.
    
    \end{algorithmic}
\end{algorithm}

\section{Experiments}

The experiments are performed on a 64-bit Ubuntu 22.04.3 LTS system platform, equipped with 64GB memory and Gen Intel i9-13900K processor. The experiments are carried out on both simulation and real-world datasets, which include three parts: scene recognition, two-view rotation optimization, and multi-view rotation estimation.

\begin{table*}[htbp]
    \caption{Relative Rotation Initialization Errors under \textit{RotationSingular} Structures}
    \label{table:initialRotationCase2}
    \renewcommand\arraystretch{1.5}
    \resizebox{\textwidth}{!}{
        \begin{tabular}{lccccccc}
            \toprule
            \textbf{Scene structure} & \textbf{5pt (Stewenius)} & \textbf{5pt (Nister)} & \textbf{5pt (Kneip)} & \textbf{7pt} & \textbf{8pt} & \textbf{EigenSolver (10pts)} & \textbf{Nonlin. opt. (10pts)} \\ \midrule
            {\textit{Holoplane}} & 0.6045 & 0.6493 & \textbf{NAN} & 1.6444 & 1.6260 & 0.0610 & 0.3327 \\
            {\textit{RankRegular-line}} & 0.0765 & 0.1054 & \textbf{NAN} & 1.0057 & 1.6409 & 0.0096 & 0.1055 \\
            {\textit{Standard}} & 1.9086$\times 10^{-13}$ & 2.3884$\times 10^{-15}$ & 4.9698$\times 10^{-6}$ & 8.8992$\times 10^{-15}$ & 2.0569$\times 10^{-15}$ & 1.8397$\times 10^{-14}$ & 2.8583$\times 10^{-6}$ \\
            \bottomrule
            \end{tabular}
    }
    \vspace{2mm}
    \par\noindent  
    
    {\footnotesize
    \textit{Note:} 
    Kneip 5-point algorithm outputs an empty result when it fails to compute a solution, indicated by NAN.
    }
\end{table*}

\subsection{Default Settings}
\hspace*{\parindent}
\textbf{Simulation Configurations.}
In the two-view experiments, we use the simulation scene structure \cite{openGVdefaultScene} provided by OpenGV library as the default scene structure and name it \textit{Standard}, where one camera is fixed at the coordinate origin with its rotation set as the identity matrix, the other camera's position is randomly distributed within a spherical space of radius 2 meters centered at the origin, and the relative rotation is generated by sequentially constructing random rotation matrices around the $x-$axis, $y-$axis, and $z-$axis, with rotation angles uniformly sampled within $\left[-0.5, 0.5\right]$ radians. 3D points are distributed inside a sphere centered at the origin with a radius of 20 meters, as shown in Fig. \ref{fig:case3}. The observation noise follows a uniform distribution \( U(0, 5) \) in pixels. Additionally, each experiment ensures that 1000 3D points obeying the chirality constraint are observed by the matching cameras, with images of 960×960 pixels. To ensure the reliability of the experimental results, we perform 5000 Monte Carlo runs for each set of experiments and calculate the mean value of the results.
In the multi-view experiments, we employ our self-developed platform. To simulate realistic large-scale scenarios, the simulation setup includes 100 cameras and 10,000 3D points obeying the chirality constraint, with a fixed distance of 10 meters between adjacent cameras. The image pixel size is consistent with that of the two-view experiments. And we set the number of Monte Carlo experiments to 100.

\textbf{Real-world Dataset.}
We select the Strecha dataset \cite{StrechaDataset} with high-precision camera poses ground truth for evaluating the accuracy and robustness of camera pose estimation algorithms.

\textbf{Pre-computation Configurations.}
In simulation experiments, we need to initialize the relative pose. To ensure the solution quality, we use the initialization algorithms (mainstream Stewénius 5-point method \cite{STEWENIUS2006284} as the default algorithm in scene identification and rotation optimization experiments) combined with the RANSAC framework provided by the OpenGV library \cite{KneipOpenGV}. Our experiments are performed based on the filtered inliers by RANSAC. 
Moreover, following the configuration in \cite{sunghoon2023roba}, our \textit{GRRM} algorithm employs Chatterjee's rotation averaging algorithm \cite{ChatterjeePAMI2018} as the initial input.
In real-world dataset experiments, we employ the algorithms provided by OpenMVG platform \cite{OpenMVGPierre}, such as SIFT \cite{SIFT} algorithm for feature extraction, Cascade hashing \cite{cascadeMatch} for view matching, and 5-point algorithm \cite{STEWENIUS2006284,NisterRelative} combined with the AC-RANSAC \cite{ACRANSAC} framework for relative pose initialization. 
Additionally, while performing global optimization or generating BAL files,
the LiGT algorithm \cite{QiCai_TPAMI} and triangulation reconstruction algorithm \cite{hartleyTriangulation1995} are used to compute global camera translations and 3D point coordinates, respectively.

\textbf{Accuracy Evaluation.}
In simulation experiments, the camera rotation error (rads) is quantified as follows:
\begin{equation}
    \epsilon^{{error}} = \arccos\left(\frac{{trace}\left(\boldsymbol{R}_{gt}^T \boldsymbol{R}_e\right) - 1}{2}\right),
    \label{eq:rotationerrorcalculate}
\end{equation}
where the rotation estimate is \( \boldsymbol{R}_e \) and the ground truth rotation is \( \boldsymbol{R}_{gt} \).
In real-world experiments, the rotation error (degs) in the two-view experiments is calculated with (\ref{eq:rotationerrorcalculate}), while the multi-view experiments use evaluation module \cite{openMVG_evalQuality} provided by OpenMVG platform for computation.

\begin{figure}[t]
    \centering
    \begin{subfigure}[t]{0.48\linewidth} 
        \captionsetup{justification=centering, singlelinecheck=false, position=below} 
        \includegraphics[width=\linewidth]{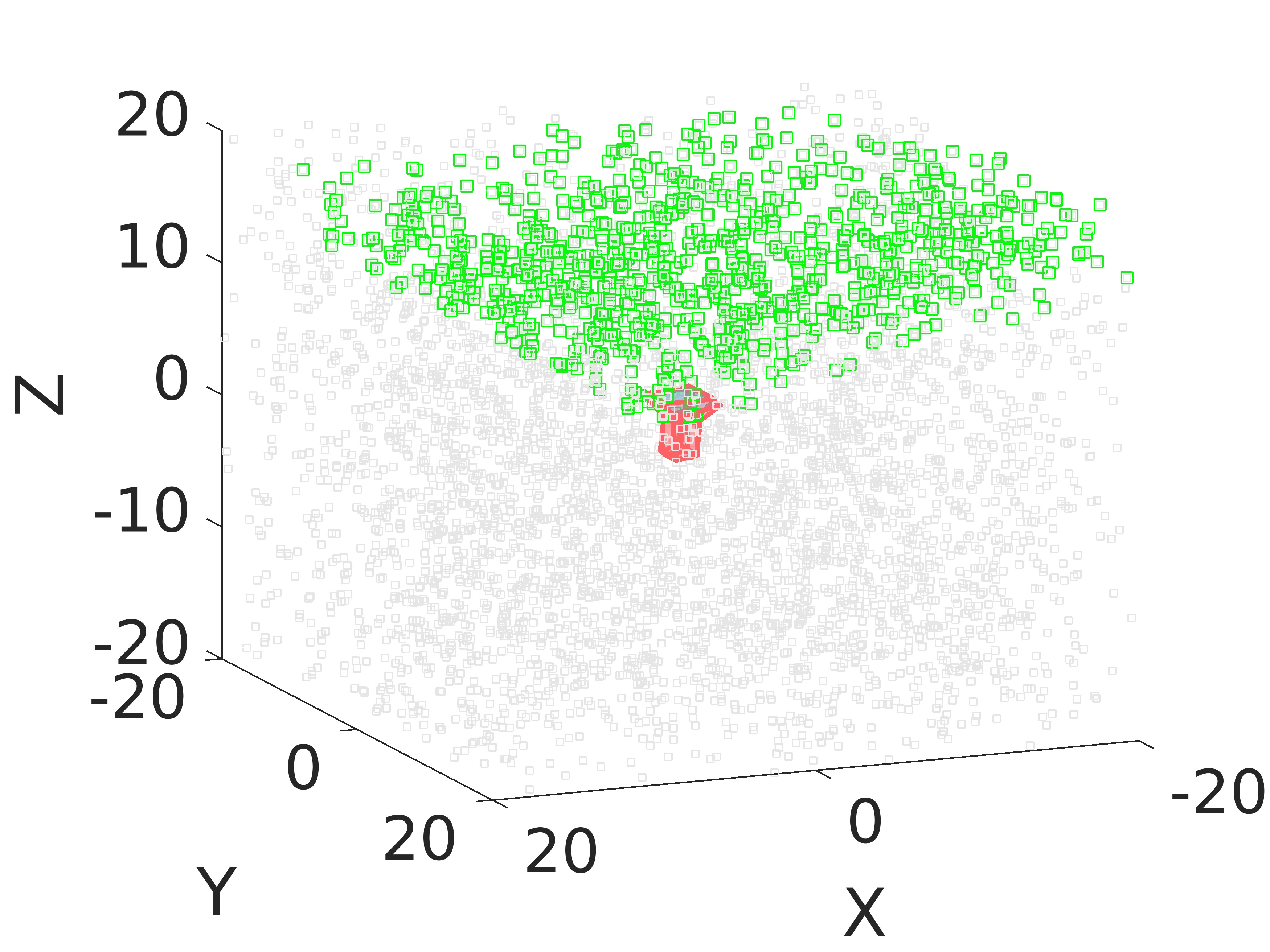}
        \caption{PureRotation}

        \label{fig:case1}
    \end{subfigure}
    \hfill
    \begin{subfigure}[t]{0.48\linewidth}
        \captionsetup{justification=centering, singlelinecheck=false, position=below} 
        \includegraphics[width=\linewidth]{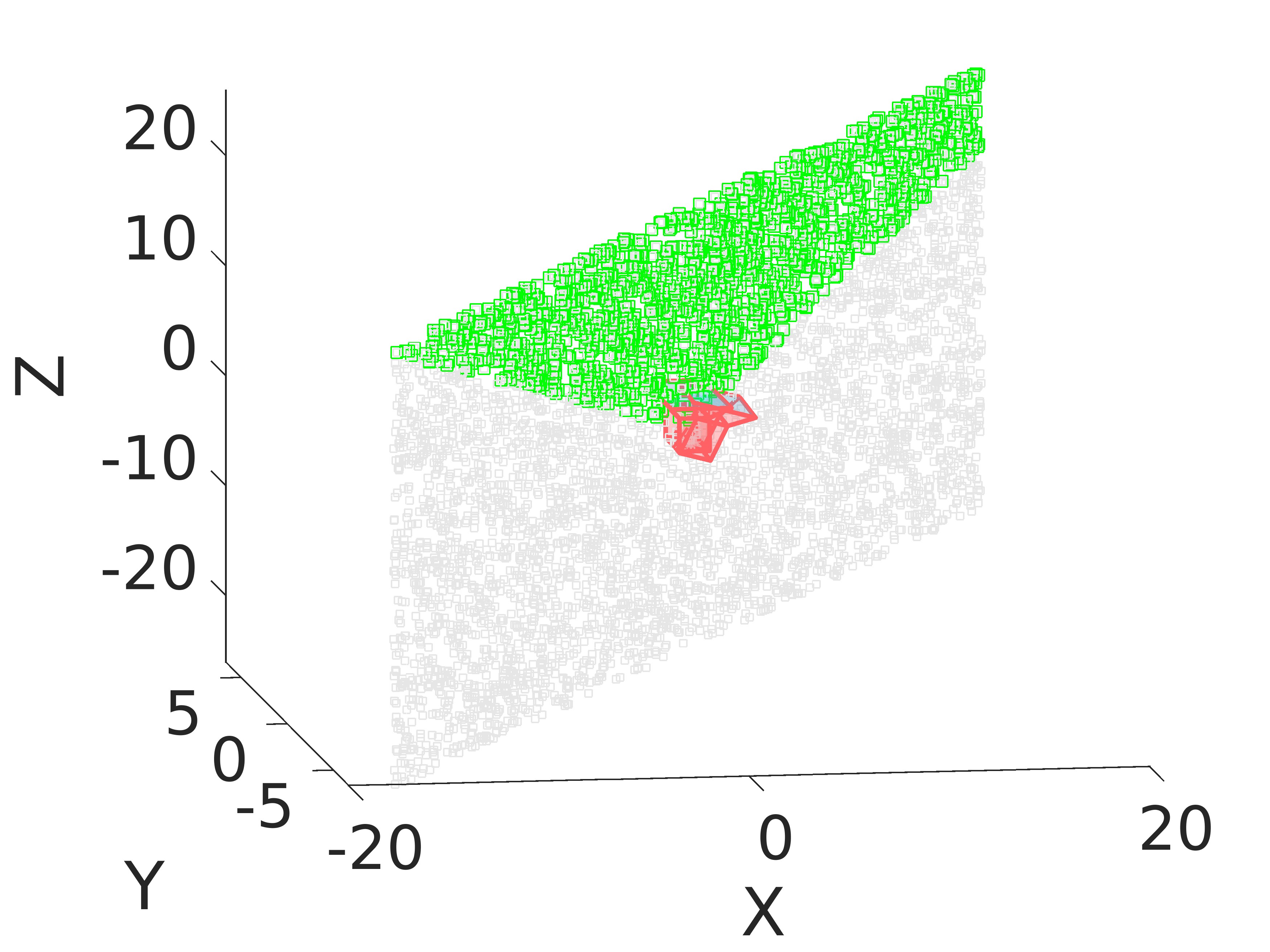}
        \caption{Holoplane} 
        \label{fig:case2}
    \end{subfigure}

    \begin{subfigure}[t]{0.48\linewidth}
        \captionsetup{justification=centering, singlelinecheck=false, position=below} 
        \includegraphics[width=\linewidth]{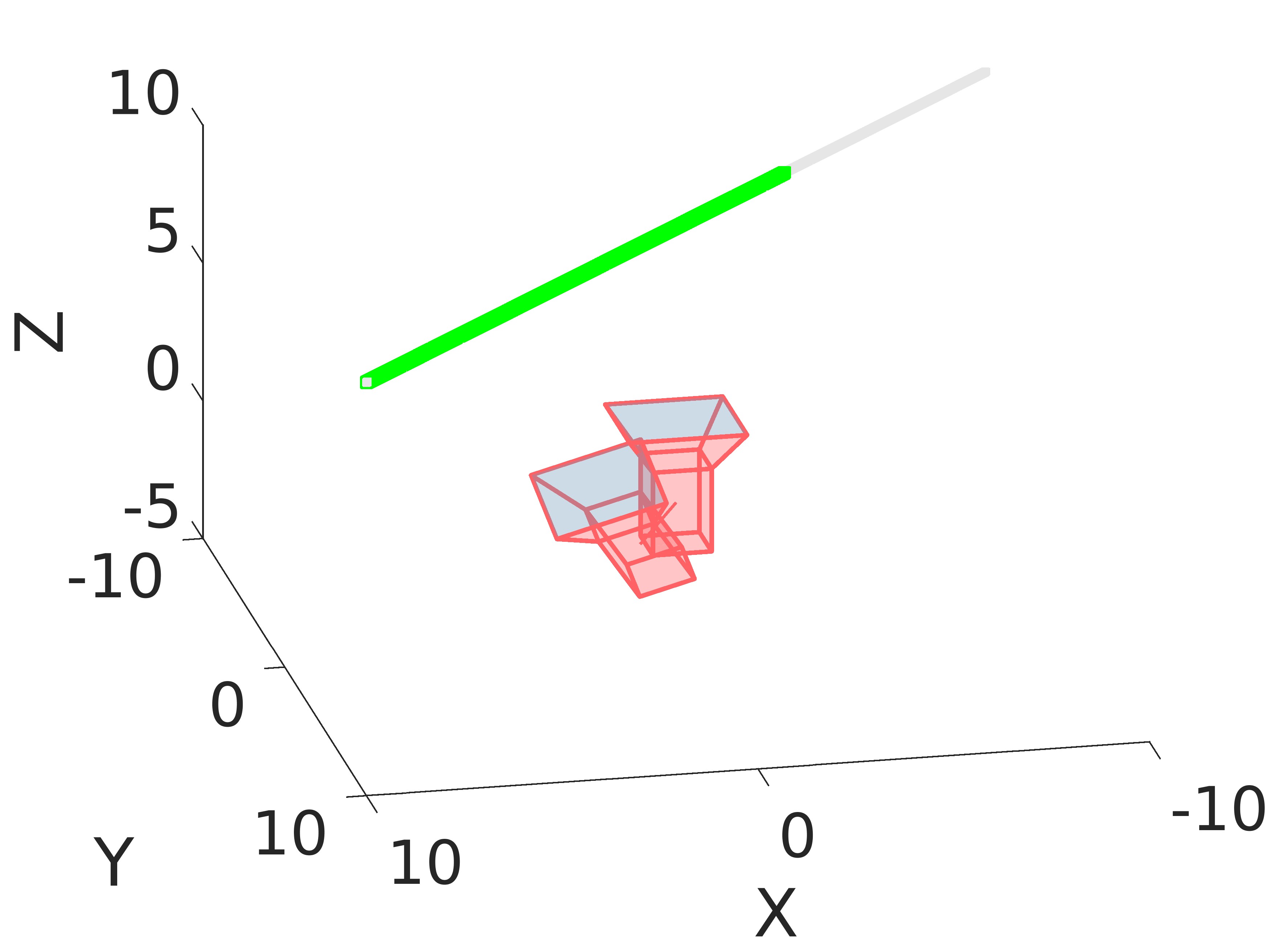}
        \caption{RankRegular-line}
        \label{fig:line}
    \end{subfigure}
    \hfill
    \begin{subfigure}[t]{0.48\linewidth}
        \captionsetup{justification=centering, singlelinecheck=false, position=below} 
        \includegraphics[width=\linewidth]{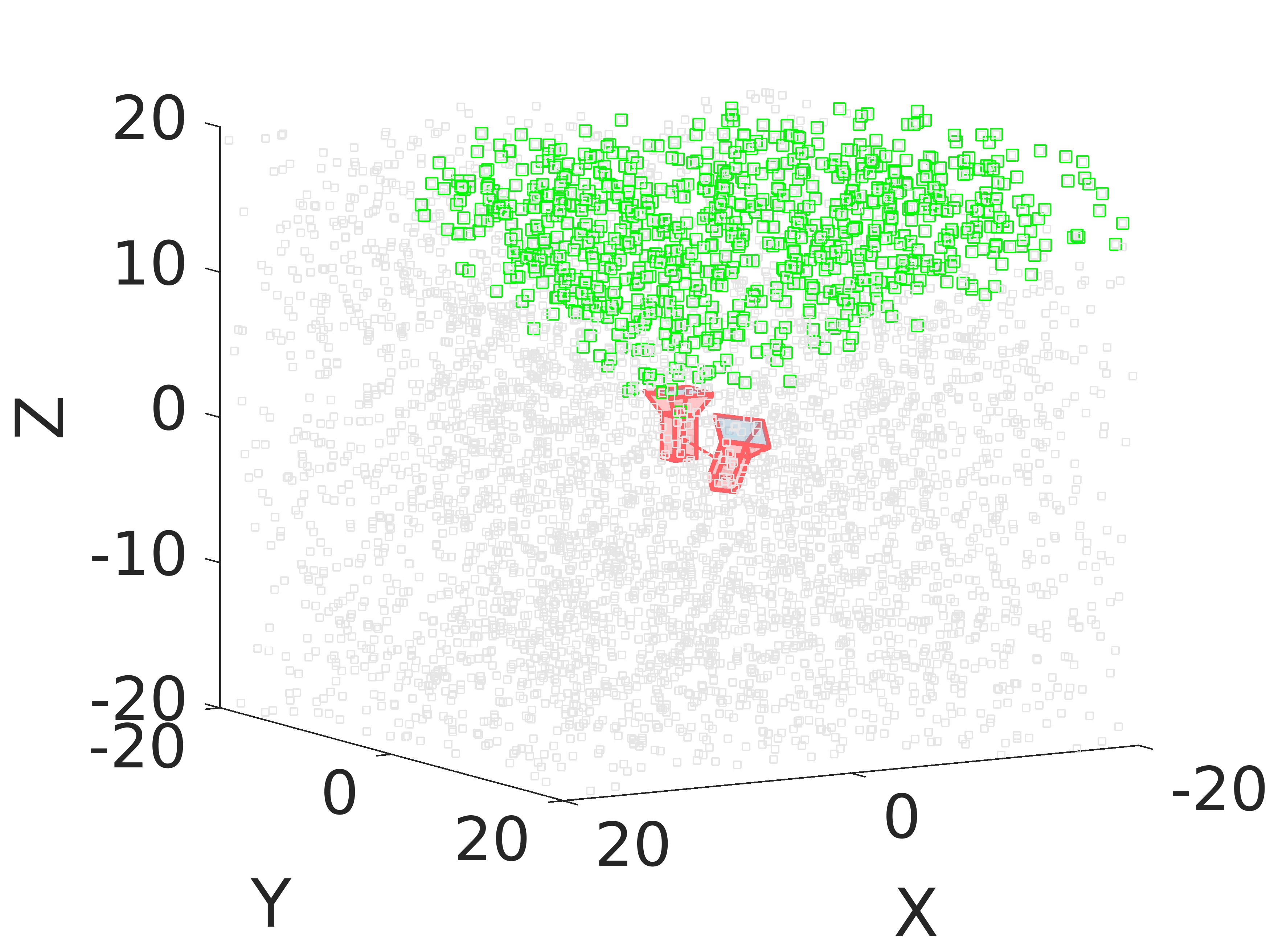}
        \caption{Standard}
        \label{fig:case3}
    \end{subfigure}

    \begin{subfigure}[t]{0.48\linewidth}
        \captionsetup{justification=centering, singlelinecheck=false, position=below} 
        \includegraphics[width=\linewidth]{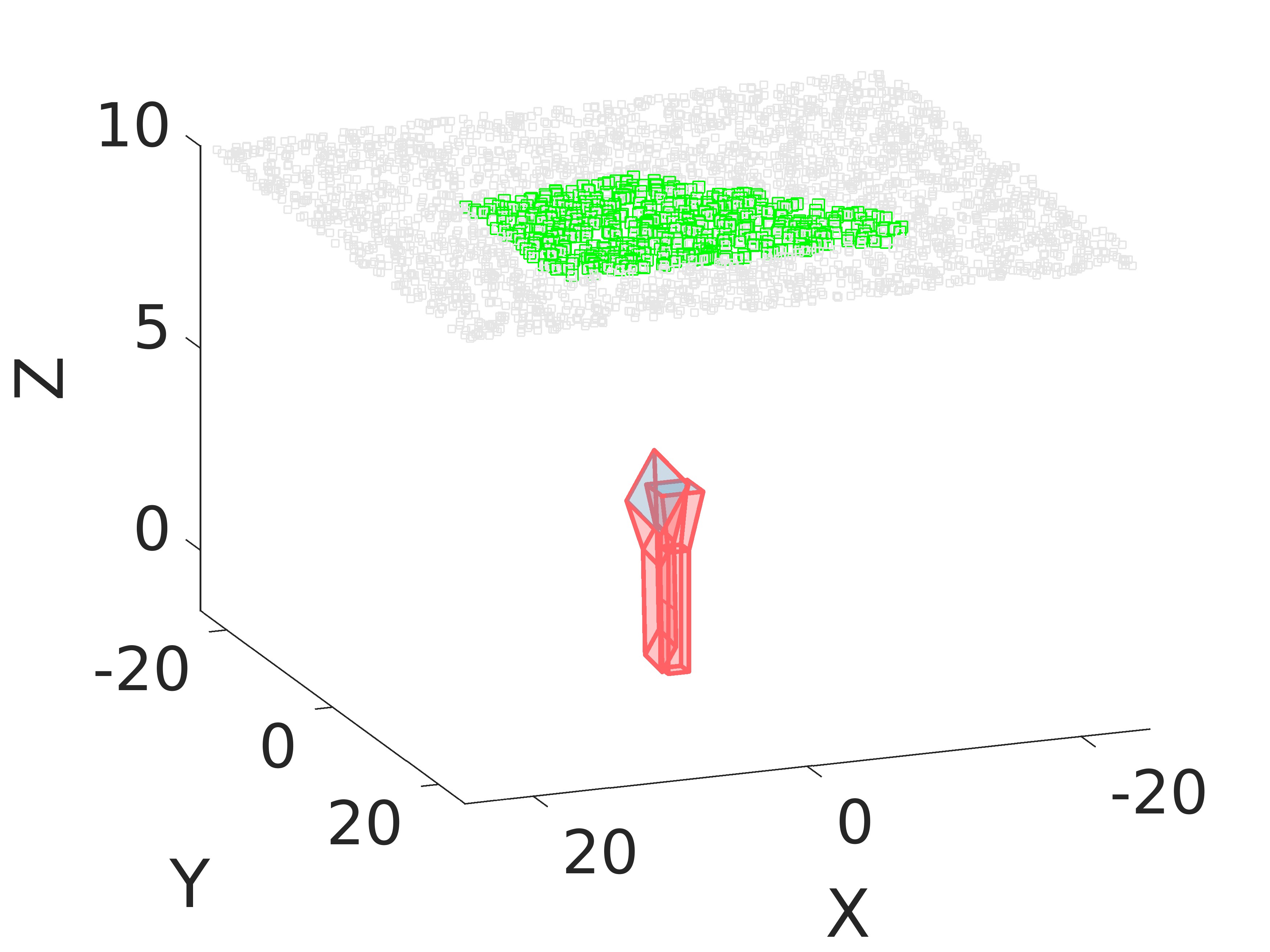}
        \caption{PlanarScene}
        \label{fig:planar1}
    \end{subfigure}
    \hfill
    \begin{subfigure}[t]{0.45\linewidth}
        \captionsetup{justification=centering, singlelinecheck=false, position=below} 
        \includegraphics[width=\linewidth]{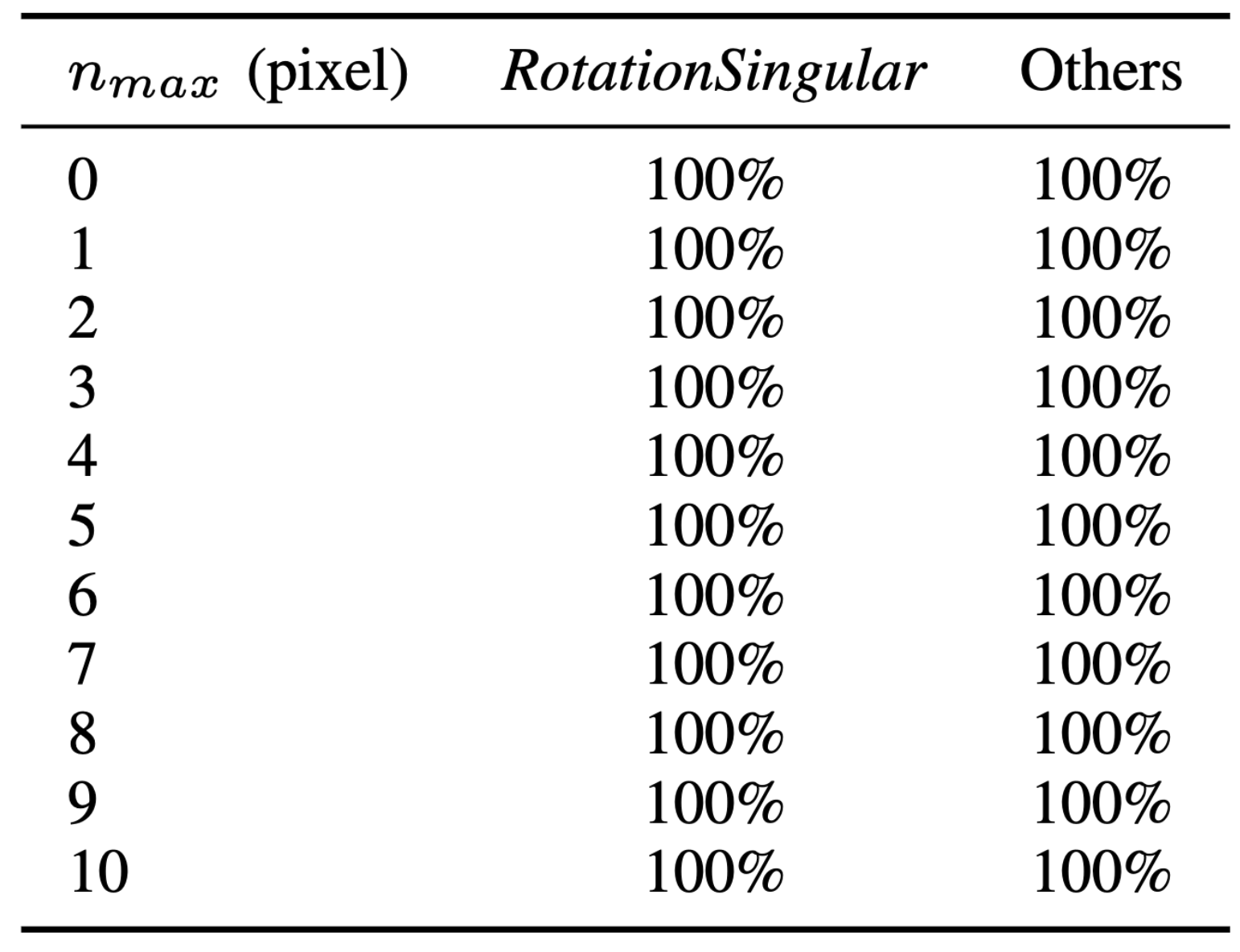}
        \caption{Recognition rate}
        \label{fig:table_recog}
    \end{subfigure}

    \caption{ (a)-(e) Scene structures in simulation, where cameras are depicted as red models, 3D points distribute throughout space are represented as gray dots, and 3D points observed by cameras are highlighted in green. (f) Recognition rate of our proposed algorithm}
    \label{fig:SceneStructure1}
\end{figure}

\begin{figure*}[htbp]
    \centering
    \begin{subfigure}[b]{0.49\textwidth}
        \includegraphics[width=\textwidth]{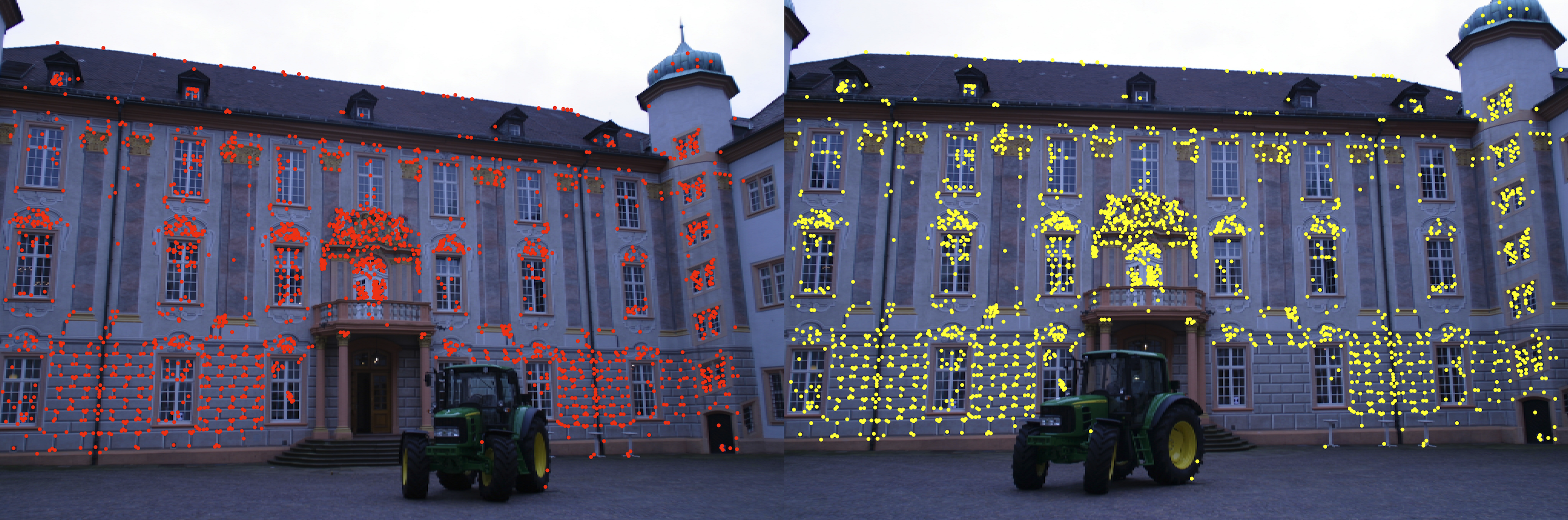}
        \captionsetup{justification=centering, singlelinecheck=false, position=below} 
        \caption{$v_{\text{rs}}=1.592 \times 10^{-2}$}
        \label{fig:CasltleLargeV1}
    \end{subfigure}
    \hfill
    \begin{subfigure}[b]{0.49\textwidth}
        \includegraphics[width=\textwidth]{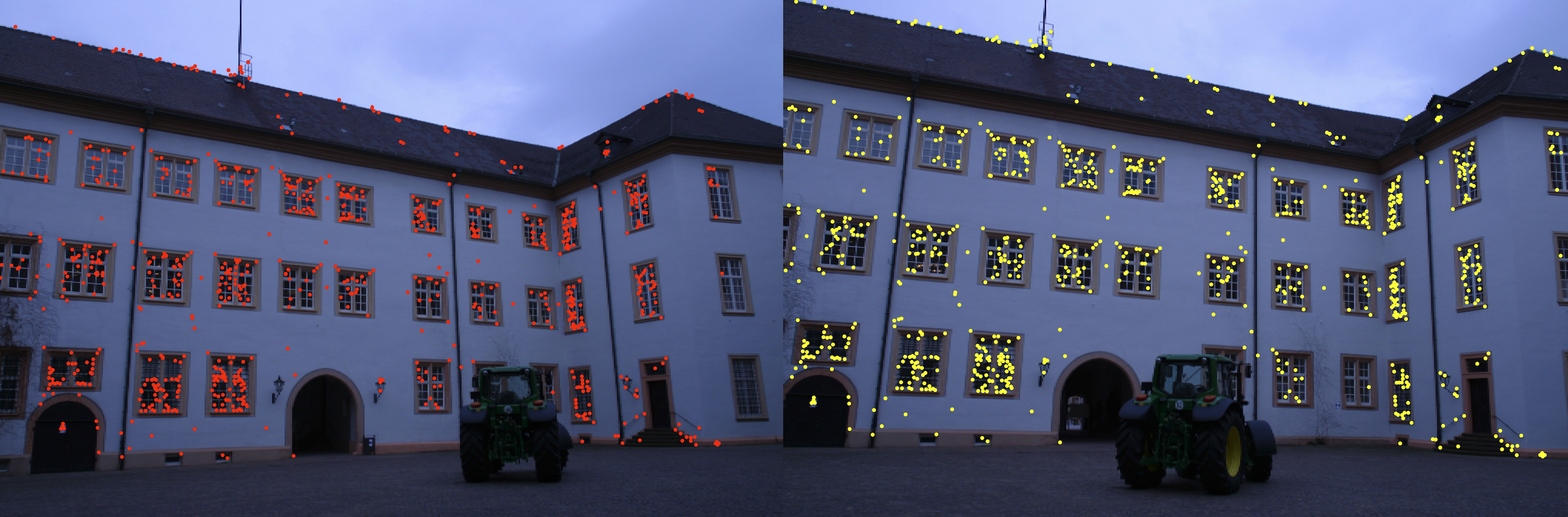}
        \captionsetup{justification=centering, singlelinecheck=false, position=below} 
        \caption{$v_{\text{rs}}=1.735 \times 10^{-2}$}
        \label{fig:CasltleLargeV2}
    \end{subfigure}
        \begin{subfigure}[b]{0.49\textwidth}
            \includegraphics[width=\textwidth]{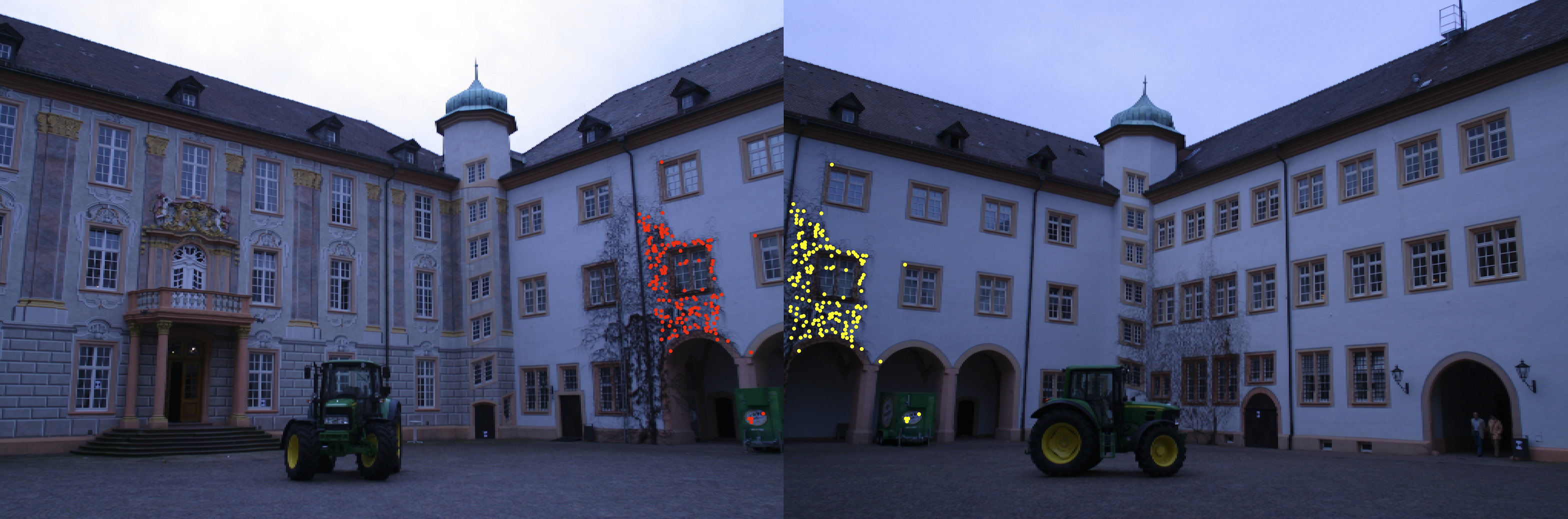}
            \captionsetup{justification=centering, singlelinecheck=false, position=below} 
            \caption{$v_{\text{rs}}=7.012 \times 10^{-4}$}
            \label{fig:CasltleSmallV1}
        \end{subfigure}
        \hfill
        \begin{subfigure}[b]{0.49\textwidth}
            \includegraphics[width=\textwidth]{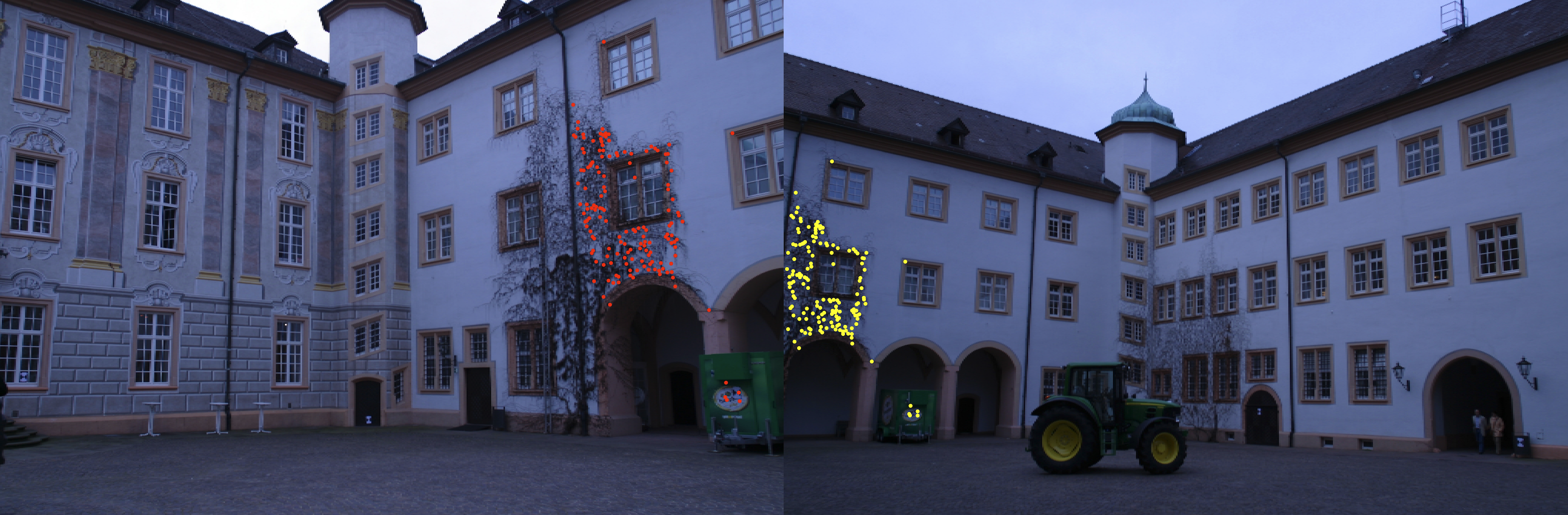}
            \captionsetup{justification=centering, singlelinecheck=false, position=below} 
            \caption{$v_{\text{rs}}=1.424 \times 10^{-3}$}
            \label{fig:CasltleSmallV2}
        \end{subfigure}
    \caption{Detection value of Castle sub-dataset of Strecha dataset. (a)(b) and (c)(d) differ by an order of magnitude in $v_{\text{rs}}$. Red and yellow points represent matched feature points in left and right views, respectively.}
    \label{fig:CastleCase2}
\end{figure*}

\begin{figure}[htbp]
    \centering
    \begin{subfigure}[t]{0.5\linewidth}

        \includegraphics[width=\linewidth]{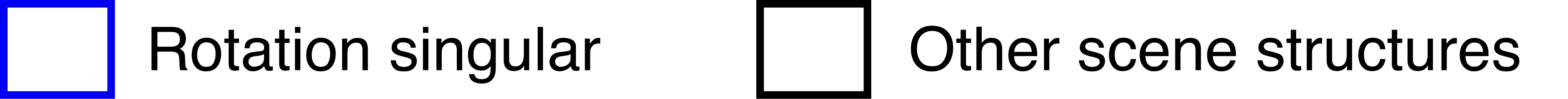}
    \end{subfigure}

    \begin{subfigure}[t]{0.9\linewidth}

        \includegraphics[width=\linewidth]{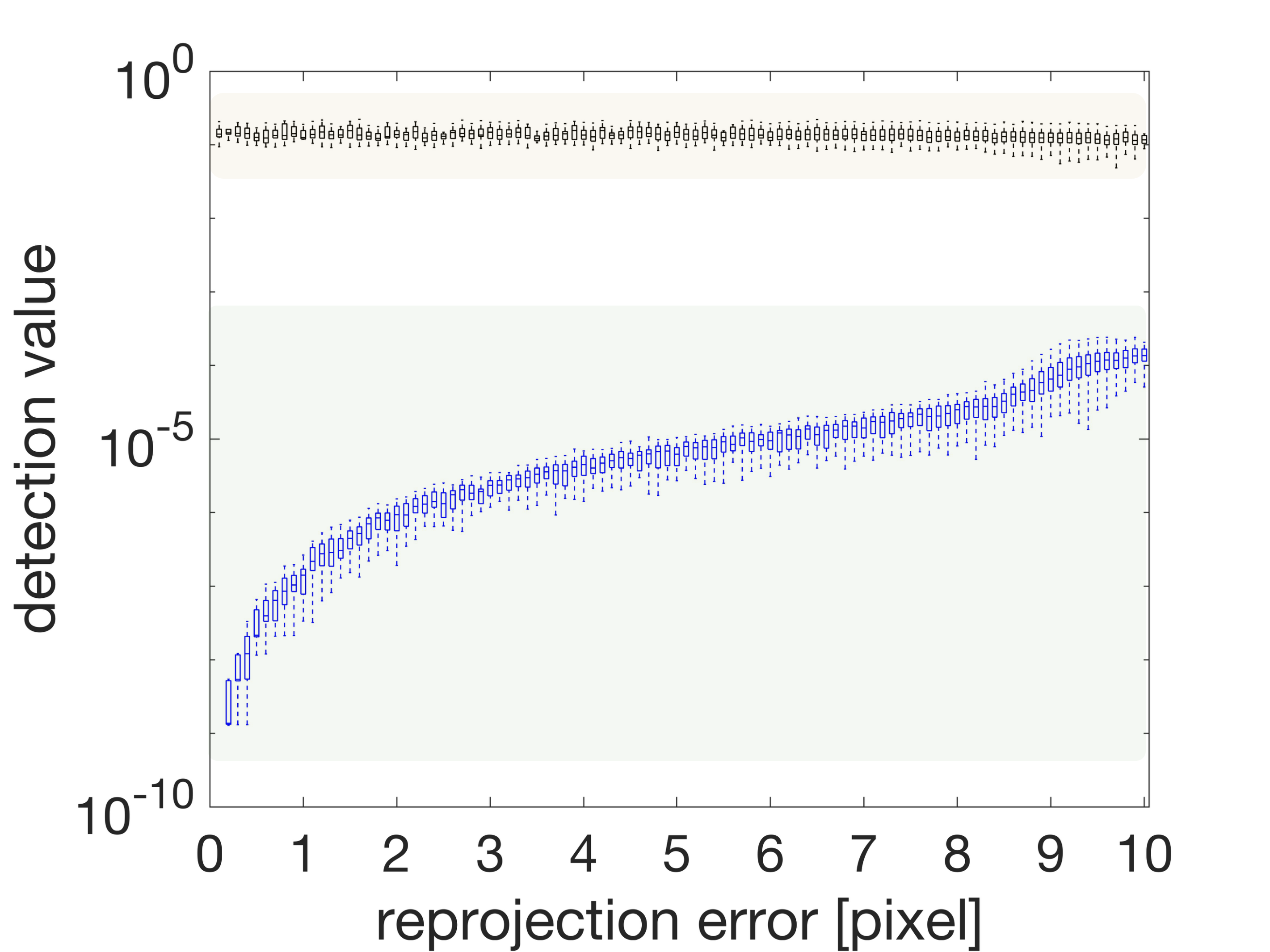}
        \captionsetup{justification=centering, singlelinecheck=false, position=below} 
    \end{subfigure}
    \hfill    

    \caption{Detection values of \( v_{\text{rs}} \) in different scene structures. 
    With growth of repetition error, detection value of rotation singular scenes increases as well, but it is at least an order of magnitude lower than those of other scene structures.
    }
    \label{fig:recog_singular}
\end{figure}

\subsection{Scene Identification}
\subsubsection{Simulation Experiments}
As discussed in section \ref{section:TWOVIEW}, \textit{RotationSingular} scene structure is a special type of degenerate scene structure in which most initial value solution methods fail. 
In the experiments, we select such algorithms as Stewénius 5-point method \cite{STEWENIUS2006284}, Nister 5-point method \cite{NisterRelative}, Kneip 5-point method \cite{KneipECCV2012}, 7-point method \cite{MultipleViewGeometry}, 8-point method \cite{LONGUETHIGGINS198761}, EigenSolver 10-point method \cite{Kneip13ICCV}, and the nonlinear 10-point optimization algorithm \cite{KneipOpenGV} for validation.
Apart from control scene structure \textit{Standard}, we also construct \textit{Holoplane} and \textit{RankRegular-line}, both of which share the default camera pose configuration.
3D points of \textit{Holoplane} are placed on a square plane with a side length of 20 meters, passing through the matched cameras and the coordinate \([0, 0, 1]\), whereas those of \textit{RankRegular-line} are distributed along a random straight line on a plane $Z=10$, as illustrated in Fig. \ref{fig:case2} and Fig. \ref{fig:line}, respectively. 
The average rotation estimation errors are as shown in Table \ref{table:initialRotationCase2}. The results demonstrate that even under noise-free condition, all initialization algorithms exhibit considerable errors or fail to compute a solution entirely. Therefore, to ensure the accuracy of camera pose estimation, it is necessary to identify and exclude \textit{RotationSingular} scene structure.

Subsequently, we evaluate the proposed \textit{RotationSingular} structures identification method. 
We construct \textit{PureRotation} scene structure, where all cameras are placed at the coordinate origin, as shown in Fig. \ref{fig:case1}. \textit{PureRotation} and \textit{Standard} serve as the control group, representing typical scene structures of \textit{PR/B/I} and \textit{RankRegular}, respectively.
Other simulation configurations follow the default setting.
We calculate the reprojection errors and statistically analyze the distribution of \( v_{\text{rs}} \) within the $0.01$ pixel interval across different scenes, as shown in Fig. \ref{fig:recog_singular}. It is evident that the difference between the \textit{RotationSingular} scene structures and others are significant up to an order of magnitude.

Empirically, we set the threshold to \( 10^{-4} \), and for different scenes, we assume that the noise follows a uniform distribution \( U(0, n_{max}) \) in pixels. Fig. \ref{fig:table_recog} summarizes the recognition rate under different scenes and $n_{max}$ conditions. It is clear that as the maximum noise increases from 0 to 10 pixels, the recognition rate of our algorithm remains $100 \%$ in all cases.

\begin{figure*}[htbp] 
    \centering
    \begin{subfigure}[t]{0.62\linewidth}
        \includegraphics[width=\linewidth]{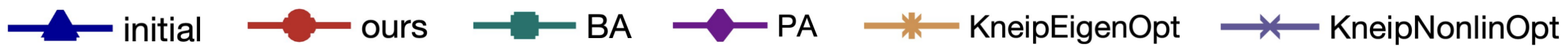}
        \label{fig:relativeOptlegend}

    \end{subfigure}

    \begin{subfigure}[t]{0.06\linewidth}
        \includegraphics[width=\linewidth]{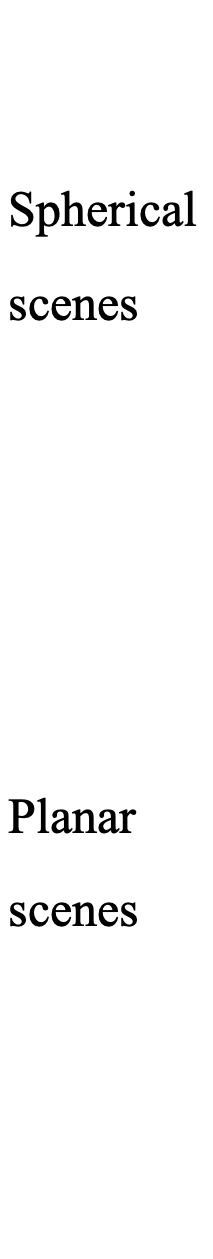}
        \captionsetup{justification=centering, singlelinecheck=false, position=below} 
        \setlength{\abovecaptionskip}{0pt}  
        \setlength{\belowcaptionskip}{5pt}  

    \end{subfigure}
    \begin{subfigure}[t]{0.23\linewidth}
        \includegraphics[width=\linewidth]{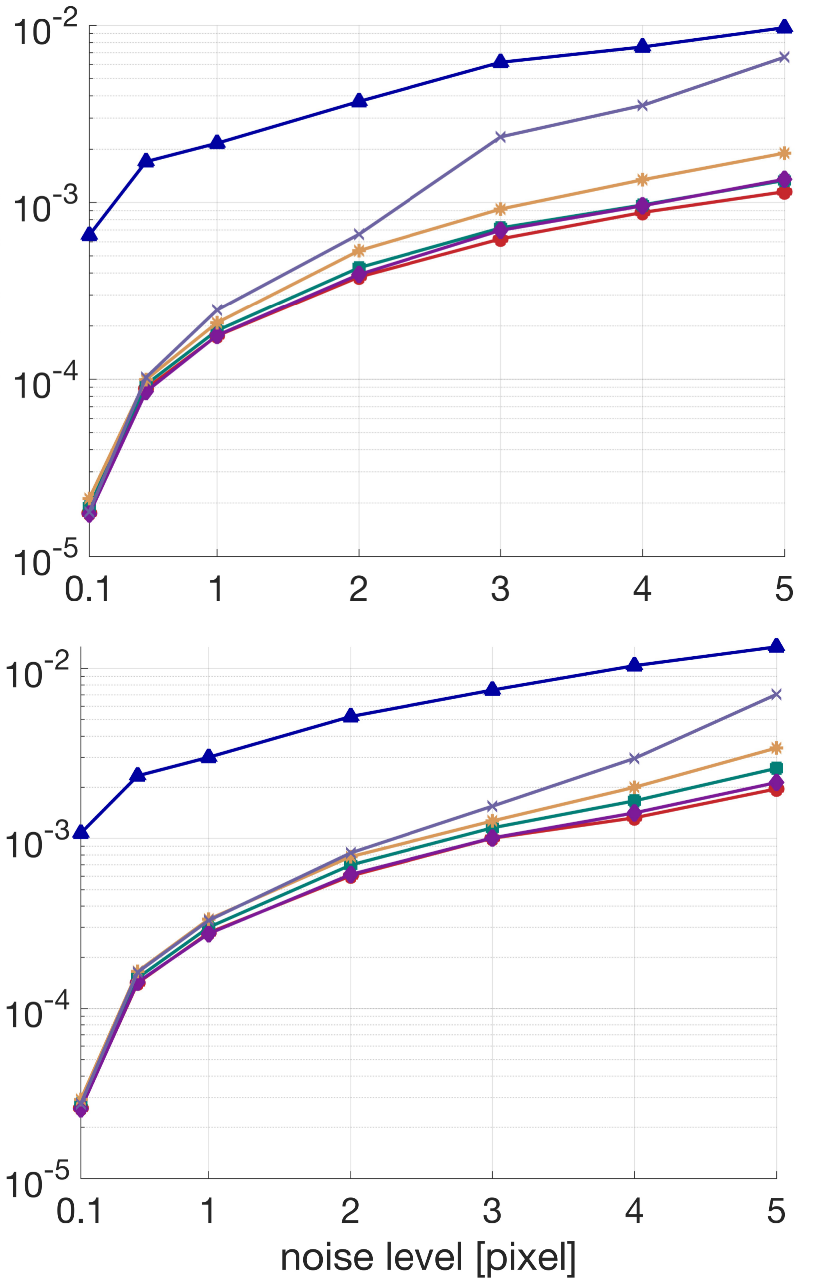}
        \captionsetup{justification=centering, singlelinecheck=false, position=below} 
        \setlength{\abovecaptionskip}{0pt}  
        \setlength{\belowcaptionskip}{5pt}  

        \caption{Noise}
        \label{fig:relativeOpt1}

    \end{subfigure}
    \begin{subfigure}[t]{0.23\linewidth}
        \includegraphics[width=\linewidth]{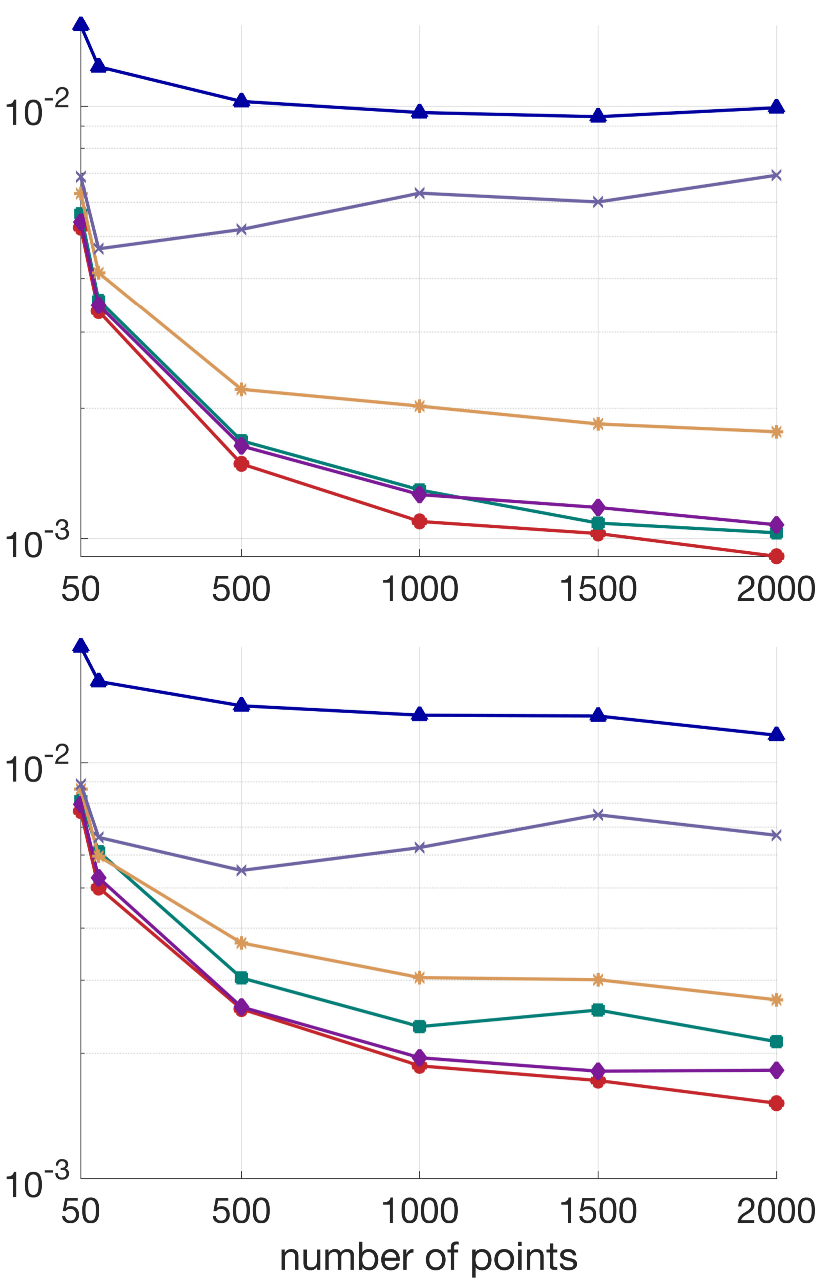}
        \captionsetup{justification=centering, singlelinecheck=false, position=below} 
        \setlength{\abovecaptionskip}{0pt}  
        \setlength{\belowcaptionskip}{5pt}  
        \caption{3D points number}
        \label{fig:relativeOpt2}
    \end{subfigure}
    \begin{subfigure}[t]{0.23\linewidth}
        \includegraphics[width=\linewidth]{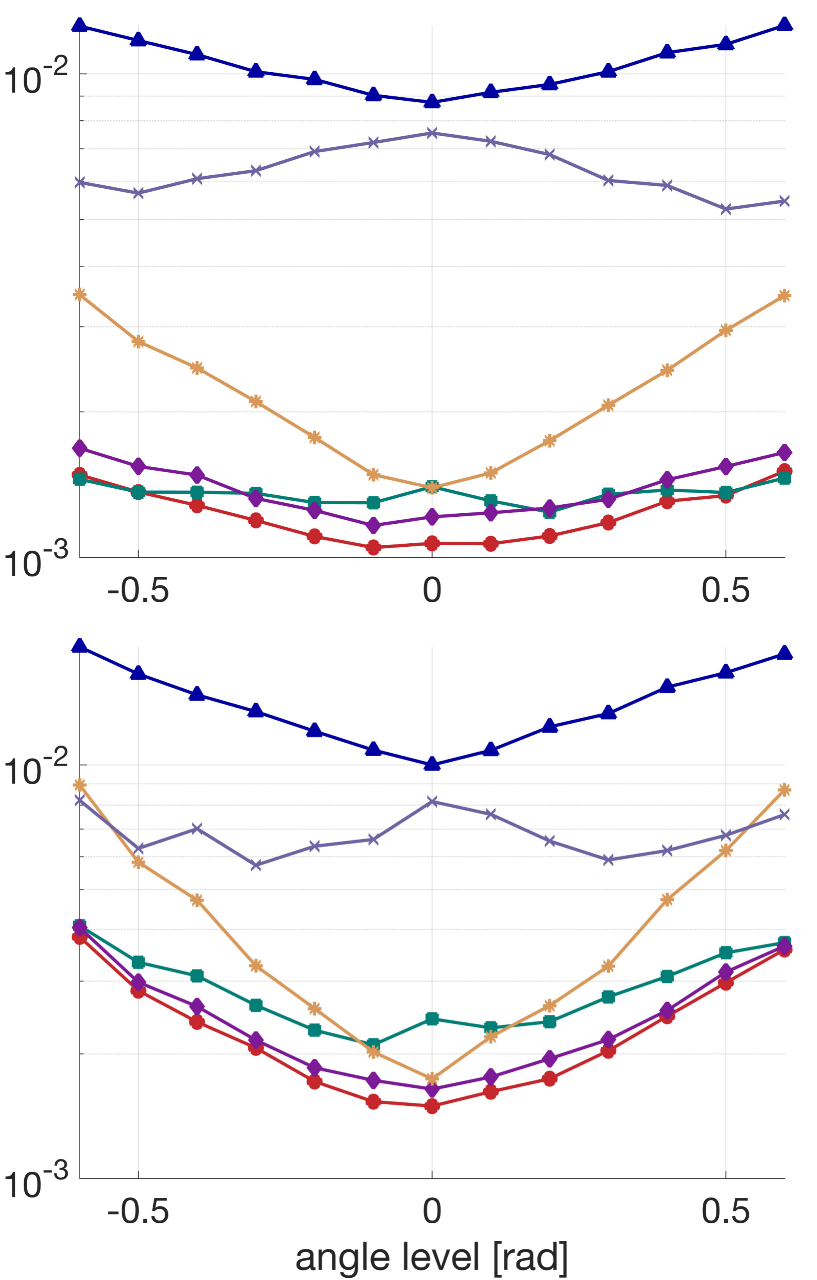}
        \captionsetup{justification=centering, singlelinecheck=false, position=below} 
        \setlength{\abovecaptionskip}{0pt}  
        \setlength{\belowcaptionskip}{5pt}  
        \caption{Rotation angle}
        \label{fig:relativeOptRot}
    \end{subfigure}
    \begin{subfigure}[t]{0.23\linewidth}
        \includegraphics[width=\linewidth]{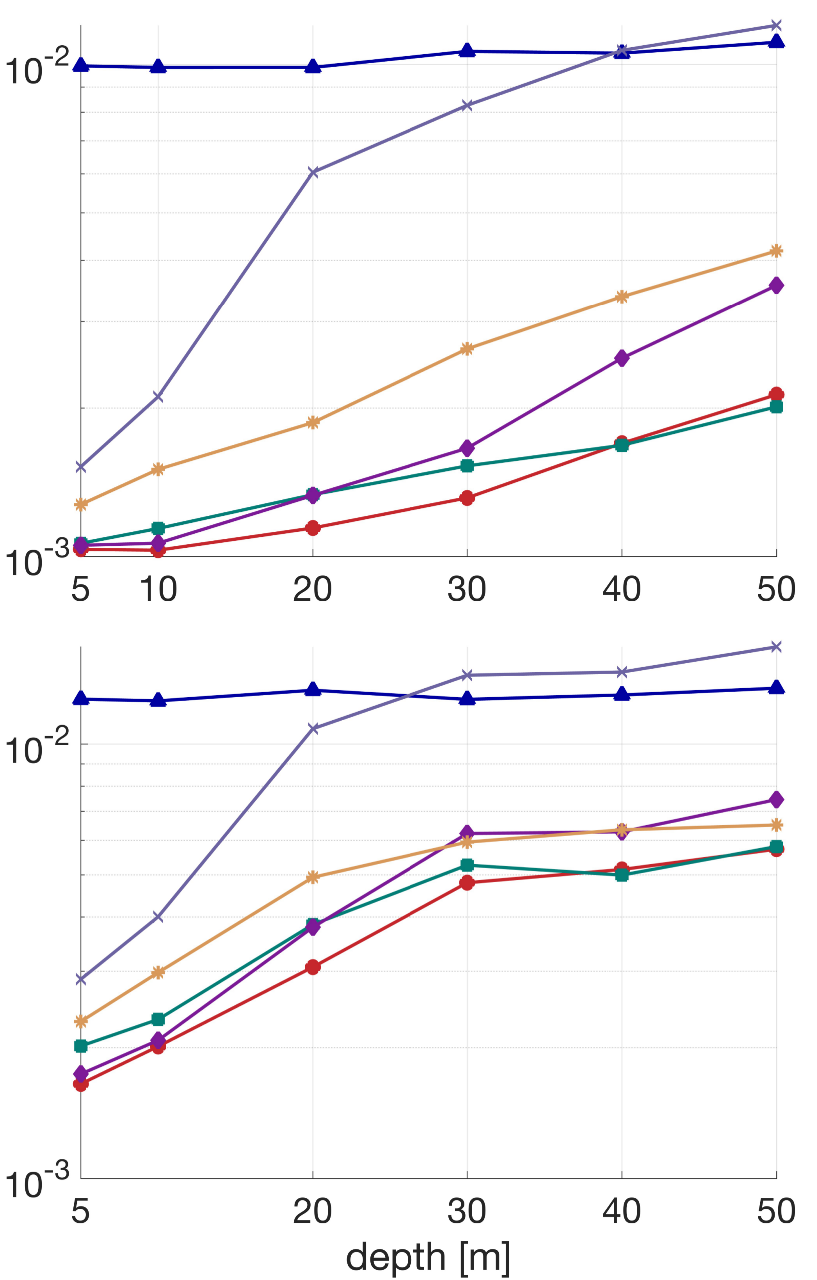}
        \captionsetup{justification=centering, singlelinecheck=false, position=below} 
        \setlength{\abovecaptionskip}{0pt}  
        \setlength{\belowcaptionskip}{5pt}  
        \caption{Scene depth}
        \label{fig:relativeOpt3}
    \end{subfigure}
    \caption{Accuracy of two-view optimization algorithms under different control variables.
    }
    \label{fig:relativeOpt}
\end{figure*}

\subsubsection{Real-world Dataset Experiments}
In the real-world applications of SfM, it is extremely rare to encounter scene structures that strictly conform to the \textit{RotationSingular} scene structure configuration. For most scene structures that are not classified as \textit{RotationSingular} scene structures, we observe an interesting phenomenon: the detection value $v_{\text{rs}}$ is proportional to the breadth of the matching feature points. For example, when the points are dispersed, the detection value is large. To verify this pattern, we perform experiments using the Castle sub-dataset from the Strecha dataset. As shown in Fig. \ref{fig:CastleCase2}, four sets of matched views are presented to exhibit significant differences in \(v_{\text{rs}}\), up to more than one order of magnitude, which further highlights and validates our observed pattern.

\subsection{Two-view Rotation Optimization}
We employ such algorithms as PA, BA, Kneip's Eigen Optimization (KneipEigenOpt) \cite{Kneip13ICCV}, and Kneip's nonlinear optimization (KneipNonlinOpt) \cite{KneipOpenGV} as benchmark methods. A comprehensive performance comparison is systematically conducted across both simulation environments and real-world datasets, aiming to thoroughly validate the effectiveness and practical value of our proposed \textit{TRRM} optimization in various application scenarios.

\subsubsection{Simulation Experiments}
In each simulation, scene structures \textit{Standard} and \textit{PlanarScene} (with all 3D points on the plane \(Z=10\), as shown in Fig. \ref{fig:planar1}) are selected as the experimental subjects.
In the following analysis, we examine four control variables: 
observation noise, 3D point numbers, relative rotation and scene depth. 
For each experiment, all non-controlled parameters are maintained at default simulation configurations.

\textit{Effect of observation noise:} 
The noise follows a uniform distribution $U(0, n_N)$ in pixels, for \(n_N = 0.1,1,2,...,5\) pixels. 
The accuracy results are shown in Fig. \ref{fig:relativeOpt1}, from which it can be observed that the algorithms based on reprojection error minimization, such as BA, PA, and our proposed \textit{TRRM} optimization, demonstrate better accuracy than alternative approaches, particularly those characterized by  lower-dimensional optimization parameters, such as \textit{TRRM} achieving the highest performance and PA following closely. The performance gap is particularly evident in \textit{PlanarScene} and under large observation noise.

\textit{Effect of 3D point numbers:}
We configure the number of 3D points as 50, 100, 500, 1000, 1500 and 2000, and
the experimental accuracy results are shown in Fig. \ref{fig:relativeOpt2}. 
While our algorithm maintains the best accuracy, we note an intriguing trend: 
as the number of 3D points increases, the accuracy of most algorithms improves, and in particular, BA slightly surpasses PA in \textit{PlanarScene}. 
In contrast, Kneip's nonlinear optimization exhibits minimal improvement and even a slight accuracy degradation, which may be attributed to the numerical instability caused by the incomplete decoupling between two error terms of reprojection in the chordal distance dimension.

\textit{Effect of relative rotation:}
Starting from zeros, we incrementally increase the amplitude of the rotation angle around the \(x\)-axis, \(y\)-axis and \(z\)-axis (within the range of $[-\pi/2 , \pi/2]$) until sufficient matching feature points are lost,
and the matched cameras are situated at the coordinate origin and $[1, 0, 0]$, respectively.
The experimental accuracy results are shown in Fig. \ref{fig:relativeOptRot}. It can be observed that as the rotation angle amplitude increases, the accuracy of the initial estimation and most optimization solutions decreases slightly. 
This may be caused by the gradual compression of the camera's common 3D viewing space (tending towards \textit{RotationSingular} scene structure). 
Our algorithm generally achieves the highest accuracy, 
though susceptible to such large angle amplitude conditions.

\textit{Effect of scene depth:}
In scene structures \textit{Standard} and \textit{PlanarScene}, we set the radius of the sphere and the height of the plane where the 3D points are distributed to be 5, 10, 20, 30, 40 and 50, respectively. 
The experimental accuracy results are shown in Fig. \ref{fig:relativeOpt3}. 
It can be observed that as the depth of the scenes increases, the accuracy of various algorithms decreases, because the camera measurement errors exhibit a distance-propagated magnification effect.
Our algorithm delivers the best accuracy for the majority of depth configurations, but when the 3D depth is excessively large (exceeding 20 times the baseline length), its accuracy becomes comparable to that of BA. Kneip's nonlinear optimization performs notably worse, even inferior to initial accuracy, which may be connected to the previously discussed stability problems.

\subsubsection{Real-world Dataset Experiments}
Table \ref{table:strecha_results_relativeopt} summarizes the accuracy of various algorithms for each sub-dataset. The results clearly demonstrate that our proposed \textit{TRRM} optimization algorithm consistently achieves the best accuracy on the Strecha dataset. Notably, 
the accuracy has an averaged $17.01\%$ improvement over the second best algorithm.

\begin{table*}[htbp]
    \caption{Accuracy of Two-view Optimization Algorithms on Strecha Dataset} 
\centering 
\begin{tabular}{lcccccc}
    \toprule
    \textbf{Methods}    & \textbf{Herz-Jesus-P25} & \textbf{Herz-Jesus-P8} & \textbf{Castle-P19} & \textbf{Castle-P30} & \textbf{Entry-P10} & \textbf{Fountain-P11} \\ \midrule
    initial           & 0.1663                  & 0.1641                & 0.5407              & 0.5540              & 0.1511            & 0.1099               
    \\
    BA                & 0.0806                  & \textbf{0.0502}        & \textbf{0.3404}              & \textbf{0.2883}              & 0.0800            & 0.0567               \\
    PA                & 0.0827                  & 0.0525                & 0.3456              & 0.3055              & 0.1077            & \textbf{0.0566}      \\
    Kneip Eigen Opt   & \textbf{0.0801}         & 0.0514                & 0.4084              & 0.3118              & \textbf{0.0794}            & 0.0571               \\
 Kneip Nonlin Opt& 0.0953& 0.0680& 0.4016& 0.3062& 0.0915&0.0667\\
    \textbf{ours}     & \textcolor{red}{\textbf{0.0657}} & \textcolor{red}{\textbf{0.0497}} & \textcolor{red}{\textbf{0.2109}} & \textcolor{red}{\textbf{0.2321}} & \textcolor{red}{\textbf{0.0598}} & \textcolor{red}{\textbf{0.0561}} \\     
    {improvement ratio}     & $17.98\%$ & $0.996\%$ & $38.04\%$ & $19.49\%$ & $24.69\%$ & $0.883\%$ \\ \bottomrule
\end{tabular}
\label{table:strecha_results_relativeopt}

\raggedright 
\vspace{2mm}

\textit{Note:} The best results are highlighted in bold red, and the second-best results are highlighted in bold black. The last row is improvement ratio between the best and the second best.
\end{table*}

\begin{figure}[!t]
    \centering
    \begin{subfigure}{0.48\linewidth} 

        \includegraphics[width=\linewidth]{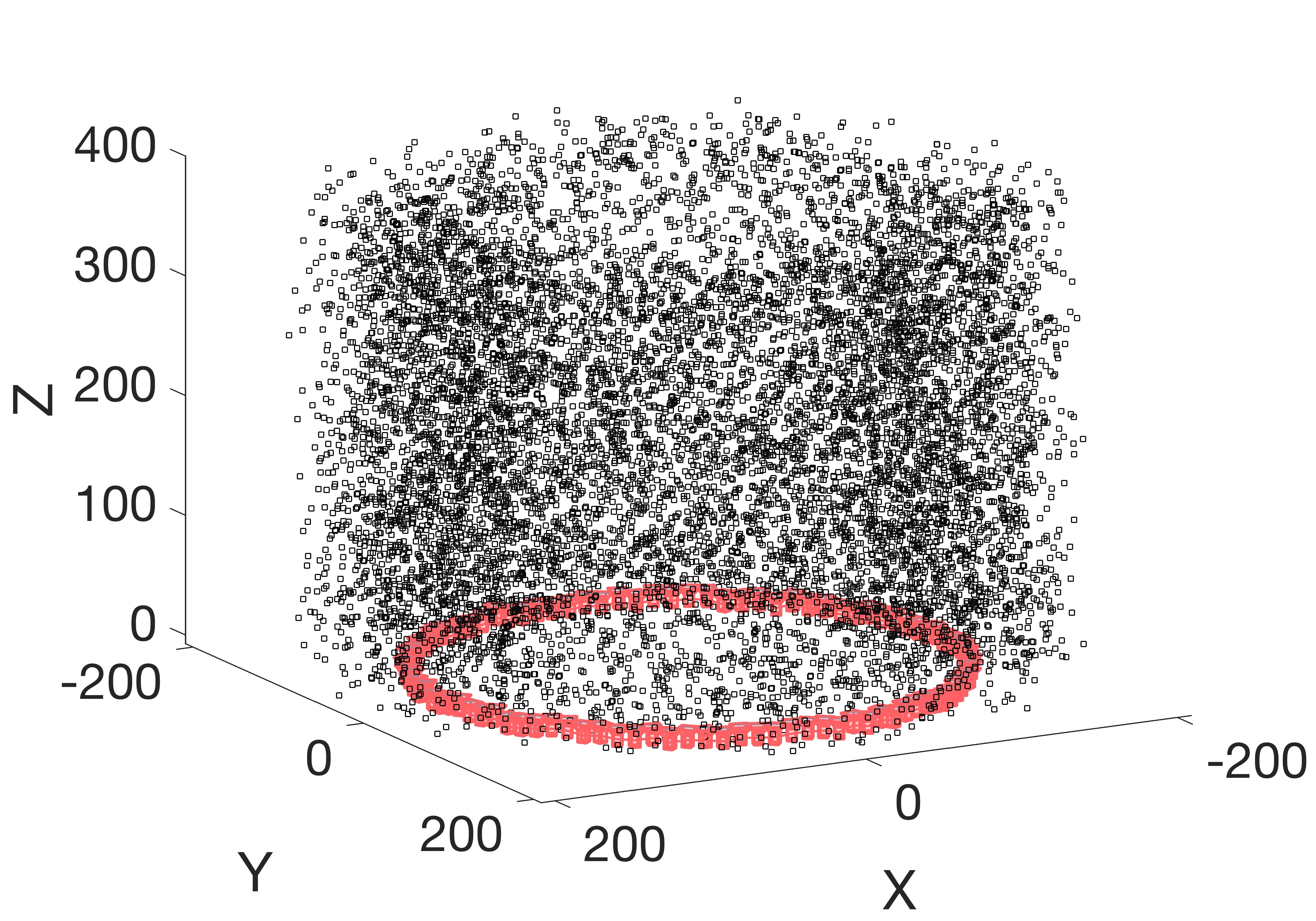}
        \captionsetup{justification=centering, singlelinecheck=false, position=below} 
        \caption{Circular motion}
        \label{fig:GRRM_scene1}
    \end{subfigure}
    \hfill
    \begin{subfigure}{0.48\linewidth}

        \includegraphics[width=\linewidth]{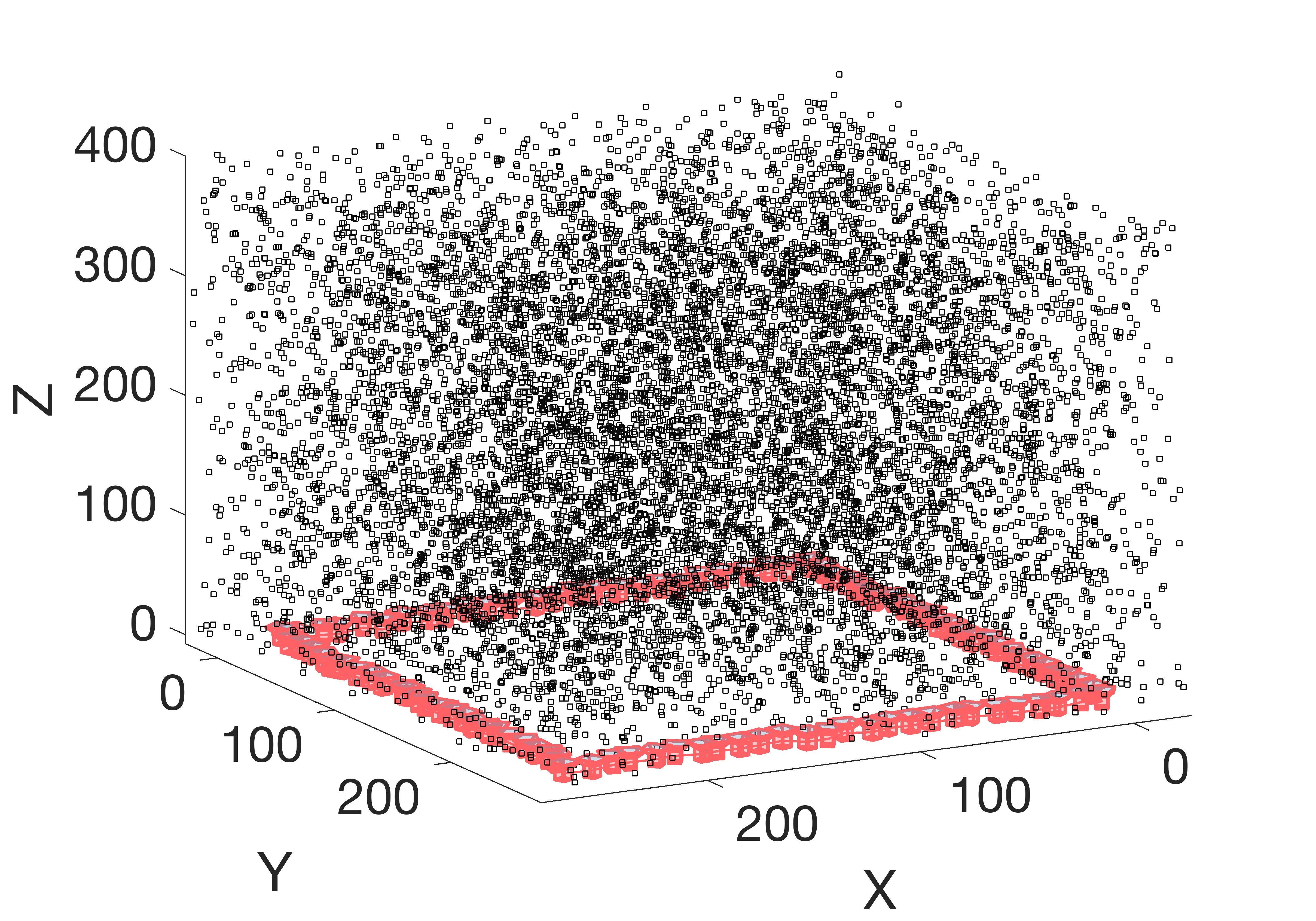}
        \captionsetup{justification=centering, singlelinecheck=false, position=below} 
        \caption{Square motion}
        \label{fig:GRRM_scene2}
    \end{subfigure}


    \begin{subfigure}{0.48\linewidth}

        \includegraphics[width=\linewidth]{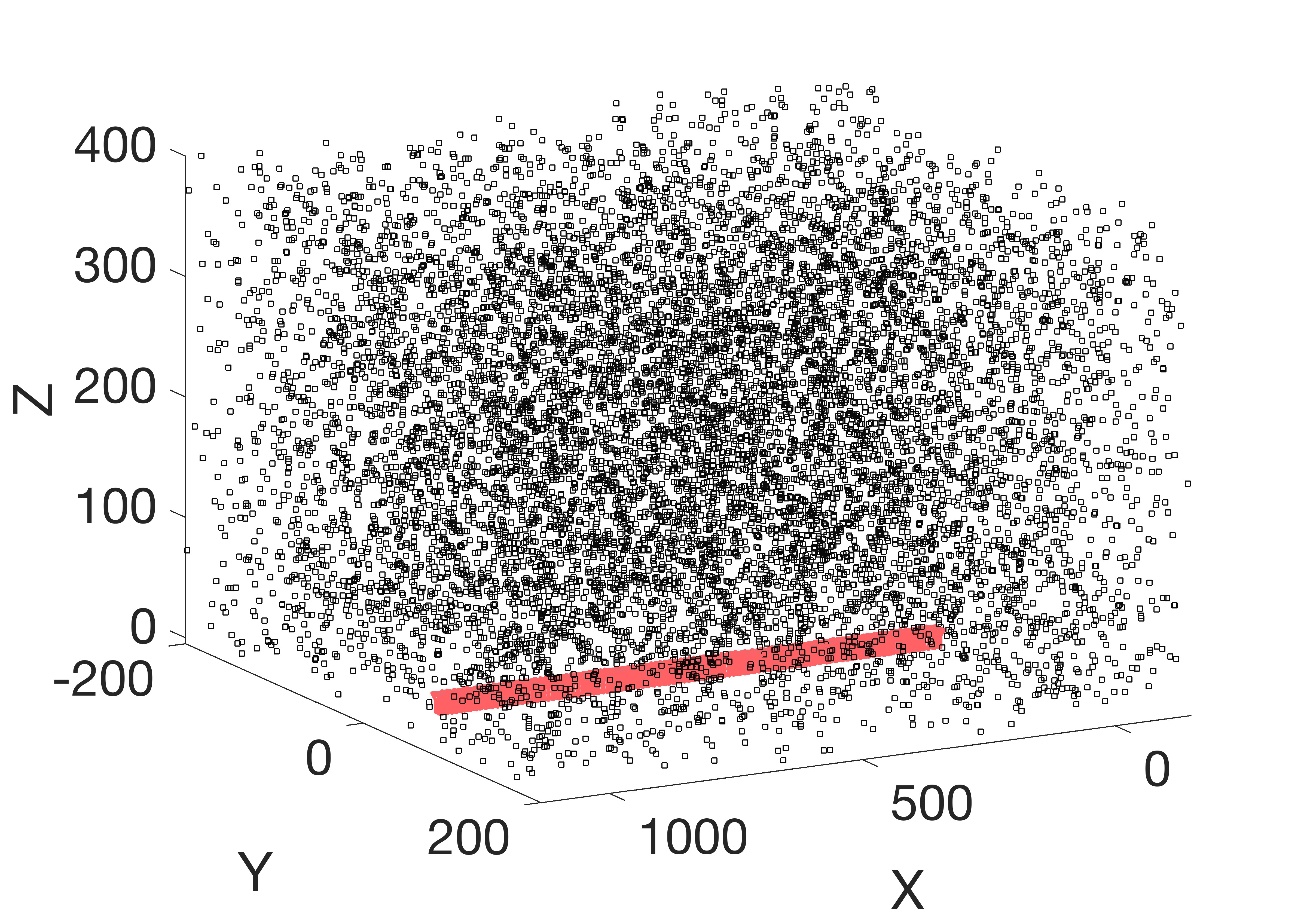}
        \captionsetup{justification=centering, singlelinecheck=false, position=below} 
        \caption{Linear motion}
        \label{fig:GRRM_scene3}
    \end{subfigure}
    \hfill
    \begin{subfigure}{0.48\linewidth}

        \includegraphics[width=\linewidth]{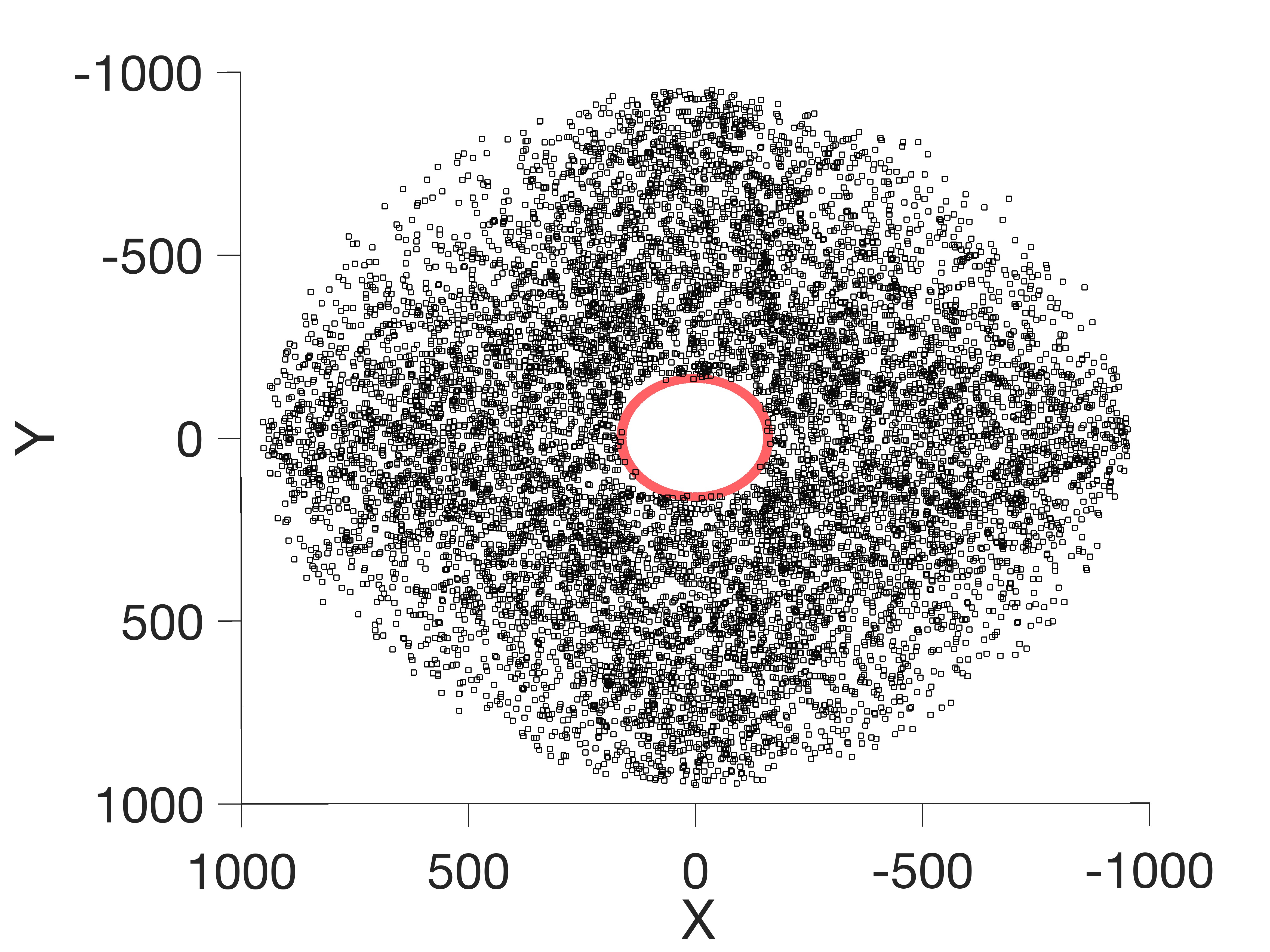}
        \captionsetup{justification=centering, singlelinecheck=false, position=below} 
        \caption{Outward-looking}
        \label{fig:GRRM_scene4}
    \end{subfigure}

    \caption{Four simulation scene structures. Cameras are represented by red models, and 3D scene points are represented by black dots. In (a) and (d), cameras are arranged in a circular layout. In (b), cameras are arranged in a square layout, and in (c), cameras are arranged in a linear layout. In (a), (b), and (c), 3D scene points are located directly above cameras, while in (d), 3D scene points surround cameras.}
    \label{fig:GRRM_scene}
\end{figure}

\subsection{Multi-view Rotation Estimation}
To comprehensively evaluate the performance of our proposed \textit{GRRM} algorithm, we systematically compare it with the rotation estimation algorithms, including the widely used rotation averaging method by Chatterjee and Govindu \cite{ChatterjeePAMI2018} and the recent representative ROBA algorithm \cite{ROBA} in both simulation environments and real-world datasets.

To further highlight the accuracy of our algorithm, we
specifically compare it against
global optimization algorithms. In the simulation experiments, the compared algorithms including classic BA \cite{BundleAdjustmentAModernSynthesis} and PA \cite{QiCai_TPAMI}, whereas in the real-world dataset experiments, 
based on OpenMVG platform, 
we benchmark our algorithm against BA computed through multiple rounds. 
Notably, our algorithm completes the task with a single optimization process.

\begin{figure}[!htbp]
    \centering
        \begin{subfigure}{0.8\linewidth} 
            \includegraphics[width=\linewidth]{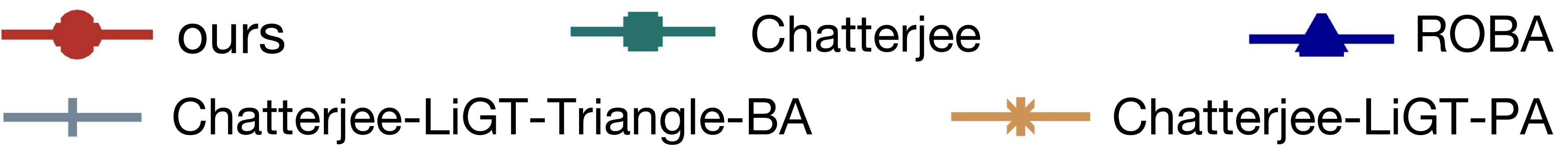}
            \label{fig:legendGROM}
        \end{subfigure}

    \begin{subfigure}{0.49\linewidth} 
        \includegraphics[width=\linewidth]{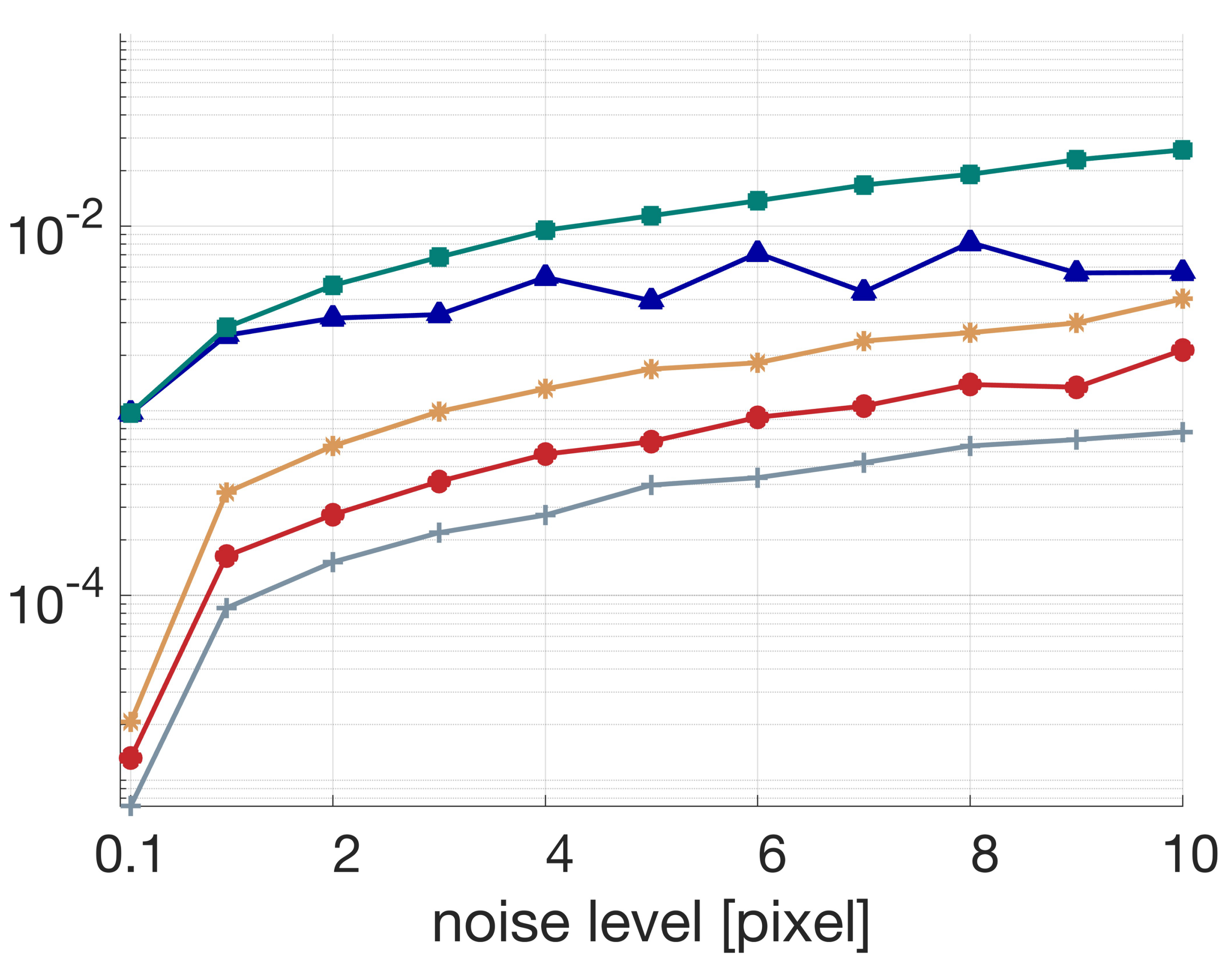}
        \captionsetup{justification=centering, singlelinecheck=false, position=below} 
        \caption{Circular motion}
        \label{fig:GRRM_eval_1}
    \end{subfigure}
    \begin{subfigure}{0.49\linewidth}
        \includegraphics[width=\linewidth]{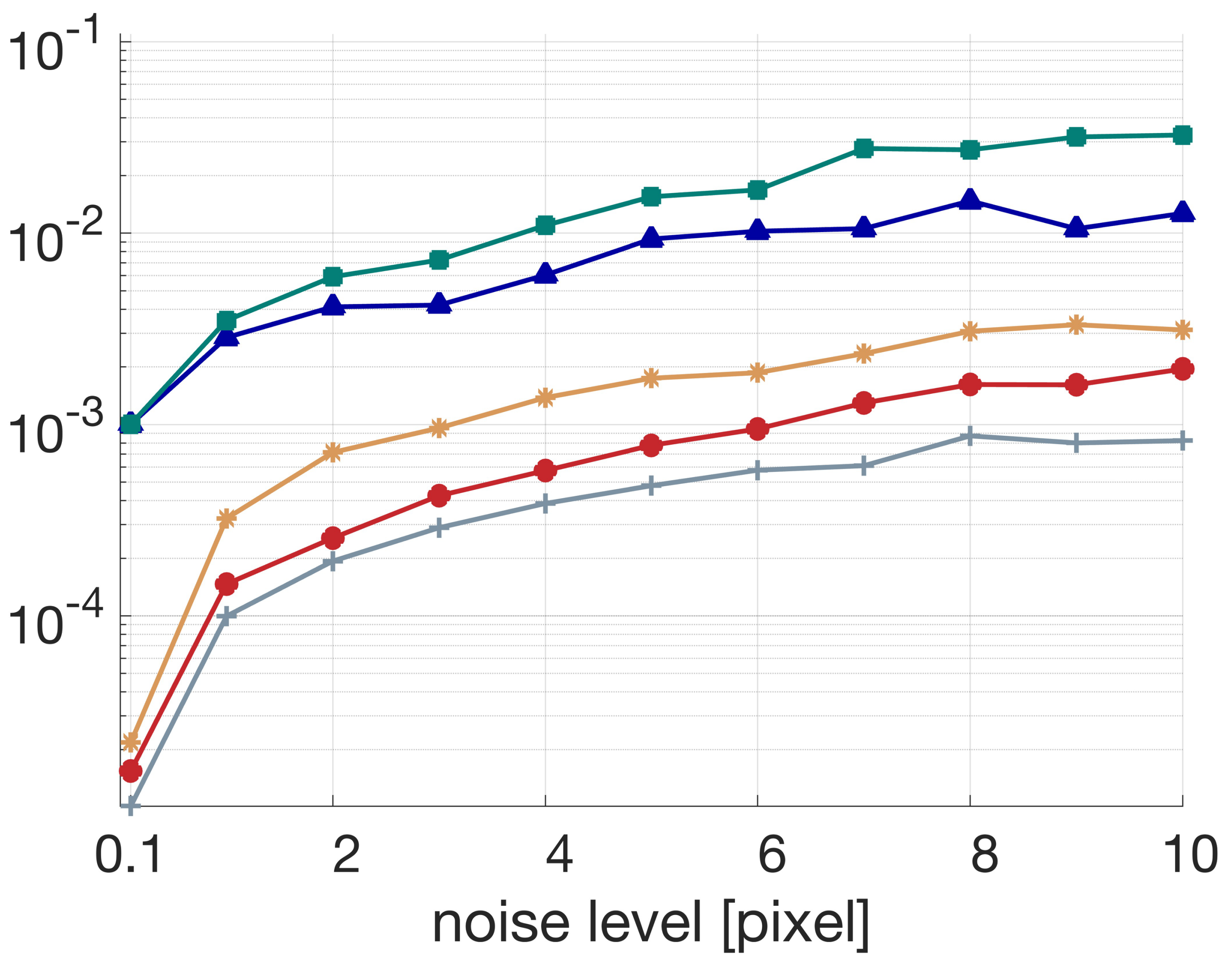}
        \captionsetup{justification=centering, singlelinecheck=false, position=below} 
        \caption{Square motion}
        \label{fig:GRRM_eval_2}
    \end{subfigure}


    \begin{subfigure}{0.49\linewidth}

        \includegraphics[width=\linewidth]{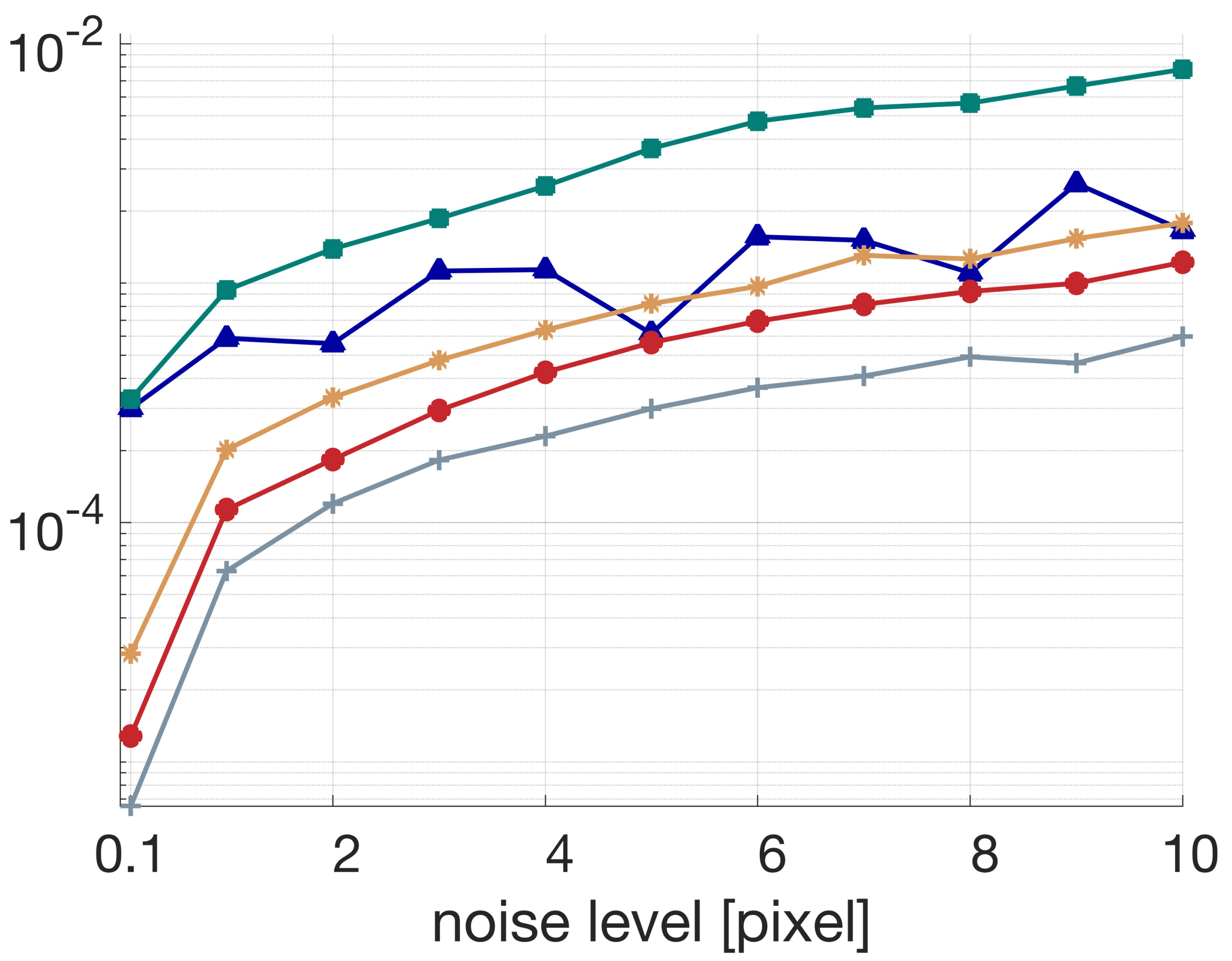}
        \captionsetup{justification=centering, singlelinecheck=false, position=below} 
        \caption{Linear motion}
        \label{fig:GRRM_eval_3}
    \end{subfigure}
    \begin{subfigure}{0.49\linewidth}

        \includegraphics[width=\linewidth]{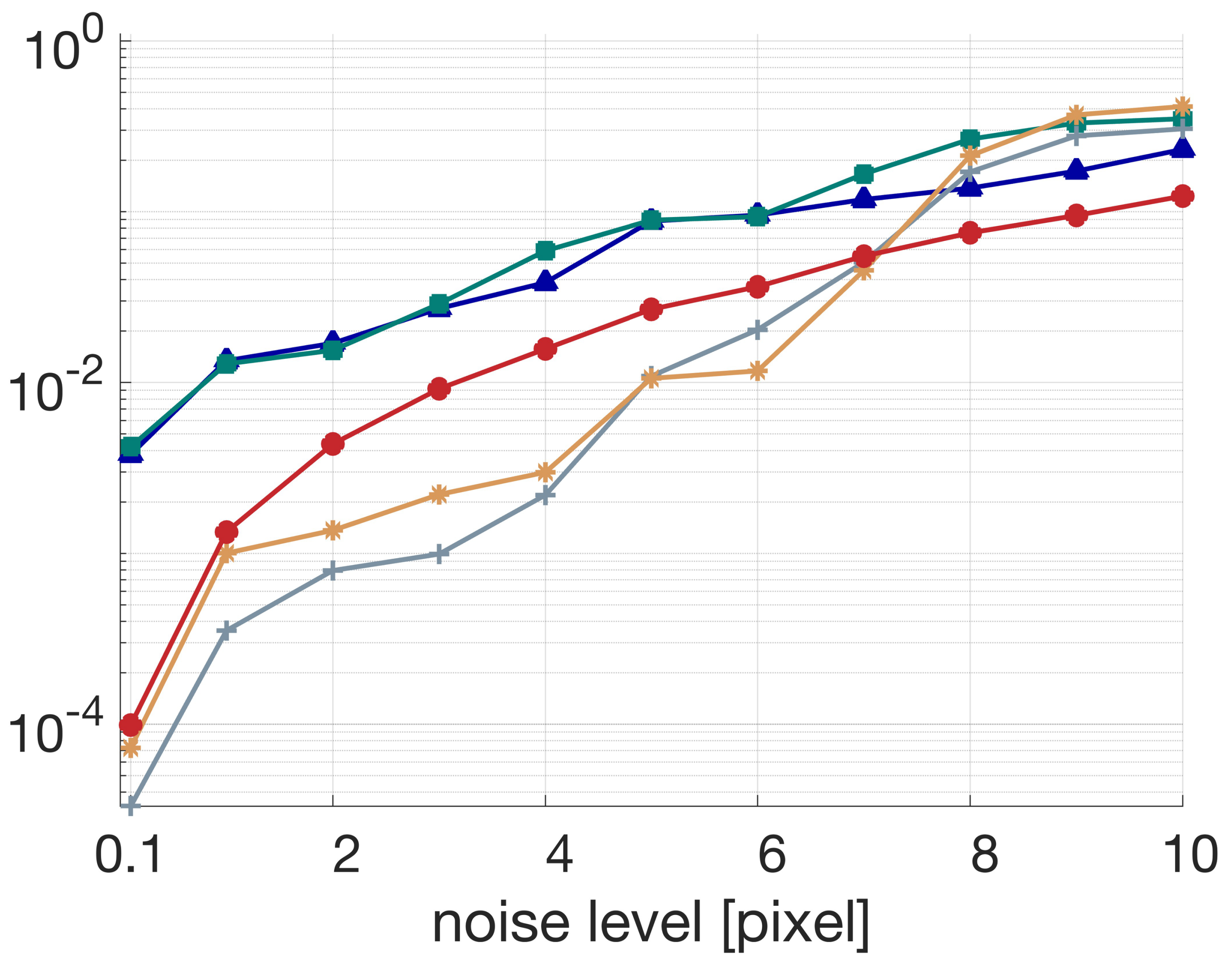}
        \captionsetup{justification=centering, singlelinecheck=false, position=below} 
        \caption{Outward-looking}
        \label{fig:GRRM_eval_4}
    \end{subfigure}

    \caption{Accuracy comparison of our proposed \textit{GRRM} algorithm with global rotation estimation algorithms (Chatterjee's algorithm and ROBA algorithm) and global optimization algorithms (BA and PA) across four simulated scenes.}
    \label{fig:GRRM_eval}
\end{figure}

\subsubsection{Simulation Experiment}
The scene configurations are as shown in Fig. \ref{fig:GRRM_scene}. In scenes (a), (b), (c) and (d), the cameras are arranged in circular, square, linear, and circular, respectively, with the center of the camera array located at the origin of the coordinate system. In scenes (a), (b) and (c), the 3D points are distributed directly above the cameras within a range of 0 to 400 meters. In scene (d), the 3D points are distributed within a cylindrical ring space with the inner radius 200 meters, outer radius 3000 meters, and height (\(z\)-axis range) of -200 to 200 meters.
We refer to these four scene structures as \textit{circular motion}, \textit{square motion}, \textit{linear motion} and \textit{outward-looking}, respectively.
The observation noise is set to follow a uniform distribution \( U(0, n_N) \) in pixels, for \(n_N = 0.1,1,2,...,10\) pixels.

The experimental accuracy results are shown in Fig. \ref{fig:GRRM_eval}. It can be observed that 
our algorithm demonstrates a notable enhancement in accuracy over the rotation estimation algorithms by an order of magnitude,
and achieves accuracy comparable to global optimization algorithms. 
It should be noted that in \textit{outward-looking} scene structure, due to the decrease in the number of matched views, 
the accuracy of global optimization algorithms BA and PA significantly drops under high noise conditions, even falling below the initial value.
However, our algorithm remains unaffected, which demonstrates its robustness for complex scenes.

\begin{table*}[ht]

    \caption{Accuracy of Multi-view Rotation Estimation Algorithms on Strecha Dataset}
    \centering
    \label{tab:strecha_rotation_accuracy}
    \begin{tabular}{lcccccc}
        \toprule
        \textbf{Methods}         & \textbf{Herz-Jesus-P25} & \textbf{Herz-Jesus-P8} & \textbf{Castle-P19} & \textbf{Castle-P30} & \textbf{Entry-P10} & \textbf{Fountain-P11} \\ \midrule
        {Chatterjee's algorithm} & 0.1063                & 0.0868               & 0.3566             & 0.3421            & \textbf{0.0839}  & 0.0422              \\
        {ROBA}      & \textbf{0.0632}        & \textbf{0.0813}      & \textbf{0.1975}            & \textbf{0.2821}            & 0.0967           & \textbf{0.0402}     \\
        \textbf{ours}      & \textbf{\textcolor{red}{0.0510}} & \textbf{\textcolor{red}{0.0636}} & \textbf{\textcolor{red}{0.1764}} & \textbf{\textcolor{red}{0.1208}} & \textbf{\textcolor{red}{0.0651}} & \textbf{\textcolor{red}{0.0262}} \\
        {improvement ratio}      & \textbf{$19.37\%$} & \textbf{$21.69\%$} & \textbf{$10.69\%$} & \textbf{$57.18\%$} & \textbf{$22.49\%$} & \textbf{$34.77\%$} \\ \bottomrule
    \end{tabular}

    \raggedright 
\vspace{2mm}

\end{table*}
\begin{table*}[ht]
    \centering
    \caption{Accuracy of \textit{GRRM} Algorithm and Global Optimization Algorithms on Strecha Dataset}
    \label{tab:strecha_pore_accuracy_global_opt}
    \begin{tabular}{lcccccc}
        \toprule
        \textbf{Methods}         & \textbf{Herz-Jesus-P25} & \textbf{Herz-Jesus-P8} & \textbf{Castle-P19} & \textbf{Castle-P30} & \textbf{Entry-P10} & \textbf{Fountain-P11} \\ \midrule
        {2th BA}    & 0.0562                & 0.0712               & 0.2281            & 0.2520             & 0.2034            & 0.0500              \\
        {4th BA}    & \textbf{\textcolor{red}{0.0479}} & \textbf{{0.0696}} & \textbf{\textcolor{red}{0.0642}} & \textbf{\textcolor{red}{0.0876}} & \textbf{{0.1282}} & \textbf{{0.0363}} \\
        \textbf{ours}     & \textbf{0.0510}       & \textbf{\textcolor{red}{0.0636}}      & \textbf{0.1764}    & \textbf{0.1208}   & \textbf{\textcolor{red}{0.0651}}  & \textbf{\textcolor{red}{0.0262}}     \\ \bottomrule
    \end{tabular}
\end{table*}

\subsubsection{Real-world Dataset Experiment.}
The accuracy results of our proposed \textit{GRRM} algorithm and rotation estimation algorithms are listed in Table \ref{tab:strecha_rotation_accuracy}. 
Our method consistently achieves the superior optimal accuracy across all datasets, with averaged $27.70\%$ improvement over the second best algorithm.

Global BA optimization is the final stage in the OpenMVG pipeline, which performs four rounds of BA optimization to ensure the accuracy and robustness of camera pose recovery and scene reconstruction. The process includes multiple rounds of optimization for both the intrinsic and extrinsic parameters of the cameras, as well as the positions of 3D points, combined with a rigorous filtering process for camera observations. The specific rounds are as follows:

\begin{itemize}
    \item {First BA:} Optimize only the position parameters of cameras and 3D scene points.
    \item {Second BA:} Simultaneously optimize the pose parameters of cameras and the position parameters of 3D scene points.
    \item {Third BA:} Simultaneously optimize both the intrinsic and extrinsic parameters of cameras as well as the position parameters of 3D scene points.
    \item {Fourth BA:} After completing the first three optimizations, remove observations with a reprojection error greater than 4 pixels or a parallax angle less than 2 degrees, and then perform the same optimization as in the third round.
\end{itemize}

We undertake an ambitious comparison between our proposed \textit{GRRM} algorithm and the second and fourth BA optimization iterations, as shown in Table \ref{tab:strecha_pore_accuracy_global_opt}.
The accuracy of our one-round algorithm surpasses that of the second BA optimization and achieves comparable performance to that of the fourth BA optimization. 
This 
also highlights its ability to maintain high accuracy with significantly greater computational efficiency and robustness, 
highlighting its significant potential in multi-view applications.

\section{Conclusion}

This paper provides a representation of the visual geometry on rotation manifold by decoupling camera translation from rotation. 
It shows how scene structure affects translation solution space and proposes an effective scene structure detection method to address the challenge of singular scene structures to translation estimation.  
The work extends the pose-only imaging geometry by further reducing the parameter dimensionality. 
A framework that minimizes the reprojection error on rotation manifold is introduced for 3D visual computing, which significantly enhances the accuracy (with averaged $17.01\%$ and $27.70\%$ improvement over the second best two-view optimization algorithm and multi-view rotation estimation algorithm, respectively) and robustness of camera rotation estimation in both two-view and multi-view scenes. The accuracy is notably comparable to four rounds of BA optimization in the OpenMVG platform.


%



\ifCLASSOPTIONcompsoc
  \section*{Acknowledgments}
\else
  \section*{Acknowledgment}
\fi

This work was supported in part by National Key R\&D Program (2022YFB3903802), National Natural Science Foundation of China (62273228) and Shanghai Jiao Tong University Scientific and Technological Innovation Funds.

\ifCLASSOPTIONcaptionsoff
  \newpage
\fi



\bibliographystyle{IEEEtran}
\bibliography{ref}

@article{QiCai_IJCV,
   author = {Cai, Qi and Wu, Yuanxin and Zhang, Lilian and Zhang, Peike},
   title = {Equivalent constraints for two-view geometry: Pose solution/pure rotation identification and 3D reconstruction},
   journal = {International Journal of Computer Vision},
   volume = {127},
   number = {2},
   pages = {163-180},
   year = {2018},
   type = {Journal Article}
}

@book{MultipleViewGeometry, 
   place={Cambridge}, 
   edition={2}, 
   title={Multiple View Geometry in Computer Vision}, 
   publisher={Cambridge University Press}, 
   author={Hartley, Richard and Zisserman, Andrew}, 
   year={2004}
}

@article{QiCai_TPAMI,
   author = {Cai, Qi and Zhang, Lilian and Wu, Yuanxin and Yu, Wenxian and Hu, Dewen},
   title = {A pose-only solution to visual reconstruction and navigation},
   journal = {IEEE Transactions on Pattern Analysis and Machine Intelligence},
   volume = {45},
   number = {1},
   pages = {73-86},
   year = {2023},
   type = {Journal Article}
}

@INPROCEEDINGS{Kneip13ICCV,
  author={Kneip, Laurent and Lynen, Simon},
  booktitle={IEEE International Conference on Computer Vision}, 
  title={Direct optimization of frame-to-frame rotation}, 
  year={2013},
  volume={},
  number={},
  pages={2352-2359}}

@ARTICLE{ChatterjeePAMI2018,
  author={Chatterjee, Avishek and Govindu, Venu Madhav},
  journal={IEEE Transactions on Pattern Analysis and Machine Intelligence}, 
  title={Robust relative rotation averaging}, 
  year={2018},
  volume={40},
  number={4},
  pages={958-972}}

@INPROCEEDINGS{ROBA,
  author={Lee, Seong Hun and Civera, Javier},
  booktitle={IEEE/CVF Conference on Computer Vision and Pattern Recognition}, 
  title={Rotation-only bundle adjustment}, 
  year={2021},
  volume={},
  number={},
  pages={424-433}}

@INPROCEEDINGS{HartleyL1,
  author={Hartley, Richard and Aftab, Khurrum and Trumpf, Jochen},
  booktitle={IEEE/CVF Conference on Computer Vision and Pattern Recognition}, 
  title={L1 rotation averaging using the {Weiszfeld} algorithm}, 
  year={2011},
  volume={},
  number={},
  pages={3041-3048}}

@ARTICLE{cai2024linearrelativeposeestimation,
      title={Linear relative pose estimation founded on pose-only imaging geometry}, 
      author={Qi Cai and Xinrui Li and Yuanxin Wu},
      journal={arXiv:2401.13357} 
}

@INPROCEEDINGS{Martinec07,
  author={Martinec, Daniel and Pajdla, Tomas},
  booktitle={IEEE/CVF Conference on Computer Vision and Pattern Recognition}, 
  title={Robust rotation and translation estimation in multiview reconstruction}, 
  year={2007},
  volume={},
  number={},
  pages={1-8}}

@INPROCEEDINGS{KneipOpenGV,
  author={Kneip, Laurent and Furgale, Paul},
  booktitle={IEEE International Conference on Robotics and Automation}, 
  title={Open{GV}: A unified and generalized approach to real-time calibrated geometric vision}, 
  year={2014},
  volume={},
  number={},
  pages={1-8}}

@article{agarwal2022chiral,
  title={The chiral domain of a camera arrangement},
  author={Agarwal, Sameer and Pryhuber, Andrew and Sinn, Rainer and Thomas, Rekha R},
  journal={Journal of Mathematical Imaging and Vision},
  volume={64},
  number={9},
  pages={948--967},
  year={2022},
  publisher={Springer}
}

@INPROCEEDINGS{StrechaDataset,
  author={Strecha, C. and von Hansen, W. and Van Gool, L. and Fua, P. and Thoennessen, U.},
  booktitle={IEEE/CVF Conference on Computer Vision and Pattern Recognition}, 
  title={On benchmarking camera calibration and multi-view stereo for high resolution imagery}, 
  year={2008},
  volume={},
  number={},
  pages={1-8}}

@InProceedings{OpenMVGPierre,
author="Moulon, Pierre
and Monasse, Pascal
and Marlet, Renaud",
title="Adaptive structure from motion with a contrario model estimation",
booktitle="Asian Conference on Computer Vision",
year="2013",
publisher="Springer Berlin Heidelberg",
address="Berlin, Heidelberg",
pages="257--270"
}

@article{STEWENIUS2006284,
title = {Recent developments on direct relative orientation},
journal = {ISPRS Journal of Photogrammetry and Remote Sensing},
volume = {60},
number = {4},
pages = {284-294},
year = {2006},
author = {Henrik Stewénius and Christopher Engels and David Nistér}
}

@InProceedings{BundleAdjustmentAModernSynthesis,
author="Triggs, Bill
and McLauchlan, Philip F.
and Hartley, Richard I.
and Fitzgibbon, Andrew W.",
title="Bundle adjustment -- a modern synthesis",
booktitle="Vision Algorithms: Theory and Practice",
year="2000",
publisher="Springer Berlin Heidelberg",
address="Berlin, Heidelberg",
pages="298--372"
}

@ARTICLE{NisterRelative,
  author={Nister, D.},
  journal={IEEE Transactions on Pattern Analysis and Machine Intelligence}, 
  title={An efficient solution to the five-point relative pose problem}, 
  year={2004},
  volume={26},
  number={6},
  pages={756-770}}

@inproceedings{Pizarro2003RelativePE,
  title={Relative pose estimation for instrumented, calibrated imaging platforms},
  author={Oscar Pizarro and Ryan M. Eustice and Hanumant Singh},
  booktitle={International Conference on Digital Image Computing: Techniques and Applications},
  year={2003}
}

@INPROCEEDINGS{SevenPtHartley,
  author={Hartley, Richard},
  booktitle={Proceedings of IEEE International Conference on Computer Vision}, 
  title={In defence of the 8-point algorithm}, 
  year={1995},
  volume={},
  number={},
  pages={1064-1070}}

@ARTICLE{ZhaoJi,
  author={Zhao, Ji},
  journal={IEEE Transactions on Pattern Analysis and Machine Intelligence}, 
  title={An efficient solution to non-minimal case essential matrix estimation}, 
  year={2022},
  volume={44},
  number={4},
  pages={1777-1792}}

@misc{theia-manual,
  author = {Chris Sweeney},
  title = {Theia multiview geometry library: Tutorial \& reference},
  url = "http://theia-sfm.org",
}

@INPROCEEDINGS{Govindu01,
  author={Govindu, Venu Madhav},
  booktitle={IEEE/CVF Conference on Computer Vision and Pattern Recognition}, 
  title={Combining two-view constraints for motion estimation}, 
  year={2001},
  volume={2}}

@INPROCEEDINGS{Chatterjee13,
  author={Chatterjee, Avishek and Govindu, Venu Madhav},
  booktitle={IEEE International Conference on Computer Vision}, 
  title={Efficient and robust large-scale rotation averaging}, 
  year={2013},
  volume={},
  number={},
  pages={521-528}}

@INPROCEEDINGS{Govindu004,
  author={Govindu, Venu Madhav},
  booktitle={IEEE/CVF Conference on Computer Vision and Pattern Recognition},
  title={Lie-algebraic averaging for globally consistent motion estimation},
  year={2004},
  volume={1}}

@INPROCEEDINGS{BAinLarge,
	author = {Agarwal, Sameer and Snavely, Noah and Seitz, Steven M. and Szeliski, Richard},
	booktitle = {European Conference on Computer Vision},
	date = {2010//},
	pages = {29--42},
	publisher = {Springer Berlin Heidelberg},
	title = {Bundle adjustment in the large},
	year = {2010}}

@article{Maybank90,
 author = {S. J. Maybank},
 journal = {Philosophical Transactions: Physical Sciences and Engineering},
 number = {1623},
 pages = {1--47},
 publisher = {Royal Society},
 title = {The projective geometry of ambiguous surfaces},
 urldate = {2025-01-19},
 volume = {332},
 year = {1990}
}

@misc{lunddataset,
  author = {Carl Olsson and Olof Enqvist and Fredrik Kahl},
  title = {Lund dataset},
  year = {2011},
  url = {https://www.maths.lth.se/matematiklth/personal/calle/dataset/dataset.html},
}

@inproceedings{glomap,
author = {Pan, Linfei and Bar\'{a}th, D\'{a}niel and Pollefeys, Marc and Sch\"{o}nberger, Johannes L.},
title = {Global structure-from-motion revisited},
year = {2024},
publisher = {Springer-Verlag},
address = {Berlin, Heidelberg},
booktitle = {European Conference on Computer Vision},
pages = {58--77},
numpages = {20}
}

@article{LONGUETHIGGINS198761,
  title = {A computer algorithm for reconstructing a scene from two projections},
  author = {{Longuet-Higgins}, H. C.},
  year = {1981},
  journal = {Nature},
  volume = {293},
  number = {5828},
  pages = {133--135}
}

@InProceedings{KneipECCV2012,
author="Kneip, Laurent
and Siegwart, Roland
and Pollefeys, Marc",
title="Finding the exact rotation between two images independently of the translation",
booktitle="European Conference on Computer Vision",
year="2012",
publisher="Springer Berlin Heidelberg",
address="Berlin, Heidelberg",
pages="696--709"
}

@article{SIFT,
  title = {Distinctive {{image features}} from {{scale-invariant keypoints}}},
  author = {Lowe, David G.},
  year = {2004},
  journal = {International Journal of Computer Vision},
  volume = {60},
  number = {2},
  pages = {91--110}
}

@INPROCEEDINGS{cascadeMatch,
  author={Cheng, Jian and Leng, Cong and Wu, Jiaxiang and Cui, Hainan and Lu, Hanqing},
  booktitle={IEEE/CVF Conference on Computer Vision and Pattern Recognition}, 
  title={Fast and accurate image matching with cascade hashing for 3D reconstruction}, 
  year={2014},
  volume={},
  number={},
  pages={1-8}}

@article{ACRANSAC,
    title   = {{Automatic homographic registration of a pair of images, with a contrario elimination of outliers}},
    author  = {Moisan, Lionel and Moulon, Pierre and Monasse, Pascal},
    journal = {{Image Processing On Line}},
    volume  = {2},
    pages   = {56--73},
    year    = {2012}
}

@INPROCEEDINGS{OutOfCoreBA,
  author={Ni, Kai and Steedly, Drew and Dellaert, Frank},
  booktitle={IEEE International Conference on Computer Vision}, 
  title={Out-of-core bundle adjustment for large-scale 3D reconstruction}, 
  year={2007},
  volume={},
  number={},
  pages={1-8}}

@ARTICLE{DistributedBA,
  author={Zhang, Runze and Zhu, Siyu and Shen, Tianwei and Zhou, Lei and Luo, Zixin and Fang, Tian and Quan, Long},
  journal={IEEE Transactions on Pattern Analysis and Machine Intelligence}, 
  title={Distributed very large scale bundle adjustment by global camera consensus}, 
  year={2020},
  volume={42},
  number={2},
  pages={291-303}}

@InProceedings{1dsfm,
author="Wilson, Kyle
and Snavely, Noah",
title="Robust global translations with 1DSfM",
booktitle="European Conference on Computer Vision",
year="2014",
publisher="Springer International Publishing",
pages="61--75"
}

@incollection{hartleyTriangulation1995,
  title = {Triangulation},
  booktitle = {Computer {{Analysis}} of {{Images}} and {{Patterns}}},
  author = {Hartley, Richard and Sturm, Peter},
  year = {1995},
  volume = {970},
  pages = {190--197},
  publisher = {Springer Berlin Heidelberg},
  address = {Berlin, Heidelberg}}

@inproceedings{LUD,
  title = {Robust camera location estimation by convex programming},
  booktitle = {{{IEEE Conference}} on {{Computer Vision}} and {{Pattern Recognition}}},
  author = {Ozyesil, Onur and Singer, Amit},
  year = {2015},
  pages = {2674--2683},
  publisher = {IEEE},
  address = {Boston, MA, USA}
}

@inproceedings{jiangGlobalLinearMethod2013a,
  title = {A {{global linear method}} for {{camera pose registration}}},
  booktitle = {{{IEEE International Conference}} on {{Computer Vision}}},
  author = {Jiang, Nianjuan and Cui, Zhaopeng and Tan, Ping},
  year = {2013},
  pages = {481--488},
  publisher = {IEEE},
  address = {Sydney, Australia}
}

@article{zhaoParallaxBABundleAdjustment2015,
  title = {{{ParallaxBA}}: Bundle adjustment using parallax angle feature parametrization},
  shorttitle = {{{ParallaxBA}}},
  author = {Zhao, Liang and Huang, Shoudong and Sun, Yanbiao and Yan, Lei and Dissanayake, Gamini},
  year = {2015},
  journal = {The International Journal of Robotics Research},
  volume = {34},
  number = {4-5},
  pages = {493--516}
}

@inproceedings{erikssonConsensusBasedFrameworkDistributed2016,
  title = {A {{consensus-based framework}} for {{distributed bundle adjustment}}},
  booktitle = {{{IEEE Conference}} on {{Computer Vision}} and {{Pattern Recognition}}},
  author = {Eriksson, Anders and Bastian, John and Chin, Tat-Jun and Isaksson, Mats},
  year = {2016},
  pages = {1754--1762},
  publisher = {IEEE},
  address = {Las Vegas, NV, USA}
}

@article{KANG20142974,
title = {Robust multi-view L2 triangulation via optimal inlier selection and 3D structure refinement},
journal = {Pattern Recognition},
volume = {47},
number = {9},
pages = {2974-2992},
year = {2014},
author = {Lai Kang and Lingda Wu and Yee-Hong Yang}
}

@online{openGVdefaultScene,
  author = {Laurent Kneip and contributors},
  title = {create2D2DExperiment.m -- Helper script for 2D-2D experiments in {OpenGV}},
  url = {https://github.com/laurentkneip/opengv/blob/master/matlab/helpers/create2D2DExperiment.m},
}

@online{openMVG_evalQuality,
  author       = {OpenMVG},
  title        = {main{\_}evalQuality.cpp -- Structure-from-Motion Quality Evaluation},
  url = {https://github.com/openMVG/openMVG/blob/develop/src/software/SfM/main_evalQuality.cpp},
}

@online{sunghoon2023roba,
  author    = {Lee, Seong Hun},
  title     = {R{OBA}},
  url       = {https://github.com/sunghoon031/ROBA},
}
%

%

\begin{IEEEbiography}[{\includegraphics[width=1in,height=1.25in,clip,keepaspectratio]{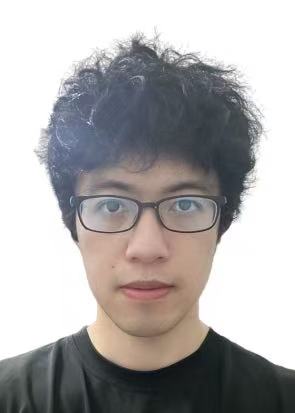}}]{Xinrui Li}
  received the B.Sc. degree in School of Physics, University of Electronic Science and Technology of China. He is now pursuing his Ph.D. degree in School of Electronic Information and Electrical Engineering, Shanghai Jiao Tong University. His research interest includes computer vision geometry and Simultaneous Localization and Mapping.
\end{IEEEbiography}

\begin{IEEEbiography}[{\includegraphics[width=1in,height=1.25in,clip,keepaspectratio]{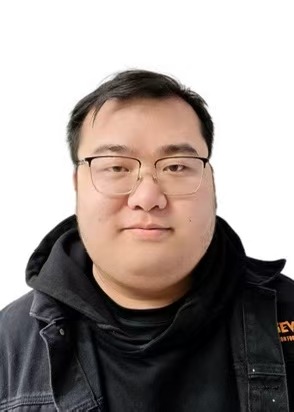}}]{Qi Cai}
  received his B.Sc. and M.Sc. degrees from Central South University, and earned his Ph.D. from the School of Electronic Information and Electrical Engineering at Shanghai Jiao Tong University. He was awarded the Shanghai Super Postdoctor Fellowship. He is currently a postdoctoral researcher at Shanghai Jiao Tong University, China. His research focuses on geodesy, computer vision, and inertial–visual fusion for navigation.
\end{IEEEbiography}


\begin{IEEEbiography}[{\includegraphics[width=1in,height=1.25in,clip,keepaspectratio]{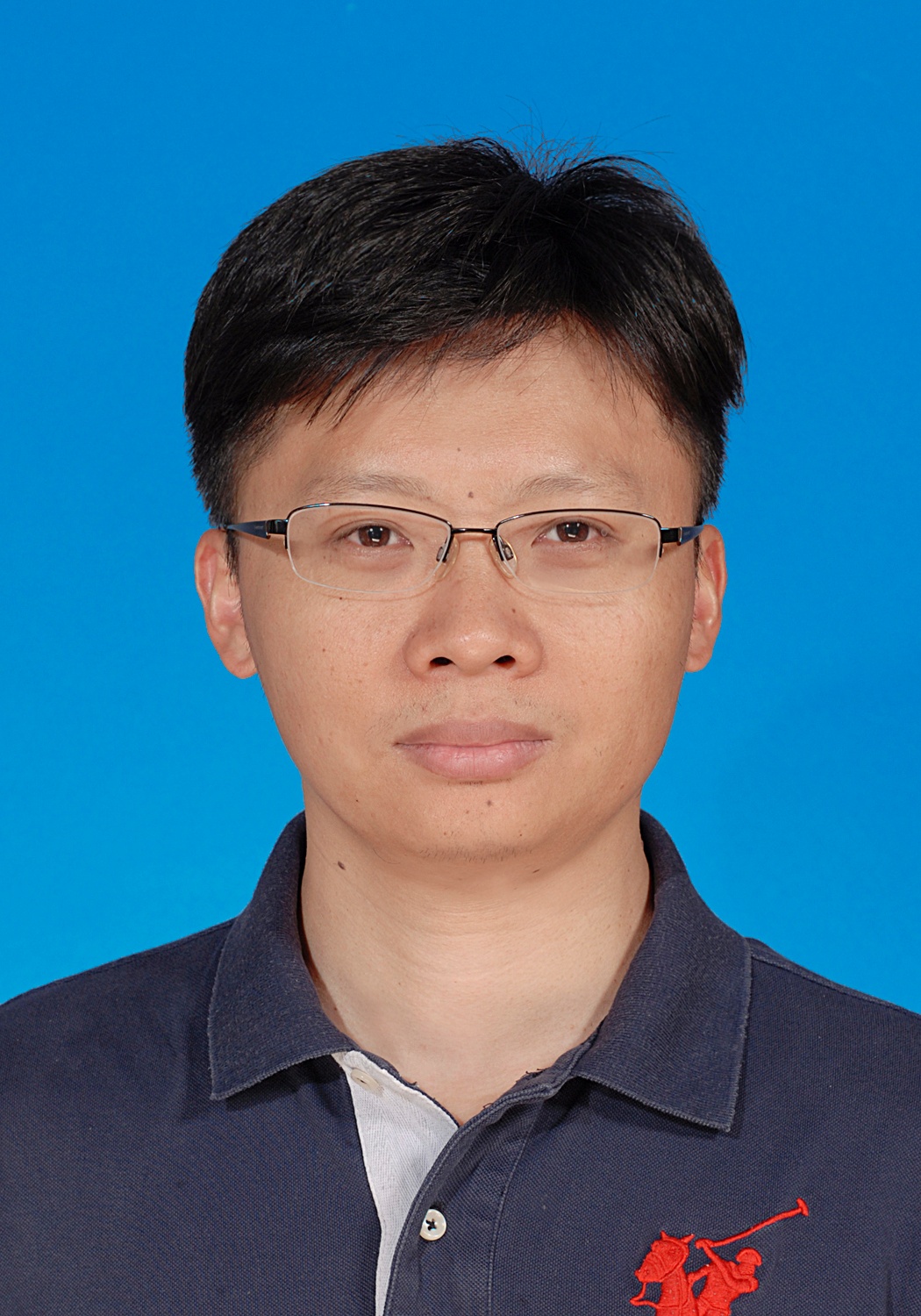}}]{Yuanxin Wu}
  (Senior Member, IEEE) received
  the B.Sc. and Ph.D. degrees from
  the Department of Automatic Control, National
  University of Defense Technology, Changsha,
  China, respectively.
  He was a Lecturer and an Associate Professor
  with the National University of Defense Technology, a visiting scholar with the
  Department of Geomatics Engineering, University of Calgary, Canada, and a Professor with Central South University, Changsha.
  He is currently a Professor with the School of Electronic Information and
  Electrical Engineering, Shanghai Jiao Tong University, Shanghai, China.
  His current research interests include inertial-based navigation system, inertial-visual fusion, and wearable human motion sensing.
  Dr. Wu was the recipient of NSFC Award for Excellent Young Scientists, Natural Science and Technology Award in University,
  and Elsevier's Most Cited Chinese Researchers. He is currently an Associate Editor for IEEE TRANSACTIONS ON AEROSPACE AND ELECTRONIC SYSTEMS, Program Committee
  for DGON Inertial Sensors and Systems, Saint Petersburg
  International Conference on Integrated Navigation Systems. He served as IEEE Aerospace and Electronic Systems Society Distinguished
  Lecturer (2020--2024).
\end{IEEEbiography}



\end{document}